%% file: main.tex
\documentclass[10pt]{article}
\usepackage{enumerate}
\usepackage{geometry}
\usepackage[OT1]{fontenc}
\usepackage{amsmath,amssymb}
\usepackage{mathtools}
\usepackage{natbib}
\usepackage[usenames]{color}
\usepackage{dsfont}
\usepackage{pifont}
\usepackage{pgfplots}
\usepackage{smile}
\usepackage{multirow}
\usepackage{rotating}
\usepackage{enumerate}
\usepackage{esvect}
\usepackage{tikz}
\usetikzlibrary{patterns}
\usetikzlibrary{arrows}
\usepackage{color,xcolor}
\usepackage{colortbl}
\usepackage[colorlinks,
            linkcolor=red,
            anchorcolor=red,
            citecolor=blue
            ]{hyperref}
\usepackage{algorithm}
\usepackage{algorithmic}
\usepackage{cancel}
\usepackage{paralist}

\def\supp{\mathop{\text{supp}}}

\long\def\comment#1{}

\def\cS{{\mathcal{S}}}

\newcommand{\bel}{\begin{eqnarray}\label}
\newcommand{\eel}{\end{eqnarray}}
\newcommand{\bes}{\begin{eqnarray*}}
\newcommand{\ees}{\end{eqnarray*}}

\let\tilde\widetilde
\def\mid{\,|\,}

\def\supp{\mathop{\text{supp}\kern.2ex}}

\def\given{{\,|\,}}

\def\supp{\mathop{\text{supp}}}

\def\##1\#{\begin{align}#1\end{align}}
\def\$#1\${\begin{align*}#1\end{align*}}

\usepackage{undertilde}

\title{Maximize to Explore: One Objective Function Fusing Estimation, Planning, and Exploration}
\author{Zhihan Liu\thanks{Equal contribution.} \thanks{Northwestern University. \texttt{\{zhihanliu2027,shenaozhang2028,siruizheng2025\}@u.northwestern.edu,zhaoranwang@gmail.com}} \qquad
    Miao Lu\footnotemark[1] \thanks{Stanford University. \texttt{miaolu@stanford.edu}} \qquad
    Wei Xiong\footnotemark[1] \thanks{University of Illinois Urbana-Champaign. \texttt{wx13@illinois.edu}} \qquad
    Han Zhong\thanks{Peking University.  \texttt{hanzhong@stu.pku.edu.cn}} \qquad
    Hao Hu\thanks{Tsinghua University. \texttt{huh22@mails.tsinghua.edu.cn}}\\
    Shenao Zhang\footnotemark[2]\qquad
    Sirui Zheng\footnotemark[2]\qquad
    Zhuoran Yang\thanks{Yale University. \texttt{zhuoran.yang@yale.edu}}\qquad 
    Zhaoran Wang\footnotemark[2]}
\date{\today}

\begin{document}

\maketitle

\begin{abstract}
    In online reinforcement learning (online RL), balancing exploration and exploitation is crucial for finding an optimal policy in a sample-efficient way. 
    To achieve this, existing sample-efficient online RL algorithms typically consist of three components: estimation, planning, and exploration. 
    However, in order to cope with general function approximators, most of them involve impractical algorithmic components to incentivize exploration, such as optimization within data-dependent level-sets or complicated sampling procedures. 
    To address this challenge, we propose an easy-to-implement RL framework called \textit{Maximize to Explore} (\texttt{MEX}), which only needs to optimize \emph{unconstrainedly} a single objective that integrates the estimation and planning components while balancing exploration and exploitation automatically. 
    Theoretically, we prove that \texttt{MEX} achieves a sublinear regret with general function approximations for Markov decision processes (MDP) and is further extendable to two-player zero-sum Markov games (MG).  
    Meanwhile, we adapt deep RL baselines to design practical versions of \texttt{MEX}, in both model-free and model-based manners, which can outperform baselines by a stable margin in various MuJoCo environments with sparse rewards.
    Compared with existing sample-efficient online RL algorithms with general function approximations, \texttt{MEX} achieves similar sample efficiency while enjoying a lower computational cost and is more compatible with modern deep RL methods. Our codes are available at \url{https://github.com/agentification/MEX}.
\end{abstract}

\tableofcontents

\newpage

\input{tex/introduction.tex}

\input{tex/preliminaries.tex}

\input{tex/algorithm.tex}

\input{tex/theory.tex}

\input{tex/examples.tex}
\input{tex/mg.tex}
\input{tex/exp.tex}

\input{tex/conclusion.tex}

\bibliographystyle{ims}

\bibliography{ref}

\newpage
\appendix

\input{tex/appendix/proof_main.tex}

\input{tex/appendix/proof_mdp_example.tex}

\input{tex/appendix/proof_mg_example.tex}

\input{tex/appendix/technical.tex}

\end{document}

%% file: tex/introduction.tex
\section{Introduction}

The crux of online reinforcement learning (online RL) lies in maintaining a balance between i) exploiting the current knowledge of the agent about the environment and ii) exploring unfamiliar areas \citep{sutton2018reinforcement}. 
To fulfill this, agents in existing sample-efficient RL algorithms  predominantly undertake three tasks: 
i) \emph{estimate} a hypothesis using historical data to encapsulate their understanding of the environment; 
ii) perform \emph{planning} based on the estimated hypothesis to exploit their current knowledge; 
iii) further \emph{explore} the unknown environment via carefully designed exploration strategies.

There exists a long line of research on integrating the aforementioned three components harmoniously to find the optimal policy in a sample-efficient manner. 
From theoretical perspectives, existing theories aim to minimize the notion of \emph{online external regret} which measures the cumulative suboptimality gap of the policies learned during online learning. 
It is well studied that one can design both \emph{statistically} and \emph{computationally} efficient algorithms (e.g., upper confidence bound (UCB), \cite{azar2017minimax, jin2020provably, cai2020provably, zhou2021nearly}) with sublinear online regret for tabular and linear Markov decision processes (MDPs). 
But when it comes to MDPs with general function approximations, most of them involve impractical algorithmic components to incentivize exploration. 
Usually, to cope with general function approximations, agents need to solve constrained optimization problems within data-dependent level-sets \citep{jin2021bellman, du2021bilinear}, or sample from complicated 
posterior distributions over the space of hypotheses 
\citep{dann2021provably, agarwal2022model, zhong2022posterior}, both of which pose considerable  challenges for implementation. 
From a practical perspective, a prevalent approach in deep RL for  balancing exploration and exploitation is to use an ensemble of   neural networks 
\citep{wiering2008ensemble,osband2016deep,chen2017ucb,lu2017ensemble,kurutach2018model,chua2018deep,lee2021sunrise}, which serves as an empirical approximation of the UCB method. 
However, such an ensemble method suffers from high computational cost and lacks a theoretical guarantee when the underlying MDP is neither linear nor tabular. 
As for other deep RL algorithms for exploration \citep{haarnoja2018off,aubret2019survey,burda2018exploration,bellemare2016unifying,choi2018contingency}, such as the curiosity-driven method \citep{pathak2017curiosity}, it also remains unknown in theory whether they are provably sample-efficient in the context of general function approximations.

Hence, in this paper, we are aimed at tackling these issues and
answering the following question: 
\begin{center}
    \emph{Under general function approximation, can we design a sample-efficient and \\ easy-to-implement RL framework to 
    trade off between exploration and exploitation? }
\end{center}
Towards this goal, we propose an easy-to-implement RL framework, \textit{\underline{M}aximize to \underline{Ex}plore} (\texttt{MEX}), as an affirmative answer to above question. 
In order to strike a balance between exploration and exploitation, \texttt{MEX} propose to maximize a weighted sum of two objectives: 
\textcolor{blue}{\bf (a)} the optimal expected total return associated with a given hypothesis, 
and \textcolor{blue}{\bf (b)} the negative estimation error of that hypothesis.
Consequently, \texttt{MEX} naturally combines planning and estimation components in just one single objective.
By choosing the hypothesis that maximizes the weighted sum and executing the optimal policy with respect to the chosen hypothesis, \texttt{MEX} automatically balances between exploration and exploitation.

We highlight that the objective of \texttt{MEX} is \emph{not} obtained through the Lagrangian duality of the constrained optimization objective within data-dependent level-sets \citep{jin2021bellman, du2021bilinear,chen2022general}.
This is because the coefficient of the weighted sum, which remains fixed, is data-independent and predetermined for all episodes. 
Contrary to the Lagrangian duality, \texttt{MEX} does not necessitate an inner loop of optimization for dual variables, thereby circumventing the complications associated with minimax optimization.
As a maximization-only framework, \texttt{MEX} is friendly to implementations with neural networks and does not rely on sampling or ensemble.

In the theory part, we prove that \texttt{MEX} achieves a sublinear  $\tilde{\cO}(\texttt{Poly}(H)d_{\text{GEC}}^{1/2}(1/\sqrt{HK})K^{1/2})$ regret under mild structural assumptions and is thus sample-efficient. 
Here $K$ is the number of episodes, $H$ is the horizon length, and $d_{\text{GEC}}(\cdot)$ is the \textbf{\underline{G}eneralized \underline{E}luder \underline{C}oefficient} (GEC) \citep{zhong2022posterior} that characterizes the complexity of learning the underlying MDP using general function approximations in the online setting. 
Because the class of low-GEC MDPs includes almost all known theoretically tractable MDP instances, our result can be tailored to a multitude of specific settings with either a model-free or  a model-based hypothesis, such as MDPs with low Bellman eluder dimension \citep{jin2021bellman}, MDPs of bilinear class \citep{du2021bilinear}, and MDPs with low witness rank \citep{sun2019model}.
Thanks to the flexibility of the \texttt{MEX} framework, we further extend it to online RL in two-player zero-sum Markov games (MGs), for which we also generalize the definition of GEC to two-player zero-sum MGs and establish the sample efficiency with general function approximations. 
Finally, as the low-GEC class also contains many tractable Partially Observable MDP (POMDP) classes \citep{zhong2022posterior}, \texttt{MEX} can also be applied to these POMDPs.

Moving beyond theory and into practice, we adapt famous RL baselines \texttt{TD3} \citep{fujimoto2018addressing} and \texttt{MBPO} \citep{janner2019trust} to design practical versions of \texttt{MEX} in model-free and model-based fashion, respectively. 
On various MuJoCo environments \citep{todorov2012mujoco} with sparse rewards, experimental results show that \texttt{MEX} outperforms baselines steadily and significantly. Compared with other deep RL algorithms, \texttt{MEX} has low computational overhead and easy implementation while maintaining a theoretical guarantee.

\subsection{Main Contributions} 

We conclude our main contributions from the following three perspectives.
\vspace{1mm}
\begin{enumerate}
    \item We propose an easy-to-implement RL algorithm framework \texttt{MEX} that \emph{unconstrainedly} maximizes a single objective to fuse estimation and planning, automatically trading off between exploration and exploitation. 
        Under mild structural assumptions, we prove that \texttt{MEX} achieves a sublinear regret 
        \begin{align*}
            \tilde{\cO}\Big(\texttt{Poly}(H)\cdot d_{\text{GEC}}(1/\sqrt{HK})^{\frac{1}{2}}\cdot K^{\frac{1}{2}}\Big)
        \end{align*}
        with general function approximators, and thus is sample-efficient. 
        Here $K$ denotes the number of episodes, $\texttt{Poly}(H)$ is a polynomial term in horizon length $H$ which is specified in Section \ref{sec: examples mdp}, $d_{\text{GEC}}(\cdot)$ is the Generalized Eluder Coefficient (GEC) \citep{zhong2022posterior} of the underlying MDP.
    \item We instantiate the generic \texttt{MEX} framework to solve several model-free and model-based MDP instances and establish corresponding theoretical results. 
        Beyond MDPs, we further extend the \texttt{MEX} framework to two-player zero-sum MGs and also prove the sample efficiency with an extended definition of GEC.
    \item We design deep RL implementations of \texttt{MEX} in both model-free and model-based styles. 
        Experiments on various MuJoCo environments with sparse rewards demonstrate the effectiveness of \texttt{MEX} framework. 
\end{enumerate}

\subsection{Related Works}

\paragraph{Sample-efficient RL with function approximation.} 
The success of DRL methods has motivated a line of works focused on function approximation scenarios. 
This line of works is originated in the linear function approximation case \citep{wang2019optimism, yang2019sample, cai2020provably, jin2020provably, zanette2020frequentist, ayoub2020model, yang2020provably, modi2020sample, zhou2021nearly, zhong2023theoretical} and is later extended to general function approximations. 
\citet{wang2020reinforcement} first study the general function approximation using the notion of eluder dimension \citep{russo2013eluder}, which takes the linear MDP \citep{jin2020provably} as a special case but with inferior results. 
\citet{zanette2020learning} consider a different type of framework based on Bellman completeness, which assumes that the class used for approximating the optimal Q-functions is closed in terms of the Bellman operator and improves the results for linear MDP. 
After this, \citet{jin2021bellman} consider the eluder dimension of the class of Bellman residual associated with the RL problems, which captures more solvable problems (low Bellman eluder (BE) dimension). 
Another line of works focuses on the low-rank structures of the problems, where \citet{jiang@2017} propose the Bellman rank for model-free RL and \citet{sun2019model} propose the witness rank for model-based RL. 
Following these two works, \citet{du2021bilinear} propose the bilinear class, which contains more MDP models with low-rank structures \citep{azar2017minimax,sun2019model,jin2020provably,modi2020sample,cai2020provably,zhou2021nearly} by allowing a flexible choice of discrepancy function class.
However, it is known that neither BE nor bilinear class captures each other. 
\citet{dann2021provably} first consider eluder-coefficient-type complexity measure on the Q-type model-free RL. 
It was later extended by \citet{zhong2022posterior} to cover all the above-known solvable problems in both model-free and model-based manners. 
\citet{foster2021statistical, foster2023tight} study another notion of complexity measure, the decision-estimation coefficient (DEC), which also unifies the BE dimension and bilinear class and is appealing due to the matching lower bound in some decision-making problems but may not be applied to the classical optimism-based or sampling-based methods due to the presence of a minimax subroutine in the definition. 
\citet{chen2022unified, foster2022note} extend the vanilla DEC by incorporating an optimistic modification. 
\citet{chen2022general} study Admissible Bellman Characterization
(ABC) class to generalize BE. They also extend the GOLF algorithm \citep{jin2021bellman} and the Bellman completeness in model-free RL by considering more general (vector-form) discrepancy loss functions and obtaining sharper bounds in some problems. 
\citet{xie2022role} connect the online RL with the coverage condition in the offline RL, and also study the GOLF algorithm proposed in \citet{jin2021bellman}. 

\paragraph{Algorithmic design in sample-efficient RL with function approximation.} 
The most prominent approach in this area is based on the principle of ``Optimism in the Face of Uncertainty'' (OFU), which dates back to \citet{auer2002finite}. 
For instance, for linear function approximation, \citet{jin2020provably} propose an optimistic variant of Least-Squares Value Iteration (LSVI), which achieves optimism by adding a bonus at each step. 
For the general case, \citet{jiang2017contextual} first propose an elimination-based algorithm with optimism in model-free RL and is extended to model-based RL by \cite{sun2019model}. 
After these, \citet{du2021bilinear, jin2021bellman} propose two OFU-based algorithms, which are more similar to the lin-UCB algorithm \citep{abbasi2011improved} studied in the linear contextual bandit literature.
The model-based counterpart (Optimistic Maximum Likelihood Estimation (OMLE)) is studied in \citet{liu2022partially, chen2022unified}. 
Specifically, these algorithms explicitly maintain a confidence set that contains the ground truth with high probability and conducts a constrained optimization step to select the most optimistic hypothesis in the confidence set. 
The other line of works studies another powerful algorithmic framework based on posterior sampling. 
For instance, \citet{zanette2020frequentist} study randomized LSVI (RLSVI), which can be interpreted as a sampling-based algorithm and achieves an order-optimal result for linear MDPs. 
For general function approximations, the works mainly follow the idea of the ``feel-good'' modification of the Thompson sampling algorithm \citep{thompson1933likelihood} proposed in \citet{zhang2022feel}. 
These algorithms start from some prior distribution over the hypothesis space and update the posterior distribution according to the collected samples but with certain optimistic modifications in either the prior or the loglikelihood function.
Then the hypothesis for each iteration is sampled from the posterior and guides data collection. 
In particular, \citet{dann2021provably} study the model-free Q-type problem, and \citet{agarwal2022model} study the model-based problems, but under different notions of complexity measures.
\citet{zhong2022posterior} further utilize the idea in \citet{zhang2022feel} and extend the posterior sampling algorithm in \citet{dann2021provably} to be a unified sampling-based framework to solve both model-free and model-based RL problems, which is also shown to apply to the more challenging partially observable setting. 
In addition to the OFU-based algorithm and the sampling-based framework, \citet{foster2021statistical} propose the Estimation-to-Decisions (E2D) algorithm, which can solve problems with low Decision-Estimation Coefficient (DEC) but requires solving a complicated minimax subroutine to fit in the framework of DEC.

\paragraph{Relationship with reward-biased maximum likelihood estimation.} Our work is also related to a line of work in reward-biased maximum likelihood estimation. 
While \cite{r1_kumar1982new} firstly proposed an estimation criterion that biases maximum likelihood estimation (RBMLE) with the cost or the value, their algorithm is actually different from ours, by their Equation (6) and (8) in Section 3, their algorithm performs the estimation of model 
 and policy optimization separately, for which they only obtained asymptotic convergence guarantees. Also, how well their decision rule explores remains unknown in theory. In contrast, MEX adopts a single optimization objective that combines estimation with policy optimization, which also ensures sample-efficient online exploration.
\cite{r2_liu2020exploration,r3_hung2021reward,r4_mete2021reward,r5_mete2022rbmle,r6_mete2022augmented} study RBMLE in Multi-arm bandit \citep{r2_liu2020exploration}, Linear Stochastic Bandits \citep{r3_hung2021reward}, tabular RL \citep{r4_mete2021reward}, and Linear Quadratic Regulator settings (linear parameterized models of MDPs, \citep{r5_mete2022rbmle,r6_mete2022augmented}) and also obtain the theoretical guarantees. While these settings are special cases for our proposed algorithms,  our proven theoretical guarantee can also be generalized to these concrete cases. As we claim in this paper, our main contribution is to address the exploration-exploitation trade-off issue under general function approximation, which makes our work differ from these papers. 
\cite{wu2022bayesian} consider an algorithm similar to MEX, but our theory differs from theirs in both techniques and results. Our theory is based upon a unified framework of online RL with general function approximations, which covers their setup for the model-based hypothesis with kernel function approximation (RKHS). More importantly, they derived asymptotic regret of their algorithm based upon certain uniform boundedness and asymptotic normality assumptions, which are relatively strong conditions. In contrast, we derive finite sample regret upper bound for MEX, and the only fundamental assumption needed is a lower Generalized Eluder Coefficient (GEC) MDP, which contains almost all known theoretically tractable MDP classes (therefore covers their RKHS model). Finally, our paper further extends MEX to two-player zero-sum Markov games where similar algorithms and theories are previously unknown to the best of our knowledge. Moreover, the works mentioned above do not implement experiments in deep RL environments, while we propose deep RL implementations and demonstrate their effectiveness in several MuJoco tasks.
\paragraph{Exploration in deep RL.} 
There has also been a long line of works that studies the exploration-exploitation trade-off from a practical perspective, where a prominent approach is referred to as the curiosity-driven method \citep{pathak2017curiosity}. 
Curiosity-driven method focuses on the intrinsic rewards \citep{pathak2017curiosity} (to handle the sparse extrinsic reward case) when making decisions, whose formulation can be largely grouped into either encouraging the algorithm to explore ``novel'' states \citep{bellemare2016unifying,lopes2012exploration} or encouraging the algorithm to pick actions that reduce the uncertainty in its knowledge of the environment \citep{houthooft2016vime, mohamed2015variational, stadie2015incentivizing}. 
These methods share the same theoretical motivation as the OFU principle. 
In particular, one popular approach in this area is to use ensemble methods, which combine multiple neural networks of the value function and (or) policy (see \citep{wiering2008ensemble,osband2016deep,chen2017ucb,lu2017ensemble,kurutach2018model,chua2018deep,lee2021sunrise} and reference therein). 
For instance, \citet{chen2017ucb} leverage the idea of upper confidence bound by estimating the uncertainty via ensembles to improve the sample efficiency. 
However, the uncertainty estimation via ensembles is more computationally inefficient as compared to the vanilla algorithm. 
Meanwhile, these methods lack theoretical guarantees beyond tabular and linear settings. It remains unknown in theory whether they are provably sample-efficient in the context of general function approximations. There is a rich body of literature, and we refer interested readers to Section 4 of \citet{zha2021rank} for a comprehensive review.

\paragraph{Two-player zero-sum Markov game.} 
There have been numerous works on designing provably efficient algorithms for zero-sum Markov games (MGs). 
In the tabular case, \citet{bai2020near,bai2020provable,liu2020sharp} propose algorithms with regret guarantees polynomial in the number of states and actions. 
\citet{xie2020learning, chen2021almost} then study the MGs in the linear function approximation case and design algorithms with a $\tilde{\cO}(\text{poly}(d, H) \sqrt{K})$ regret, where $d$ is the dimension of the linear features. 
These approaches are later extended to general function approximations by \citet{jin2021power, huang2021towards, xiong22b}, where the former two works studied OFU-based algorithms and the last one studied posterior sampling.

\subsection{Notations and Outlines}
For a measurable space $\mathcal{X}$, we use $\Delta(\mathcal{X})$ to denote the set of probability measure on $\mathcal{X}$. 
For an integer  $n\in\mathbb{N}$, we use $[n]$ to denote the set $\{1,\cdots,n\}$.
For a random variable $X$, we use $\mathbb{E}[X]$ and $\mathbb{V}[X]$ to denote its expectation and variance respectively.
For two probability densities on $\mathcal{X}$, we denote their Hellinger distance $D_{\mathrm{H}}$ as
\begin{align*}
    D_{\mathrm{H}}(p\|q) = \frac{1}{2}\int_{\mathcal{X}}\big(\sqrt{p(x)} - \sqrt{q(x)}\big)^2\mathrm{d}x.
\end{align*}
For two functions $f(x)$ and $g(x)$, we denote $f\lesssim g$ if there is a constant $C$ such that $f(x)\leq C\cdot g(x)$ for any $x$.

The paper is organized as follows. 
In Section \ref{sec: pre}, we introduce the basics of online RL in MDPs, where we also define the settings for general function approximations. 
In Section \ref{sec: alg}, we propose the \texttt{MEX} framework, and we provide generic theoretical guarantees for \texttt{MEX} in Section \ref{sec: theory}.
In Section \ref{sec: examples mdp}, we instantiate \texttt{MEX} to solve several model-free and model-based MDP instances, with some details referred to Appendix \ref{sec: proof examples mdp}.
We further extend the algorithm and the theory of \texttt{MEX} to zero-sum two-player MGs in Section \ref{sec: mg}.
In Section \ref{sec: experiments}, we conduct deep RL experiments to demonstrate the effectiveness of \texttt{MEX} in various MuJoCo environments.

%% file: tex/preliminaries.tex
\section{Preliminaries}\label{sec: pre}

\subsection{Episodic Markov Decision Process and Online Reinforcement Learning}

We consider an episodic MDP defined by a tuple $(\cS, \cA, H, \mathbb{P}, r)$, where $\cS$ and $\cA$ are the state and action spaces, $H\in\mathbb{N}_+$ is a finite horizon, $\mathbb{P} = \{\mathbb{P}_h\}_{h\in[H]}$ with $\mathbb P_h:\mathcal{S}\times\mathcal{A}\mapsto\Delta(\mathcal{S})$ the transition kernel at the $h$-th timestep, and $r = \{r_h\}_{h\in[H]}$ with $r_h\colon \cS\times \cA\to [0,1]$ the reward function at the $h$-th timestep. 
Without loss of generality, we assume that the reward function $r$ is both deterministic and known by the learner.

We consider \emph{online} reinforcement learning in the episodic MDP, where the agent interacts with the MDP for $K\in \mathbb{N}_+$ episodes through the following protocol.  
At the beginning of the $k$-th episode, the agent selects a policy $\pi^k = \{\pi^k_h:\mathcal{S}\mapsto\Delta(\mathcal{A})\}_{h\in[H]}$.
Then at the $h$-th timestep of this episode, the agent is at some state $x_h^k$ and it takes an action $a_h^k \sim \pi_h^k(\cdot\given x_h^k)$.
After receiving the reward $r_h^k = r_h(x_h^k, a_h^k)$, it transits to the next state $x_{h+1}^k \sim \mathbb{P}_h(\cdot\given x_h^k, a_h^k)$.
When it reaches the state $x_{H+1}^k$, it ends the $k$-th episode. 
Without loss of generality, we assume that the initial state $x_1^k = \underline{x}$ is fixed all $k\in[K]$.
Our algorithm and analysis can be directly generalized to the setting where $x_1$ is sampled from a distribution on $\cS$.  

\paragraph{Policy and value functions.} 
For any given policy $\pi = \{\pi_h:\mathcal{S}\mapsto\Delta(\mathcal{A})\}_{h\in[H]}$, we denote by $V_h^\pi: \mathcal{S} \mapsto\mathbb{R}_+$ and $Q_h^\pi: \mathcal{S}\times\mathcal{A} \mapsto\mathbb{R}_+$ its state-value function and its state-action value function at the $h$-th timestep, which characterize the expected total rewards received by executing the policy $\pi$ starting from some $x_h=x\in\cS$ (or $x_h = x\in\cS, a_h = a\in\cA$, resp.), till the end of the episode. 
Specifically, for any $(x,a)\in\cS\times\cA$,
\begin{align}
    V_h^\pi(x):=\mathbb{E}_{\mathbb{P},\pi}\left[\sum_{h^{\prime}=h}^H r_{h^{\prime}}(x_{h^{\prime}}, a_{h^{\prime}}) \middle|\, x_h=x\right], \quad Q_h^\pi(x, a):=\mathbb{E}_{\mathbb{P},\pi}\left[\sum_{h^{\prime}=h}^H r_{h^{\prime}}(x_{h^{\prime}}, a_{h^{\prime}}) \middle|\, x_h=x, a_h=a\right].
\end{align}
It is known that there exists an optimal policy, denoted by $\pi^{\ast}$, which has the optimal state-value function for all initial states \citep{puterman2014markov}.
That is, $V_h^{\pi^{\ast}}(x)=\sup _\pi V_h^\pi(x)$ for all $h \in[H]$ and $x \in \mathcal{S}$. For simplicity, we abbreviate $V^{\pi^{\ast}}$ as $V^{\ast}$ and the optimal state-action value function $Q^{\pi^{\ast}}$ as $Q^{\ast}$. 
Moreover, the optimal value functions $Q^{\ast}$ and $V^\ast$ satisfy the following Bellman optimality equation \citep{puterman2014markov},
\begin{align}
    V_{h}^\ast (x)= \max_{a\in\mathcal{A}} Q_{h}^\ast(x,a),\quad Q_h^{\ast}(x, a)=(\mathcal{T}_h Q_{h+1}^{\ast})(x, a):=r_h(x, a)+\mathbb{E}_{x^{\prime} \sim \mathbb{P}_h(\cdot \mid x, a)} \Big[\max_{a'\in\mathcal{A}}Q_{h+1}^{\ast}\left(x^{\prime},a'\right)\Big] ,
\end{align}
with $Q_{H+1}^{\ast}(\cdot,\cdot) = 0$ for all $(x, a, h) \in \mathcal{S} \times \mathcal{A} \times[H]$. 
We call $\mathcal{T}_h$ the Bellman optimality operator at timestep $h$. 
Also, for any two functions $Q_h$ and $Q_{h+1}$ on $\cS\times\cA$, we define
\begin{align}
    \mathcal{E}_h(Q_h,Q_{h+1};x,a):=Q_h(x,a) - \mathcal{T}_hQ_{h+1}(x,a),\quad\forall (x,a)\in\cS\times\cA,
\end{align}
as the Bellman residual at timestep $h$ of $(Q_h,Q_{h+1})$.

\paragraph{Performance metric.} 
We measure the performance of an online RL algorithm after $K$ episodes by its \emph{regret}.
We assume that the learner predicts the optimal policy $\pi^{\ast}$ via $\pi^k$ in the $k$-th episode for each $k\in[K]$. 
Then the regret after $K$ episodes is defined as the cumulative suboptimality gap of $\{\pi^k\}_{k\in[K]}$\footnote{We allow the agent to predict the optimal policy via $\pi^k$ while executing some other exploration policy $\pi^k_{\exp}$ to interact with the environment and collect data, as is considered in the related literature \citep{sun2019model,du2021bilinear, zhong2022posterior}}, defined as
\begin{align}\label{eq: regret}
    \text{Regret}(K) =  \sum_{k=1}^K  V^\ast_1(x_1) - V^{\pi^k}_1(x_1).
\end{align}
The target of sample-efficient online RL is to achieve sublinear regret \eqref{eq: regret} with respect to $K$.

\subsection{Function Approximation: Model-Free and Model-Based Hypothesis} 
To deal with MDPs with large or even infinite state space $\mathcal{S}$, we introduce a class of function approximators. 
In specific, we consider an abstract hypothesis class $\mathcal{H} = \mathcal{H}_1\times\cdots\times\mathcal{H}_H$, which can be specified to model-based and model-free settings, respectively.
Also, we denote $\Pi = \Pi_1\times\cdots\times\Pi_H$ as the space of all Markovian policies.

The following two examples show how to specify $\mathcal{H}$ for model-free and model-based settings.

\begin{example}[Model-free hypothesis class]\label{exp: model free}
    For model-free setting, $\mathcal{H}$ contains approximators of the optimal state-action value function of the MDP, i.e., $\mathcal{H}_h\subseteq\{f_h:\mathcal{S}\times\mathcal{A}\mapsto\mathbb{R}\}$.
    For any $f = (f_1,\cdots,f_H)\in\mathcal{H}$:
    \begin{enumerate}
        \item we denote corresponding state-action value function $Q_{f} = \{Q_{h,f}\}_{h\in[H]}$ with $Q_{h,f} = f_h$;
        \item we denote corresponding state-value function $V_f = \{V_{h,f}\}_{h\in[H]}$ with $V_{h,f}(\cdot) = \max_{a\in\mathcal{A}}Q_{h,f}(\cdot,a)$, and we denote the corresponding optimal policy by $\pi_f = \{\pi_{h,f}\}_{h\in[H]}$ with $\pi_{h,f}(\cdot) = \arg\max_{a\in\mathcal{A}}Q_{h,f}(\cdot,a)$.
        \item we denote the optimal state-action value function under the true model, i.e.,  $Q^{\ast}$, by $f^{\ast}$.
    \end{enumerate}
\end{example}

\begin{example}[Model-based hypothesis class]\label{exp: model based}
    For model-based setting, $\mathcal{H}$ contains approximators of the transition kernel of the MDP, for which we denote $f=\mathbb{P}_f=(\mathbb{P}_{1,f},\cdots,\mathbb{P}_{H,f}) \in \mathcal{H}$.
    For any $(f,\pi)\in\mathcal{H}\times\Pi$:
    \begin{enumerate}
        \item we denote $V_{f}^{\pi} = \{V_{h,f}^{\pi}\}_{h\in[H]}$ as the state-value function induced by model $\mathbb{P}_f$ and policy $\pi$.
        \item we denote $V_{f} = \{V_{h,f}\}_{h\in[H]}$ as the optimal state-value function under model $\mathbb{P}_f$, i.e., $V_{h, f} = \sup_{\pi\in\Pi}V_{h, f}^{\pi}$.
        The corresponding optimal policy is denoted by $\pi_{f} = \{\pi_{h,f}\}_{h\in[H]}$, where $\pi_{h,f} = \arg\sup_{\pi\in\Pi}V_{h, f}^{\pi}$.
        \item we denote the true model $\mathbb{P}$ of the MDP as $f^{\ast}$.
    \end{enumerate} 
\end{example}

We remark that the main difference between the model-based hypothesis (Example \ref{exp: model based}) and the model-free hypothesis (Example \ref{exp: model free}) is that model-based RL directly learns the transition kernel of the underlying MDP, while model-free RL  learns the optimal state-action value function. 
Since we do not add any specific structural form to the hypothesis class, e.g., linear function or kernel function, we are in the context of \emph{general function approximations} \citep{sun2019model, jin2021bellman, du2021bilinear, zhong2022posterior, chen2022general}.

%% file: tex/algorithm.tex
\section{Algorithm Framework: Maximize to Explore (MEX)}\label{sec: alg}

In this section, we propose an algorithm framework, named \textit{\underline{M}aximize to \underline{Ex}plore} (\texttt{MEX}, Algorithm \ref{alg: implicit optimism mdp}), for online RL in MDPs with general function approximations. 
With a novel single objective, \texttt{MEX} automatically balances the goal of exploration and exploitation in online RL. 
Since \texttt{MEX} only requires an \emph{unconstrained} maximization procedure, it is friendly to implement in practice.

We first give a generic algorithm framework and then instantiate it to model-free (Example \ref{exp: model free}) and model-based (Example \ref{exp: model based}) hypotheses respectively.

\paragraph{Generic algorithm.} 
In each episode $k\in[K]$, the agent first estimates a hypothesis $f^k\in\mathcal{H}$ using historical data $\{\mathcal{D}^{s}\}_{s=1}^{k-1}$ by maximizing a composite objective \eqref{eq: implicit optimism mdp}. 
Specifically, in order to achieve exploiting history knowledge while encouraging exploration, the agent considers a single objective that sums:
\textbf{(a)} the negative loss $ - L_h^{k-1}(f)$ induced by the hypothesis $f$, which represents the exploitation of the agent’s current knowledge;
\textbf{(b)} the expected total return of the optimal policy associated with this hypothesis, i.e., $V_{1,f}$, which represents exploration for a higher return. 
With a tuning parameter $\eta>0$, the agent balances the weight put on the tasks of exploitation and exploration.

Then the agent predicts $\pi^{\ast}$ via the optimal policy associated with the hypothesis $f^k$, i.e., $\pi_{f^k}$.
Also, the agent executes some exploration policy $\pi_{\exp}(f^k)$ to collect data $\mathcal{D}^k = \{(x_h^k,a_h^k,r_h^k, x_{h+1}^k)\}_{h=1}^H$ and updates the loss function $L_h^k(\cdot)$. 
The choice of the loss function $L(\cdot)$ varies between model-free and model-based hypotheses, which we specify in the following.
The choice of the exploration policy $\pi_{\exp}(f^k)$ depends on the specific MDP structure, and we refer to examples in Section \ref{sec: examples mdp} and Appendix \ref{sec: proof examples mdp} for detailed discussions.

We need to highlight that \texttt{MEX} is not a Lagrangian duality of the constrained optimization objectives within data-dependent level-sets proposed by previous works \citep{jin2021bellman, du2021bilinear,chen2022general}.
In fact, \texttt{MEX} only needs to fix the parameter $\eta$ across each episode $k$. 
Thus $\eta$ is independent of data and predetermined,  which contrasts Lagrangian methods that involve an inner loop of optimization for the dual variables. We also remark that we can rewrite \eqref{eq: implicit optimism mdp} as a joint optimization 
$
(f,\pi) = \operatornamewithlimits{argsup}_{f\in\mathcal{H},\pi\in\Pi} V_{1,f}^{\pi}(x_1) -\eta \sum_{h=1}^H L_h^{k-1}(f).
$ 
When $\eta$ tends to infinity, \texttt{MEX} conincides with the vanilla actor-critic framework \citep{NIPS1999_6449f44a}, where the critic $f$ minimizes the estimation error and the actor $\pi$ conducts greedy policy associated with the critic $f$.
In the following two parts, we instantiate Algorithm~\ref{alg: implicit optimism mdp} to model-based and mode-free hypotheses respectively by specifying the loss function $L_h^k(f)$. 

\begin{algorithm}[t]
	\caption{Maximize to Explore 
 (\texttt{MEX})}
	\label{alg: implicit optimism mdp}
	\begin{algorithmic}[1]
	\STATE \textbf{Input}: Hypothesis class $\mathcal{H}$, parameter $\eta>0$.%loss function $\mathcal{L} = \{\mathcal{L}_h\}_{h\in[H]}$, 
    \FOR{$k=1,\cdots,K$}
        \STATE Solve $f^k\in\mathcal{H}$ via
        \begin{align}\label{eq: implicit optimism mdp}
            f^k = \operatornamewithlimits{argsup}_{f\in\mathcal{H}}\left\{V_{1, f}(x_1) - \eta \cdot  \sum_{h=1}^HL_h^{k-1} (f)\right\}.
        \end{align}
        \STATE Execute $\pi_{\exp}(f^k)$ to collect data $\mathcal{D}^k = \{\mathcal{D}_h^k\}_{h\in[H]}$ with $\mathcal{D}_h^k = (x_h^k,a_h^k,r_h^k,x_{h+1}^k)$.
        \STATE Calculate the loss function $L_h^k(\cdot)$ for each $h\in[H]$ based on historical data $\{\mathcal{D}^s\}_{s\in[k]}$.
        \STATE Predict the optimal policy via $\pi_{f^k}$.
    \ENDFOR
	\end{algorithmic}
\end{algorithm}

\paragraph{Model-free algorithm.} 
For model-free hypothesis (Example \ref{exp: model free}), the composite objective \eqref{eq: implicit optimism mdp} becomes
\begin{align}
\label{model_free}
    f^k = \operatornamewithlimits{argsup}_{f\in\mathcal{H}}\left\{\max_{a_1\in\mathcal{A}}Q_{1, f}(x_1,a_1) - \eta \cdot \sum_{h=1}^H   L_h^{k-1}(f)\right\}.
\end{align}
Regarding the choice of the loss function, for seek of theoretical analysis, to deal with MDPs with low Bellman eluder dimension \citep{jin2021bellman} and MDPs of bilinear class \citep{du2021bilinear}, we  assume the existence of certain function $l$, which generalizes the notion of Bellman residual.

\begin{assumption}\label{ass: l function}
    The function $l:\mathcal{H}\times\mathcal{H}_h\times\mathcal{H}_{h+1}\times(\mathcal{S}\times\mathcal{A}\times\mathbb{R}\times\mathcal{S})\mapsto\mathbb{R}$ satisfies\footnote{For simplicity we drop the dependence of $l$ on the index $h$ since this makes no confusion. Similar simplications are used later.}: 
    \begin{enumerate}
        \item $\mathrm{(Generalized\,\,Bellman\,\,completeness)}$ \citep{zhong2022posterior,chen2022general}. There exists a functional operator $\mathcal{P}_{h}:\mathcal{H}_{h+1}\mapsto\mathcal{H}_{h}$ such that for any $(f',f_h,f_{h+1})\in\mathcal{H}\times\mathcal{H}_h\times\mathcal{H}_{h+1}$ and $\mathcal{D}_h= (x_h,a_h,r_h,x_{h+1})\in\mathcal{S}\times\mathcal{A}\times\mathbb{R}\times\mathcal{S}$,
    \begin{align*}
        l_{f'}\big((f_{h},f_{h+1});\mathcal{D}_h\big) - l_{f'}\big((\mathcal{P}_hf_{h+1},f_{h+1});\mathcal{D}_h\big) = \mathbb{E}_{x_{h+1}\sim\mathbb{P}_h(\cdot|x_h,a_h)}\big[l_{f'}\big((f_{h},f_{h+1});\mathcal{D}_h\big)\big],
    \end{align*}
    where we require that $\mathcal{P}_hf^{\ast}_{h+1} = f^{\ast}_h$ and that $\mathcal{P}_hf_{h+1}\in\mathcal{H}_h$ for any $f_{h+1}\in\mathcal{H}_{h+1}$ and $h\in[H]$;
    \item $\mathrm{(Boundedness)}$. It holds that $|l_{f'}((f_h,f_{h+1});\mathcal{D}_h)|\leq B_l$ for some $B_l>0$ and any  $(f',f_h,f_{h+1})\in\mathcal{H}\times\mathcal{H}_h\times\mathcal{H}_{h+1}$ and $\mathcal{D}_h= (x_h,a_h,r_h,x_{h+1})\in\mathcal{S}\times\mathcal{A}\times\mathbb{R}\times\mathcal{S}$.
    \end{enumerate}
\end{assumption}

Intuitively, the operator $\mathcal{P}_h$ can be considered as a generalization of the Bellman optimality operator.
We set the choice of $l$ and $\mathcal{P}$ for concrete model-free examples in Section \ref{sec: examples mdp}.
We then set the loss function $L_h^{k}$ as an empirical estimation of the generalized squared Bellman error $|\mathbb{E}_{x_{h+1}\sim \mathbb{P}_h(\cdot|x_h,a_h)}[l_{f^s}((f_{h},f_{h+1}),\mathcal{D}_h^s )]|^2$, given by 
\begin{align}\label{eq: implicit optimism mdp free based L}
    L_h^k(f) = \sum_{s=1}^kl_{f^s}\bigl ((f_{h},f_{h+1} );\mathcal{D}_h^s\bigr )^2 - \inf_{f_h'\in\mathcal{H}_h}\sum_{s=1}^kl_{f^s} \bigl ((f_{h}',f_{h+1} );\mathcal{D}_h^s \bigr )^2.
\end{align}
We remark that the subtracted infimum term in \eqref{eq: implicit optimism mdp free based L} is for handling the variance terms in the estimation to achieve a fast theoretical rate.
Similar essential ideas are also adopted by \cite{jin2021bellman, xie2021bellman, dann2021provably, jin2022power, lu2022pessimism, agarwal2022model, zhong2022posterior}.

\paragraph{Model-based algorithm.} 
For model-based hypothesis (Example \ref{exp: model based}), the composite objective \eqref{eq: implicit optimism mdp} becomes
\begin{align}\label{eq: implicit optimism mdp model based}
    f^k = \operatornamewithlimits{argsup}_{f\in\mathcal{H}}\left\{\sup_{\pi\in\Pi}V_{1, \mathbb{P}_f}^\pi(x_1) - \eta \cdot \sum_{h=1}^HL_h^{k-1}(f)\right\},
\end{align}
which gives a joint optimization over the model $\mathbb{P}_f$ and the policy $\pi$.
In the model-based algorithm, we choose the loss function $L_h^k$ as the negative log-likelihood loss, defined as 
\begin{align}\label{eq: implicit optimism mdp model based L}
    L_h^k(f) = - \sum_{s=1}^k\log\mathbb{P}_{h,f}(x_{h+1}^s|x_h^s,a_h^s).
\end{align}

%% file: tex/theory.tex
\section{Regret Analysis for MEX Framework}\label{sec: theory}

In this section, we analyze the regret of the \texttt{MEX} framework (Algorithm \ref{alg: implicit optimism mdp}).
Specifically, we give an upper bound of its regret which holds for both model-free (Example \ref{exp: model free}) and model-based (Example \ref{exp: model based}) settings. 
To derive the theorem, we first present three key assumptions needed. 
In Section \ref{sec: examples mdp}, we specify the generic upper bound to specific examples of MDPs and hypothesis classes that satisfy these assumptions.

We first assume that the hypothesis class $\mathcal{H}$ is well-specified, containing the true hypothesis $f^{\ast}$.

\begin{assumption}[Realizablity]\label{ass: realizability}
    We assume that the true hypothesis $f^{\ast} \in \mathcal{H}$.
\end{assumption}

Moreover, we make a structural assumption on the underlying MDP to ensure sample-efficient online RL.
Inspired by \citet{zhong2022posterior}, we require the MDP to have low \textbf{\underline{G}eneralized \underline{E}luder \underline{C}oefficient} (GEC). 
In MDPs with low GEC, the agent can effectively mitigate out-of-sample prediction error by minimizing in-sample prediction error based on the historical data.  
Therefore, the GEC can be used to measure the difficulty inherent in  generalization from the observation to the unobserved trajectory, thus further quantifying  the hardness of learning the MDP.
We refer the readers to \citet{zhong2022posterior} for a detailed discussion of GEC.

To define GEC, we introduce a discrepancy function
\begin{align*}
    \ell_{f'}(f;\xi_h):\mathcal{H}\times\mathcal{H}\times(\mathcal{S}\times\mathcal{A}\times\mathbb{R}\times\mathcal{S})\mapsto\mathbb{R},
\end{align*}
which characterizes the error incurred by hypothesis $f\in\mathcal{H}$ on data $\xi_h = (x_h,a_h,r_h,x_{h+1})$. 
Specific choices of $\ell$ are given in Section \ref{sec: examples mdp} for concrete  model-free and model-based examples.

\begin{assumption}[Low generalized eluder coefficient \citep{zhong2022posterior}]\label{ass: gec}
 We assume that given an $\epsilon>0$,
 there exists $d(\epsilon)\in\mathbb{R}_+$, such that for any sequence of $\{f^k\}_{k\in[K]}\subseteq\mathcal{H}$, $\{\pi_{\mathrm{exp}}(f^k)\}_{k\in[K]}\subseteq\Pi$,
 \begin{align}
    \sum_{k=1}^KV_{1, f^k} - V_{1}^{\pi_{f^k}}\le \inf_{\mu > 0} \left\{\frac{\mu}{2} \sum_{h=1}^H\sum_{k=1}^K\sum_{s = 1}^{k-1} \mathbb{E}_{\xi_h\sim \pi_{\exp}(f^{s})}[
        \ell_{f^s}(f^k;\xi_h)]+ \frac{d(\epsilon)}{2\mu} + \sqrt{d(\epsilon)HK} + \epsilon HK \right\}. \notag
 \end{align}
We denote the smallest number $d(\epsilon)\in\mathbb{R}_+$ satisfying this condition as $d_{\mathrm{GEC}}(\epsilon)$.
\end{assumption}

As is shown by \citet{zhong2022posterior}, the low-GEC MDP class covers almost all known theoretically tractable MDP instances, such as linear MDP \citep{yang2019sample,jin2020provably}, linear mixture MDP \citep{ayoub2020model,modi2020sample,cai2020provably}, MDPs of low witness rank \citep{sun2019model}, MDPs of low Bellman eluder dimension \citep{jin2021bellman}, and MDPs of bilinear class \citep{du2021bilinear}. 

Finally, we make a concentration-style assumption which characterizes how the loss function $L_h^k$ is related to the expectation of the discrepancy function $\mathbb{E}[\ell]$ appearing in the definition of GEC.
For ease of presentation, we assume that $\cH$ is finite, i.e., $|\mathcal{H}|<\infty$, but our result can be directly extended to an infinite $\mathcal{H}$ using covering number arguments \citep{wainwright2019high,jin2021bellman, liu2022welfare, jin2022power}.

\begin{assumption}[Generalization]\label{ass: supervised learning}
We assume that $\cH$ is finite, i.e., $|\cH|<+\infty$, and that with probability at least $1-\delta$, for any episode $k\in[K]$ and hypothesis $f\in\mathcal{H}$, it holds that
\begin{align}
    \sum_{h=1}^HL_h^{k-1}(f^{\ast}) - L_h^{k-1}(f) \lesssim -\sum_{h=1}^H\sum_{s=1}^{k-1}\mathbb{E}_{\xi_h\sim \pi_{\mathrm{exp}}(f^s)}[\ell_{f^s}(f;\xi_h)] + B\cdot\big(H\log(HK/\delta) + \log(|\mathcal{H}|)\big), \notag
\end{align}
where $B = B_l^2$ for model-free hypothesis (see Assumption \ref{ass: l function}) and $B = 1$ for model-based hypothesis.
\end{assumption}

As we will show in Proposition \ref{prop: supervised learning model free rl} and Proposition \ref{prop: supervised learning model based rl}, Assumption \ref{ass: supervised learning} holds for both the model-free and model-based settings.
Such a concentration style inequality is well known in the literature and similar analysis is also adopted by \cite{jin2021bellman, chen2022general}.
With Assumptions \ref{ass: realizability}, \ref{ass: gec}, and \ref{ass: supervised learning}, we can present our main theoretical result. 

\begin{theorem}[Online regret of \texttt{MEX} (Algorithm \ref{alg: implicit optimism mdp})]\label{thm:reg}
    Under Assumptions \ref{ass: realizability}, \ref{ass: gec}, and \ref{ass: supervised learning}, by setting 
    $$
    \eta = \sqrt{\frac{d_{\mathrm{GEC}}(1/\sqrt{HK})}{(H\log(HK/\delta) + \log(|\mathcal{H}|))\cdot B\cdot K}},
    $$ 
    then the regret of Algorithm \ref{alg: implicit optimism mdp} after $K$ episodes is upper bounded by
    \begin{align*}
        \mathrm{Regret}(K)\lesssim \sqrt{d_{\mathrm{GEC}}(1/\sqrt{HK})\cdot (H\log(HK/\delta) + \log(|\mathcal{H}|)) \cdot B\cdot K},
    \end{align*}
    with probability at least $1-\delta$. 
    Here $d_{\mathrm{GEC}}(\cdot)$ is defined in Assumption \ref{ass: gec}.
\end{theorem} 

\begin{proof}[Proof of Theorem \ref{thm:reg}]
    See Appendix \ref{subsec: proof regret mdp} for a detailed proof.
\end{proof}

By Theorem \ref{thm:reg}, the regret of Algorithm \ref{alg: implicit optimism mdp} scales with the square root of the number of episodes $K$ and the polynomials of the horizon $H$, the GEC $d_{\text{GEC}}(1/\sqrt{K})$, and the log of the hypothesis class cardinality $\log|\mathcal{H}|$. 
When the number of episodes $K$ tends to infinity, the average regret $\mathrm{Regret}(K)/K$ vanishes, meaning that the output policy of Algorithm \ref{alg: implicit optimism mdp} is approximately optimal. 
Thus Algorithm \ref{alg: implicit optimism mdp} is provably sample-efficient.

Besides, as we can see in Theorem \ref{thm:reg} and its specifications in Section \ref{sec: examples mdp}, \texttt{MEX} matches existing theoretical results in the literature of online RL under general function approximations \citep{jiang2017contextual,sun2019model, du2021bilinear, jin2021bellman, dann2021provably, agarwal2022model, zhong2022posterior}.
But meanwhile, \texttt{MEX} does not require explicitly solving a constrained optimization problem within  data-dependent level-sets or performing a complex sampling procedure, as is required by previous theoretical algorithms.
This advantage makes \texttt{MEX} a principled approach with much easier practical implementations. 
We conduct deep RL experiments for \texttt{MEX} in Section~\ref{sec: experiments} to demonstrate its power in complicated online tasks.

Finally, thanks to the simple and flexible form of \texttt{MEX}, in Section \ref{sec: mg}, we further extend this framework and its analysis to two-player zero-sum Markov games (MGs), for which we also extend the definition of generalized eluder coefficient (GEC) to two-player zero-sum MGs.  
Moreover, a vast variety of tractable partially observable problems also enjoy low GEC \citep{zhong2022posterior}, including regular PSR \citep{zhan2022pac}, weakly revealing POMDPs \citep{jin2020sample}, low rank POMDPs \citep{wang2022embed}, and PO-bilinear class POMDPs \citep{uehara2022provably}. 
We believe that our proposed \texttt{MEX} framework can also be applied to solve these POMDPs.

%% file: tex/examples.tex
\section{Examples of MEX Framework}\label{sec: examples mdp}

In this section, we specify Algorithm \ref{alg: implicit optimism mdp} to model-based and model-free hypothesis classes for various examples of MDPs of low GEC (Assumption \ref{ass: gec}), including MDPs with low witness rank \citep{sun2019model}, MDPs with low Bellman eluder dimension \citep{jin2021bellman}, and MDPs of bilinear class \citep{du2021bilinear}.
Meanwhile, we show that Assumption \ref{ass: supervised learning} (generalization) holds for both model-free and model-based settings.
It is worth highlighting that for both model-free and model-based hypotheses, we provide generalization guarantees in a neat and unified manner, independent of specific MDP examples.

\subsection{Model-free online RL in Markov Decision Processes}\label{subsec: model free MDP}

In this subsection, we specify Algorithm \ref{alg: implicit optimism mdp} for model-free hypothesis (Example \ref{exp: model free}).
For a model-free hypothesis class, we choose the discrepancy function $\ell$ as, given $\mathcal{D}_h = (x_h,a_h,r_h,x_{h+1})$, 
\begin{align}\label{eq: example model free ell}
    \ell_{f'}(f;\mathcal{D}_h) = \left(\mathbb{E}_{x_{h+1}\sim\mathbb{P}_h(\cdot|x_h,a_h)}[l_{f'}((f_h,f_{h+1});\mathcal{D}_h)]\right)^2.
\end{align}
where the function $l:\mathcal{H}\times\mathcal{H}_h\times\mathcal{H}_{h+1}\times(\mathcal{S}\times\mathcal{A}\times\mathbb{R}\times\mathcal{S})\mapsto\mathbb{R}$ satisfies Assumption \ref{ass: l function}.
We specify the choice of $l$ in concrete examples of MDPs later.

In the following, we check and specify Assumptions \ref{ass: gec} and \ref{ass: supervised learning} for model-free hypothesis classes.

\begin{prop}[Generalization: model-free RL]\label{prop: supervised learning model free rl}
    We assume that $\mathcal{H}$ is finite, i.e., $|\mathcal{H}|<+\infty$. Then under Assumption \ref{ass: l function}, with probability at least $1-\delta$, for any $k\in[K]$ and $f\in\mathcal{H}$, it holds that 
    \begin{align*}
        \sum_{h=1}^HL_h^{k-1}(f^{\ast}) - L_h^{k-1}(f) \lesssim -\sum_{h=1}^H\sum_{s=1}^{k-1}\mathbb{E}_{\xi_h\sim \pi_{\exp}(f^s)}[ \ell_{f^s}(f;\xi_h)] +  HB_l^2\log(HK/\delta) + B_l^2\log(|\mathcal{H}|),
    \end{align*}
    where $L$ and $\ell$ are defined in \eqref{eq: implicit optimism mdp free based L} and \eqref{eq: example model free ell} respectively. Here $B_l$ is specified in Assumption \ref{ass: l function}.
\end{prop}

\begin{proof}[Proof of Proposition \ref{prop: supervised learning model free rl}]
    See Appendix \ref{subsec: proof prop supervised learning model free rl} for detailed proof.
\end{proof}

Proposition \ref{prop: supervised learning model free rl} specifies Assumption \ref{ass: supervised learning}.
For Assumption~\ref{ass: gec}, we need structural assumptions on the MDP.
Given an MDP with GEC $d_{\mathrm{GEC}}$, we have the following corollary of Theorem \ref{thm:reg}.

\begin{corollary}[Online regret of \texttt{MEX}: model-free hypothesis]\label{cor: regret model free mdp}
    Given an MDP with generalized eluder coefficient $d_{\mathrm{GEC}}(\cdot)$ and a finite model-free hypothesis class $\mathcal{H}$ with $f^{\ast}\in\mathcal{H}$, under Assumption~\ref{ass: l function}, setting 
    \begin{align}
        \eta = \sqrt{\frac{d_{\mathrm{GEC}}(1/\sqrt{HK})}{ (H\log(HK/\delta) + \log(|\mathcal{H}|))\cdot B_l^2\cdot K}},
    \end{align}
    then the regret of Algorithm \ref{alg: implicit optimism mdp} after $K$ episodes is upper bounded by
    \begin{align}
        \mathrm{Regret}(T)&\lesssim  B_l\cdot \sqrt{d_{\mathrm{GEC}}(1/\sqrt{HK})\cdot (H\log(HK/\delta)+\log(|\mathcal{H}|))\cdot K},\label{eq: regret model free 1}
    \end{align}
    with probability at least $1-\delta$. Here $B_l$ is specified in Assumption \ref{ass: l function}.
\end{corollary}

Corollary \ref{cor: regret model free mdp} can be directly specified to MDPs with low GEC, including MDPs with low Bellman eluder dimension \citep{jin2021bellman} and MDPs of bilinear class \citep{du2021bilinear}. 
We refer the readers to Appendix~\ref{subsec: example model free mdp} for a detailed discussion of these two examples.

\subsection{Model-based online RL in Markov Decision Processes}\label{subsec: model based MDP}

In this part, we specify Algorithm \ref{alg: implicit optimism mdp} to model-based hypothesis (Example \ref{exp: model based}).
For a model-based hypothesis class, we choose the discrepancy function $\ell$ as the \emph{Hellinger distance}.
Given $\mathcal{D}_h = (x_h,a_h,r_h,x_{h+1})$, we let
\begin{align}\label{eq: example model based ell}
    \ell_{f'}(f;\mathcal{D}_h) = D_{\mathrm{H}}(\mathbb{P}_{h,f}(\cdot|x_h,a_h)\|\mathbb{P}_{h,f^{\ast}}(\cdot|x_h,a_h)),
\end{align}
where $D_{\mathrm{H}}(\cdot\|\cdot)$ denotes the Hellinger distance.
According to \eqref{eq: example model based ell}, the discrepancy function $\ell$ does not depend on the input $f'\in\mathcal{H}$.
In the following, we check and specify Assumptions \ref{ass: gec} and \ref{ass: supervised learning}.

\begin{prop}[Generalization: model-based ]\label{prop: supervised learning model based rl}
    We assume that $\mathcal{H}$ is finite, i.e., $|\mathcal{H}|<+\infty$. Then with probability at least $1-\delta$, for any $k\in[K]$, $f\in\mathcal{H}$, it holds that 
    \begin{align}
        \sum_{h=1}^HL_h^{k-1}(f^{\ast}) - L_h^{k-1}(f) \lesssim -\sum_{h=1}^H\sum_{s=1}^{k-1}\mathbb{E}_{\xi_h\sim \pi_{\mathrm{exp}}(f^s)}[\ell_{f^s}(f;\xi_h)] + H\log(H/\delta) + \log(|\mathcal{H}|),  \notag
    \end{align}
    where $L$ and $\ell$ are defined in \eqref{eq: implicit optimism mdp model based L} and \eqref{eq: example model based ell} respectively.
\end{prop}

\begin{proof}[Proof of Proposition \ref{prop: supervised learning model based rl}]
    See Appendix \ref{subsec: proof prop supervised learning model based rl} for detailed proof.
\end{proof}

Proposition \ref{prop: supervised learning model based rl} specifies Assumption \ref{ass: supervised learning}.
For Assumption \ref{ass: gec}, we also need structural assumptions on the MDP.
Given an MDP with GEC $d_{\mathrm{GEC}}$, we have the following corollary of Theorem~\ref{thm:reg}.

\begin{corollary}[Online regret of \texttt{MEX}: model-based hypothesis]\label{cor: regret model based mdp}
    Given an MDP with generalized eluder coefficient $d_{\mathrm{GEC}}(\cdot)$ and a finite model-based hypothesis class $\mathcal{H}$ with $f^{\ast}\in\mathcal{H}$, by setting 
    \begin{align*}
        \eta = \sqrt{\frac{d_{\mathrm{GEC}}(1/\sqrt{HK})}{(H\log(H/\delta)+\log(|\mathcal{H}|))\cdot K}},
    \end{align*}
    then the regret of Algorithm \ref{alg: implicit optimism mdp} after $K$ episodes is upper bounded by, with probability at least $1-\delta$,
    \begin{align}
        \mathrm{Regret}(K)\lesssim \sqrt{d_{\mathrm{GEC}}(1/\sqrt{HK})\cdot (H\log(H/\delta)+\log(|\mathcal{H}|))\cdot K},\label{eq: regret model based 1}
    \end{align}
\end{corollary}

Corollary \ref{cor: regret model based mdp} can be directly specified to MDPs having low GEC, including MDPs with low witness rank \citep{sun2019model}.
We refer the readers to Appendix \ref{subsec: example model based mdp} for a detailed discussion of this example.

%% file: tex/mg.tex
\section{Extensions to Two-player Zero-sum Markov Games}\label{sec: mg}
In this section, we extend the definition of GEC to the two-player zero-sum MG setting and adapt \texttt{MEX} to this setting in both model-free and model-based styles. Then we provide the theoretical guarantee for our proposed algorithms and specify the results in concrete examples such as linear two-player zero-sum MG.
\subsection{Online Reinforcement Learning in Two-player Zero-sum Markov Games}

Markov games (MGs) generalize the standard Markov decision process to the multi-agent setting. 
We consider the episodic two-player zero-sum MG, which is denoted as $(H, \mathcal{S}, \mathcal{A}, \mathcal{B}, \mathbb{P}, r)$. 
Here $\mathcal{S}$ is the state space shared by both players, $\mathcal{A}$ and $\mathcal{B}$ are the action spaces of the two players (referred to as the max-player and the min-player) respectively, $H\in\mathbb{N}_+$ denotes the length of each episode,
$\mathbb{P} = \{\mathbb{P}_h\}_{h\in[H]}$ with $\mathbb{P}_h:\cS\times\cA\times\cB\mapsto\Delta(\cS)$ the transition kernel of the next state given the current state and two actions from the two players at timestep $h$, and $r = \{r_h\}_{h\in[H]}$ with $r_h:\mathcal{S}\times\cA\times\cB\mapsto[0,1]$ the reward function at timestep $h$.

We consider \emph{online} reinforcement learning in the episodic two-player zero-sum MG, where the two players interact with the MG for $K\in\mathbb{N}_+$ episodes through the following protocal.
Each episode $k$ starts from an initial state $x_1^k$. 
At each timestep $h$, two players observe the current state $x_h^k$, take joint actions $(a_h^k, b_h^k)$ individually, and observe the next state $x_{h+1}^k \sim \mathbb{P}_h(\cdot \mid x_h^k, a_h^k, b_h^k)$. 
The $k$-th episode ends after step $H$ and then a new episode starts. 
Without loss of generality, we assume each episode has a common fixed initial state $x_1^k=\underline{x}_1$, which can be easily generalized to having $x_1$ sampled from a fixed but unknown distribution.

\paragraph{Policies and value functions.} 
We consider Markovian policies for both the max-player and the min-player.
A Markovian policy of the max-player is denoted by $\mu=\{\mu_h: \cS \mapsto \Delta(\mathcal{A})\}_{h \in[H]}$. 
Similarly, a Markovian policy of the min-player is denoted by $\nu=\{\nu_h: \mathcal{X} \mapsto \Delta(\mathcal{B})\}_{h \in[H]}$.
Given a joint policy $\boldsymbol{\pi} = (\mu, \nu)$, its state-value function $V_h^{\mu, \nu}: \mathcal{S} \mapsto\mathbb{R}_+$ and state-action value function $Q_h^{\mu,\nu}:\cS\times\cA\times\cB\mapsto\mathbb{R}_+$ at timestep $h$ are defined as
\begin{align}
    V_h^{\mu, \nu}(x)&:=\mathbb{E}_{\mathbb{P},(\mu, \nu)}\left[\sum_{h^{\prime}=h}^H r_{h^{\prime}}(x_{h^{\prime}}, a_{h^{\prime}}, b_{h^{\prime}}) \middle|\, x_h=x\right],\\
    Q_h^{\mu, \nu}(x, a, b)&:=\mathbb{E}_{\mathbb{P},(\mu, \nu)} \left[\sum_{h=h}^H r_{h^{\prime}}(x_{h^{\prime}}, a_{h^{\prime}}, b_{h^{\prime}}) \middle|\, (x_h, a_h, b_h)=(x, a, b)\right],
\end{align}
where the expectations are taken over the randomness of the transition kernel and the policies.
In the game, the max-player wants to maximize the value functions, while the min-layer aims at minimizing the value functions.

\paragraph{Best response, Nash equilibrium, and Bellman equations.} 
Given a max-player's policy $\mu$, the \emph{best response policy} of the min-player, denoted by $\nu^{\dagger}(\mu)$, is the policy that minimizes the total rewards given that the max-player uses $\mu$.
According to this definition, and for notational simplicity, we denote 
\begin{align}
    V_h^{\mu,\dagger}(x)&:=V_h^{\mu, \nu^{\dagger}(\mu)}(x)=\inf_{\nu} V_h^{\mu, \nu}(x),\notag \\
    Q_h^{\mu,\dagger}(x,a,b)&:=Q_h^{\mu, \nu^{\dagger}(\mu)}(x,a,b)=\inf_{\nu} Q_h^{\mu, \nu}(x,a,b),\label{eq: best response max value function}
\end{align}
for any $(x,a,b,h)\in\cS\times\cA\times\cB\times[H]$.
Similarly, given a min-player's policy $\nu$, there is a \emph{best response policy} $\mu^{\dagger}(\nu)$ for the max-player that maximizes the total rewards given $\nu$.
According to the definition, we denote  
\begin{align}
    V_h^{\dagger,\nu}(x)&:=V_h^{\mu^{\dagger}(\nu), \nu}(x)=\sup_{\mu} V_h^{\mu, \nu}(x),\notag \\ Q_h^{\dagger,\nu}(x,a,b)&:=Q_h^{\mu^{\dagger}(\nu),\nu}(x,a,b)=\sup_{\mu} Q_h^{\mu, \nu}(x,a,b),
\end{align}
for any $(x,a,b,h)\in\cS\times\cA\times\cB\times[H]$.
Furthermore, there exists a \emph{Nash equilibrium} (NE) joint policy $(\mu^*, \nu^*)$ \citep{filar2012competitive} such that both players are optimal against their best responses.
That is,
\begin{align}
    V_h^{\mu^*, \dagger}(x)=\sup _\mu V_h^{\mu, \dagger}(x), \quad V_h^{\dagger, \nu^*}(x)=\inf _\nu V_h^{\dagger, \nu}(x),
\end{align}
for any $(x, h) \in \mathcal{S} \times[H]$. 
For the NE joint policy, we have the following minimax equation,
\begin{align}
    \sup_\mu \inf_\nu V_h^{\mu, \nu}(x)=V_h^{\mu^*, \nu^*}(x)=\inf_\nu \sup_\mu V_h^{\mu, \nu}(x).
\end{align}\label{eq: NE value function}
for any $(x, h) \in \mathcal{S} \times[H]$.
This shows that: i) the for two-player zero-sum MG, the sup and the inf exchanges; ii) the NE policy has a unique state-value (state-action value) function, which we denote as $V^{\ast}$ and $Q^{\ast}$ respectively.
Finally, we introduce two sets of Bellman equations for best response value functions and NE value functions. 
In specific, for the min-player's best response value functions given max-player policy $\mu$, i.e., \eqref{eq: best response max value function}, we have the following Bellman equation,\footnote{For simplicity, we define $\mathbb{D}_{(\mu_h,\nu_h)}:=\mathbb{E}_{a\sim\mu_h(\cdot|x),b\sim\nu_h(\cdot|x)}[Q(x, a, b)]$ for any $\mu_h$, $\nu_h$, and function $Q$.} 
\begin{align}
    Q_h^{\mu,\dagger}(x,a,b) = (\mathcal{T}^{\mu}_hQ_{h+1}^{\mu,\dagger})(x,a,b):= r_h(x,a,b) + \mathbb{E}_{x'\sim \mathbb{P}_h(\cdot|x,a,b)}\bigg[\inf_{\nu_{h+1}}\mathbb{D}_{(\mu_{h+1},\nu_{h+1})}Q_{h+1}^{\mu,\dagger}(x')\bigg],
\end{align} 
for any $(x,a,b,h)\in\cS\times\cA\times\cB\times[H]$.
We name $\mathcal{T}^{\mu}_h$ as the \emph{min-player best response Bellman operator} given max-player policy $\mu$, and we define 
\begin{align}\label{eq: be bellman}
\mathcal{E}^{\mu}_h(Q_h,Q_{h+1};x,a,b):=Q_h(x,a,b) - \mathcal{T}^{\mu}_hQ_{h+1}(x,a,b),
\end{align}
as the \emph{min-player best response Bellman residual} given max-player policy $\mu$ at timestep $h$ of any functions $(Q_h, Q_{h+1})$.
Also, for the NE value functions, i.e., \eqref{eq: NE value function}, we also have the following NE Bellman equation,
\begin{align}
    Q_h^{\ast}(x,a,b) = (\mathcal{T}^{\mathrm{NE}}_hQ_{h+1}^{\ast})(x,a,b):= r_h(x,a,b) + \mathbb{E}_{x'\sim \mathbb{P}_h(\cdot|x,a,b)}\bigg[\sup_{\mu_{h+1}}\inf_{\nu_{h+1}}\mathbb{D}_{(\mu_{h+1},\nu_{h+1})}Q_{h+1}^{\ast}(x')\bigg],
\end{align} 
for any $(x,a,b,h)\in\cS\times\cA\times\cB\times[H]$.
We call $\mathcal{T}^{\mathrm{NE}}_h$ the NE Bellman operator, and we define 
\begin{align}\label{eq: ne bellman}
\mathcal{E}^{\mathrm{NE}}_h(Q_h,Q_{h+1};x,a,b):=Q_h(x,a,b) - \mathcal{T}^{\mathrm{NE}}_hQ_{h+1}(x,a,b),
\end{align}
as the \emph{NE Bellman residual} at timestep $h$ of any functions $(Q_h, Q_{h+1})$.

\paragraph{Performance metric.} 
We say a max-player's policy $\mu$ is $\epsilon$-close to Nash equilibrium if $V^*(x_1)-V^{\mu, \dagger}(x_1)<\epsilon$. 
The goal of this section is to find such a max-player policy. 
The corresponding regret after $K$ episodes is,
\begin{align}
    \mathrm{Regret}_{\mathrm{MG}}(K)=\sum_{k=1}^KV_1^*(x_1)-V_1^{\mu^k, \dagger}(x_1),
\end{align}
where $\mu^k$ is the policy used by the max-player for the $k$-th episode.
Such a problem setting is also considered by \citet{jin2022power, huang2021towards, xiong22b}.
Actually, the roles of two players can be exchanged, so that the goal turns to learning a min-player policy $\nu$ which is $\epsilon$-close to the Nash equilibrium.

\subsection{Function Approximation: Model-Free and Model-Based Hypothesis}

Parallel to the MDP setting, we study two-player zero-sum MGs in the context of general function approximations.
In specific, we assume access to an abstract hypothesis class $\mathcal{H} = \mathcal{H}_1\times\cdots\times\mathcal{H}_H$, which can be specified to model-based and model-free settings, respectively.
Also, we denote $\boldsymbol{\Pi} = \mathbf{M}\times\mathbf{N}$ with $\mathbf{M} = \mathbf{M}_1\times\cdots\times\mathbf{M}_H$ and $\mathbf{N} = \mathbf{N}_1\times\cdots\times\mathbf{N}_H$ as the space of Markovian joint policies.

The following two examples show how to specify $\mathcal{H}$ for model-free and model-based settings.

\begin{example}[Model-free hypothesis class: two-player zero-sum Markov game]\label{exp: model free mg}
    For the model-free setting, $\mathcal{H}$ contains approximators of the state-action value functions of the MG, i.e., $\mathcal{H}_h\subseteq\{f_h:\mathcal{S}\times\mathcal{A}\times\mathcal{B}\mapsto\mathbb{R}\}$.
    Specifically, for any $f = (f_1,\cdots,f_H)\in\mathcal{H}$:
    \begin{enumerate}
        \item we denote the corresponding state-action value function $Q_{f} = \{Q_{h,f}\}_{h\in[H]}$ with $Q_{h,f} = f_h$;
        \item we denote the corresponding NE state-value function $V_f = \{V_{h,f}\}_{h\in[H]}$ with 
        $$V_{h,f}(\cdot) = \sup_{\mu_h\in\mathbf{M}_h}\inf_{\nu_h\in\mathbf{N}_h}\mathbb{D}_{(\mu_h,\nu_h)}Q_{h,f}(\cdot),$$ and we denote the corresponding NE max-player policy by $\mu_f = \{\mu_{h,f}\}_{h\in[H]}$ with 
        $$\mu_{h,f}(\cdot) = \argsup_{\mu_h\in\mathbf{M}_h}\inf_{\nu_h\in\mathbf{N}_h}\mathbb{D}_{(\mu_h,\nu_h)}Q_{h,f}(\cdot).
        $$
        \item given a policy of the max-player $\mu\in\mathbf{M}$, we define $V_{f}^{\mu,\dagger} = \{V_{h, f}^{\mu,\dagger}\}_{h\in[H]}$ as the state-value function induced by $Q_f$, $\mu$ and its best response, i.e., $V_{h, f}^{\mu,\dagger}(\cdot) = \inf_{\nu_h\in\mathbf{N}_h}\mathbb{D}_{(\mu_h,\nu_h)}Q_{h,f}(\cdot)$, and we denote the corresponding best response min-player policy as $\nu_{f,\mu} = \{\nu_{h,f,\mu}\}_{h\in[H]}$, i.e., $\nu_{h,f} = \arginf_{\nu_h\in\mathbf{N}_h}\mathbb{D}_{(\mu_h,\nu_h)}Q_{h,f}(\cdot)$.
        \item we denote the NE state-action value function under the true model, i.e., $Q^{\ast}$, by $f^{\ast}$.
    \end{enumerate}
  \end{example}

\begin{example}[Model-based hypothesis class: two-player zero-sum Markov game]\label{exp: model based mg}
    For the model-based setting, $\mathcal{H}$ contains approximators of the transition kernel of the MG, for which we denote $f=\mathbb{P}_f=(\mathbb{P}_{1,f},\cdots,\mathbb{P}_{H,f}) \in \mathcal{H}$. 
    For any $(f,\boldsymbol{\pi})\in\cH\times\boldsymbol{\Pi}$ with $\boldsymbol{\pi} = (\mu,\nu)$:
    \begin{enumerate}
        \item we denote $V_{f}^{\mu,\nu} = \{V_{h,f}^{\mu,\nu}\}_{h\in[H]}$ as the state-value function induced by model $\mathbb{P}_f$ and joint policy $(\mu,\nu)$.
        \item we denote $V_{f} = \{V_{h,f}\}_{h=\in[H]}$ as the NE state-value function induced by model $\mathbb{P}_f$, and we denote the corresponding NE max-player policy as $\mu_f = \{\mu_{h,f}\}_{h\in[H]}$.
        \item given a policy of the max-player $\mu\in\mathbf{M}$, we define $V_{f}^{\mu,\dagger} = \{V_{h, f}^{\mu,\dagger}\}_{h\in[H]}$ as the state-value function induced by model $\mathbb{P}_f$, $\mu$ and its best response, i.e., $V_{h,f}^{\mu,\dagger}(\cdot) = \inf_{\nu\in\mathbf{N}}V_{h,f}^{\mu,\nu}(\cdot)$, and we denote the corresponding best response min-player policy as $\nu_{f,\mu} = \{\nu_{h,f,\mu}\}_{h\in[H]}$, i.e., $\nu_{f,\mu} = \arginf_{\nu\in\mathbf{N}}V_{h,f}^{\mu,\nu}(\cdot)$.
        \item we denote the true model $\mathbb{P}$ of the two-player zero-sum MG as $f^{\ast}$.
    \end{enumerate}
\end{example}

\subsection{Algorithm Framework: Maximize to Explore (MEX-MG)}\label{subsec: algorithm mg}

In this section, we extend the \emph{Maximize to Explore} framework (\texttt{MEX}, Algorithm \ref{alg: implicit optimism mdp}) proposed in Section \ref{sec: alg} to the two-player zero-sum MG setting, resulting in \texttt{MEX-MG} (Algorithm \ref{alg: implicit optimism mg}).
\texttt{MEX-MG} controls the max-player and the min-player in a centralized manner.
The min-player is aimed at assisting the max-player to achieve low regret.
This kind of \emph{self-play} algorithm framework has received considerable attention recently in theoretical study of two-player zero-sum MGs \citep{jin2022power, huang2021towards, xiong22b}.

We first give a generic algorithm framework and then instantiate it to model-free (Example \ref{exp: model free mg}) and model- based (Example \ref{exp: model based mg}) hypotheses respectively.

\begin{algorithm}[t]
	\caption{Maximize to Explore for two-player zero-sum Markov Game (\texttt{MEX-MG})}
	\label{alg: implicit optimism mg}
	\begin{algorithmic}[1]
	\STATE \textbf{Input}: Hypothesis class $\mathcal{H}$, parameter $\eta>0$.
    \FOR{$k=1,\cdots,K$}
        \STATE Solve $f^k\in\mathcal{H}$ via 
        \begin{align}
            f^k = \operatornamewithlimits{argsup}_{f\in\mathcal{H}}\left\{V_{1, f}(x_1) - \eta \cdot\sum_{h=1}^HL_h^{k-1}(f)\right\}.\label{eq: implicit optimism mg max}
      \end{align}
        \STATE Set the max-player policy as $\mu^k = \mu_{f^k}$.
        \STATE Solve $g^k\in\mathcal{H}$ via 
        \begin{align}
            g^k = \operatornamewithlimits{argsup}_{g\in\mathcal{H}}\left\{-V_{1, g}^{\mu^k,\dagger}(x_1) - \eta \cdot\sum_{h=1}^HL_{h,\mu^k}^{k-1}(g)\right\}.\label{eq: implicit optimism mg min}
        \end{align}
        \STATE Set the min-player policy as $\nu^k = \nu_{g^k,\mu^k}$.
        \STATE Execute $\boldsymbol{\pi}^k = (\mu^k, \nu^k)$ to collect data $\mathcal{D}^k = \{\mathcal{D}_h^k\}_{h\in[H]}$ with $\mathcal{D}_h^k = (x_h^k,a_h^k,b_h^k,r_h^k,x_{h+1}^k)$.
    \ENDFOR
	\end{algorithmic}
\end{algorithm}

\subsubsection{Generic algorithm}

\texttt{MEX-MG} leverages the asymmetric structure between the max-player and min-player to achieve sample-efficient learning. 
In specific, it picks two different hypotheses for the two players respectively, so that the max-player is aimed at approximating the NE max-player policy and the min-player is aimed at approximating the best response of the max-player, assisting its regret minimization.

\paragraph{Max-player.} At each episode $k\in[K]$, \texttt{MEX-MG} first estimates a hypothesis $f^k\in\mathcal{H}$ for the max-player using historical data $\{\cD^s\}_{s=1}^{k-1}$ by maximizing objective \eqref{eq: implicit optimism mg max}.
Parallel to \texttt{MEX}, to achieve the goal of exploiting history knowledge
while encouraging exploration, the composite objective \eqref{eq: implicit optimism mg max} sums: \textbf{(a)} the negative loss $-L_h^{k-1}(f)$ induced by
the hypothesis $f$; \textbf{(b)}  the Nash equilibrium value associated with the current hypothesis, i.e., $V_{1,f}$.
\texttt{MEX-MG} balances exploration and exploitation via a tuning parameter $\eta>0$. 
With the hypothesis $f^k$, \texttt{MEX-MG} sets the max-player's policy $\mu^k$ as the NE max-player policy with respect to $f^k$, i.e., $\mu_{f^k}$.

\paragraph{Min-player.}
After obtaining the max-player policy $\mu^k$, \texttt{MEX-MG} goes to estimate another hypothesis for the min-player in order to approximate the best response of the max-player.
In specific, \texttt{MEX-MG} estimates $g^k\in\cH$ using historical data $\{\cD^s\}_{s=1}^{k-1}$ by maximizing objective \eqref{eq: implicit optimism mg min}, which also sums two objectives: \textbf{(a)} the negative loss $-L_{h,\mu^k}^{k-1}(g)$ induced by
the hypothesis $g$. 
Here the loss function depends on $\mu^k$ since we aim to approximate the best response of $\mu^k$;
\textbf{(b)}  the negative best response min-player value associated with the current hypothesis $g$ and $\mu^k$, i.e., $-V_{1,g}^{\mu^k,\dagger}$.
The negative sign is due to the goal of min-player, i.e., minimization of the total rewards.
With $g^k$, \texttt{MEX-MG} sets the min-player's policy $\nu^k$ as the best response policy of $\mu^k$ under $g^k$, i.e., $\nu_{g^k,\mu^k}$.

\paragraph{Data collection.}
Finally, the two agents execute the joint policy $\boldsymbol{\pi}^k = (\mu^k,\nu^k)$ to collect new data $\mathcal{D}^k = \{(x_h^k,a_h^k,b_h^k,r_h^k,x_{h+1}^k)\}_{h=1}^H$ and update their loss functions $L(\cdot)$.
The choice of the loss functions varies between model-free and model-based hypotheses, which we specify in the following.

\subsubsection{Model-free algorithm}
For model-free hypothesis (Example \ref{exp: model free mg}), the composite objectives \eqref{eq: implicit optimism mg max} and \eqref{eq: implicit optimism mg min} becomes
\begin{align}
    f^k &= \operatornamewithlimits{argsup}_{f\in\mathcal{H}}\left\{\sup_{\mu_1\in\mathbf{M}_1}\inf_{\nu_1\in\mathbf{N}_1}\mathbb{D}_{(\mu_1,\nu_1)}Q_{1, f}(x_1) - \eta \cdot\sum_{h=1}^HL_h^{k-1}(f)\right\},\\
    g^k &= \operatornamewithlimits{argsup}_{g\in\mathcal{H}}\left\{-\inf_{\nu_1\in\mathbf{N}_1}\mathbb{D}_{(\mu_1^k,\nu_1)}Q_{1, g}(x_1) - \eta\cdot\sum_{h=1}^HL_{h,\mu^k}^{k-1}(g)\right\}.\label{eq: implicit optimism mg model free L 1}
\end{align}
In the model-free algorithm, we choose the loss functions as empirical estimates of squared Bellman residuals.
For the max-player who wants to approximate the NE max-player policy, 
we choose the loss function $L_h^k(f)$ as an estimation of the squared NE Bellman residual, given by
\begin{align}
    L_h^k(f) &= \sum_{s=1}^k \Big(Q_{h,f}(x_h^s,a_h^s,b_h^s) - r_h^s - V_{h+1, f}(x_{h+1}^s)\Big)^2 \notag\\
    &\qquad - \inf_{f_h'\in\mathcal{H}_h}\sum_{s=1}^k\Big(Q_{h,f'}(x_h^s,a_h^s,b_h^s) - r_h^s - V_{h+1, f}(x_{h+1}^s)\Big)^2.\label{eq: implicit optimism mg model free L 2}
\end{align}
For the min-player who aims at approximating the best response policy of $\mu^k$, we set the loss function $L_{h,\mu}^k(g)$ as an estimation of the squared best-response Bellman residual given max-player policy $\mu$,
\begin{align}
    L_{h,\mu}^k(g) &= \sum_{s=1}^k \left(Q_{h,g}(x_h^s,a_h^s,b_h^s) - r_h^s - V_{h+1, g}^{\mu,\dagger}(x_{h+1}^s)\right)^2 \notag \\
    &\qquad - \inf_{g_h'\in\mathcal{H}_h}\sum_{s=1}^k\left(Q_{h,g'}(x_h^s,a_h^s,b_h^s) - r_h^s - V_{h+1, g}^{\mu,\dagger}(x_{h+1}^s)\right)^2.\label{eq: implicit optimism mg model free L 3}
\end{align}
We remark that the subtracted infimum term in both \eqref{eq: implicit optimism mg model free L 2} and \eqref{eq: implicit optimism mg model free L 3} is for handling the variance terms in the estimation to achieve a fast theoretical rate, as we do for \texttt{MEX} with model-free hypothesis in Section~\ref{sec: alg}.

\subsubsection{Model-based algorithm.} 
For model-based hypothesis (Example \ref{exp: model based mg}),  the composite objectives \eqref{eq: implicit optimism mg max} and \eqref{eq: implicit optimism mg min} becomes
\begin{align}
    f^k &= \operatornamewithlimits{argsup}_{f\in\mathcal{H}}\left\{\sup_{\mu\in\mathbf{M}}\inf_{\nu\in\mathbf{N}}V_{1, \mathbb{P}_f}^{\mu,\nu}(x_1) - \eta\cdot\sum_{h=1}^HL_h^{k-1}(f)\right\},\label{eq: implicit optimism mg model based 1} \\
    g^k &= \operatornamewithlimits{argsup}_{g\in\mathcal{H}}\left\{-\inf_{\nu\in\mathbf{N}}V_{1, \mathbb{P}_g}^{\mu^k,\nu}(x_1) - \eta\cdot\sum_{h=1}^HL_{h,\mu^k}^{k-1}(g)\right\},\label{eq: implicit optimism mg model based 2}
\end{align}
which can be understood as a joint optimization over model $\mathbb{P}_f$ and the joint policy policy $\boldsymbol{\pi} = (\mu,\nu)$.
In the model-based algorithm, we choose the loss function $L_h^k(f)$ as the negative log-likelihood loss, 
\begin{align}\label{eq: implicit optimism mg model based L}
    L_h^{k}(f) = - \sum_{s=1}^{k}\log\mathbb{P}_{h,f}(x_{h+1}^s|x_h^s,a_h^s, b_h^s).
\end{align}
Meanwhile, we choose the loss function $L_{h, \mu}^k(g) = L^k_h(g)$, i.e., \eqref{eq: implicit optimism mg model based L}, regardless of the max-player policy $\mu$. 
But we remark that despite $L_h^k = L_{h,\mu}^k$, $f^k$ and $g^k$ are still different since the exploitation component in \eqref{eq: implicit optimism mg model based 1} and \eqref{eq: implicit optimism mg model based 2} are not the same due to the different targets of the max-player and the min-player.

\subsection{Regret Analysis for MEX-MG Framework}\label{subsec: theory mg}

In this section, we establish the regret of the \texttt{MEX-MG} framework (Algorithm \ref{alg: implicit optimism mg}). 
Specifically, we give an upper bound of its regret which holds for both model-free (Example \ref{exp: model free mg}) and model-based (Example \ref{exp: model based mg}) settings. 
We first present several key assumptions needed for the main result.

We first assume that the hypothesis class $\cH$ is well-specified, containing certain true hypotheses. 

\begin{assumption}[Realizablity]\label{ass: realizability mg}
    We make the following realizability assumptions for the model-free and model-based hypotheses respectively:
    \begin{itemize}
        \item For model-free hypothesis (Example \ref{exp: model free mg}), we assume that the true Nash equilibrium value $f^{\ast}\in\mathcal{H}$. Moreover, for any $f\in\mathcal{F}$, it holds that $Q^{\mu_f,\dagger}\in\mathcal{H}$.
        \item For model-based hypothesis (Example \ref{exp: model based mg}), we assume that the true transition $f^{\ast}\in\mathcal{H}$.
    \end{itemize}
\end{assumption}

Also, we make the following completeness and boundedness assumption on $\cH$.

\begin{assumption}[Completeness and Boundedness]\label{ass: completeness mg}
    For model-free hypothesis (Example \ref{exp: model free mg}), we assume that for any $f,g\in\mathcal{H}$, it holds that $\mathcal{T}_h^{\mu_f}g_h\in\mathcal{H}_h$, for any timestep $h\in[H]$. 
    Also, we assume that there exists $B_f\geq 1$ such that for any $f_h\in\mathcal{H}_h$, it holds that $f_h(x,a,b)\in[0, B_f]$ for any $(x,a,b,h)\in\cS\times\cA\times\cB\times[H]$.
\end{assumption}

Assumptions \ref{ass: realizability mg} and \ref{ass: completeness mg} are standard assumptions in studying two-player zero-sum MGs \citep{jin2022power, huang2021towards, xiong22b}.
Moreover, we make a structural assumption on the underlying MG to ensure sample-efficient online RL.
Inspired by the single-agent analysis, we require that the MG has a low \textbf{\underline{T}wo-player \underline{G}eneralized \underline{E}luder \underline{C}oefficient} (TGEC), which generalizes the GEC defined in Section \ref{sec: theory}.
We provide specific examples of MGs with low TGEC, both model-free and model-based, in Section \ref{subsec: example mg}.

To define TGEC, we introduce two discrepancy functions $\ell$ and $\ell_\mu$, 
\begin{align*}
    \ell_{f'}(f;\xi_h)&:\mathcal{H}\times\mathcal{H}\times(\mathcal{S}\times\mathcal{A}\times\mathbb{R}\times\mathcal{S})\mapsto\mathbb{R},\\
    \ell_{f',\mu}(f;\xi_h)&:\mathcal{H}\times\mathbf{N}\times\mathcal{H}\times(\mathcal{S}\times\mathcal{A}\times\mathbb{R}\times\mathcal{S})\mapsto\mathbb{R},
\end{align*}
which characterizes the error incurred by a hypothesis $f\in\mathcal{H}$ on data $\xi_h = (x_h,a_h,b_h,r_h,x_{h+1})$.
Intuitively, $\ell$ aims at characterizing the NE Bellman residual \eqref{eq: ne bellman}, while $\ell_{\mu}$ aims at characterizing the min-player best response Bellman residual given max-player policy $\mu$ \eqref{eq: be bellman}.
Specific choices of $\ell$ are given in Section \ref{subsec: example mg} for concrete  model-free and model-based examples.

\begin{assumption}[Low Two-Player Generalized Eluder Coefficient (TGEC)]\label{ass: gec mg}
 We assume that given an $\epsilon>0$, there exists a finite $d(\epsilon)\in\mathbb{R}_+$, such that for any sequence of hypotheses $\{(f^k,g^k)\}_{k\in[K]}\subset\mathcal{H}$ and policies $\{\boldsymbol{\pi}^k = (\mu_{f^k}, \nu_{g^k,\mu_{f^k}})\}_{k\in[K]}\subset\boldsymbol{\Pi}$, it holds that
 \begin{align}
    \sum_{k=1}^KV_{1, f^k}(x_1) - V_{1}^{\boldsymbol{\pi}^k}(x_1) \le \inf_{\zeta > 0} \left\{\frac{\zeta}{2} \sum_{h=1}^H\sum_{k=1}^K\sum_{s = 1}^{k-1} \mathbb{E}_{\xi_h\sim \boldsymbol{\pi}^k}[
        \ell_{f^s}(f^k;\xi_h)] + \frac{d(\epsilon)}{2\zeta} + \sqrt{d(\epsilon)HK} + \epsilon HK \right\}, \notag
 \end{align}
 and it also holds that
  \begin{align}
    \sum_{k=1}^K V_{1}^{\boldsymbol{\pi}^k} (x_1)- V_{1, g^k}^{\mu^k,\dagger} (x_1) \le \inf_{\zeta > 0} \left\{\frac{\zeta}{2} \sum_{h=1}^H\sum_{k=1}^K\sum_{s = 1}^{k-1} \mathbb{E}_{\xi_h\sim \boldsymbol{\pi}^k}[
        \ell_{g^s,\mu^k}(g^k;\xi_h)] + \frac{d(\epsilon)}{2\zeta} + \sqrt{d(\epsilon)HK} + \epsilon HK \right\}, \notag
 \end{align}
 where $\mu_k = \mu_{f^k}$.
We denote the smallest $d(\epsilon)\in\mathbb{R}_+$ satisfying this condition as $d_{\mathrm{TGEC}}(\epsilon)$.
\end{assumption}

Finally, we make a concentration-style assumption on loss functions, parallel to Assumption \ref{ass: supervised learning} for MDPs.
For ease of presentation, we also assume that the hypothesis class $\cH$ is finite.

\begin{assumption}[Generalization]\label{ass: supervised learning mg}
    We assume that $\cH$ is finite, i.e., $|\cH|<+\infty$, and that with probability at least $1-\delta$, for any episode $k\in[K]$ and hypotheses $f, g\in\mathcal{H}$, it holds that
    \begin{align}
        \sum_{h=1}^HL_h^{k-1}(f^{\ast}) - L_h^{k-1}(f) \lesssim -\sum_{h=1}^H\sum_{s=1}^{k-1}\mathbb{E}_{\xi_h\sim \boldsymbol{\pi}^k}[\ell_{f^s}(f;\xi_h)] + B\cdot\big(H\log(HK/\delta) + \log(|\mathcal{H}|)\big). \notag
    \end{align}
    and it also holds that, with $\star = Q^{\mu^k,\dagger}$ for model-free hypothesis and $\star = f^{\ast}$ for model-based hypothesis,
    \begin{align}
        \sum_{h=1}^HL_{h,\mu^k}^{k-1}(\star) - L_{h,\mu^k}^{k-1}(g) \lesssim -\sum_{h=1}^H\sum_{s=1}^{k-1}\mathbb{E}_{\xi_h\sim \boldsymbol{\pi}^k}[\ell_{g^s,\mu^k}(g;\xi_h)] + B\cdot\big(H\log(HK/\delta) + \log(|\mathcal{H}|)\big), \notag
    \end{align}
    Here $B = B_f^2$ for model-free hypothesis (see Assumption \ref{ass: completeness mg}) and $B = 1$ for model-based hypothesis.
\end{assumption}

As we show in Proposition \ref{prop: supervised learning model free rl mg} and Proposition \ref{prop: supervised learning model based rl mg}, Assumption \ref{ass: supervised learning mg}  holds for both model-free and model-based settings. 
With Assumptions \ref{ass: realizability mg}, \ref{ass: completeness mg} (model-free only), \ref{ass: gec mg}, and \ref{ass: supervised learning mg}, we can present our main theoretical result.

\begin{theorem}[Online regret of \texttt{MEX-MG} (Algorithm \ref{alg: implicit optimism mg})]\label{thm: reg mg}
Under Assumptions \ref{ass: realizability mg}, \ref{ass: completeness mg} (model-free only), \ref{ass: gec mg}, and \ref{ass: supervised learning mg}, by setting 
\begin{align*}
     \eta = \sqrt{\frac{d_{\mathrm{TGEC}}(1/\sqrt{HK})}{(H\log(HK/\delta) + \log(|\mathcal{H}|))\cdot B\cdot K}},
\end{align*} 
the regret of Algorithm \ref{alg: implicit optimism mg} after $K$ episodes is upper bounded by
    \begin{align*}
        \mathrm{Regret}(K) \lesssim  \sqrt{d_{\mathrm{TGEC}}(1/\sqrt{HK})\cdot(H\log(HK/\delta) + \log(|\mathcal{H}|))\cdot  B\cdot K},
    \end{align*}
    with probability at least $1-\delta$. 
    Here $d_{\mathrm{TGEC}}(\cdot)$ is given by Assumption \ref{ass: gec mg}.
\end{theorem}

\begin{proof}[Proof of Theorem \ref{thm: reg mg}]
    See Appendix \ref{subsec: proof regret mg} for detailed proof.
\end{proof}

\subsection{Examples of MEX-MG Framework}\label{subsec: example mg}

\subsubsection{Model-free online RL in Two-player Zero-sum Markov Games}

In this subsection, we specify \texttt{MEX-MG} (Algorithm \ref{alg: implicit optimism mg}) for model-free hypothesis class (Example~\ref{exp: model free mg}).
In specific, we choose the discrepancy functions $\ell$ and $\ell_\mu$ as, given $\xi_h = (x_h,a_h,b_h,r_h,x_{h+1})$, 
\begin{align}
    \ell_{f'}(f;\xi_h) &= \Big(Q_{h,f}(x_h,a_h,b_h) - r_h - \mathbb{E}_{x_{h+1}\sim\mathbb{P}_h(\cdot|x_h,a_h,b_h)}[V_{h+1,f}(x_{h+1})]\Big)^2, \label{eq: example model free ell mg 1} \\
    \ell_{f',\mu}(g;\xi_h) &= \Big(Q_{h,g}(x_h,a_h,b_h) - r_h - \mathbb{E}_{x_{h+1}\sim\mathbb{P}_h(\cdot|x_h,a_h,b_h)}[V_{h+1,g}^{\mu,\dagger}(x_{h+1})]\Big)^2.\label{eq: example model free ell mg 2} 
\end{align}
By \eqref{eq: example model free ell mg 1} and \eqref{eq: example model free ell mg 2}, both $\ell_{f'}$ and $\ell_{f',\mu}$ do not depend on the input $f'$.
In the following, we check and specify Assumptions \ref{ass: gec mg} and \ref{ass: supervised learning mg} in Section \ref{subsec: theory mg} for model-free hypothesis class.

\begin{prop}[Generalization: model-free RL]\label{prop: supervised learning model free rl mg}
    We assume that $\mathcal{H}$ is finite, i.e., $|\mathcal{H}|<+\infty$. Then with probability at least $1-\delta$, for any $k\in[K]$ and $f, g\in\mathcal{H}$, it holds simultaneously that 
    \begin{align}
        \sum_{h=1}^HL_h^{k-1}(f^{\ast}) - L_h^{k-1}(f) &\lesssim -\sum_{h=1}^H\sum_{s=1}^{k-1}\mathbb{E}_{\xi_h\sim \boldsymbol{\pi}^k}[\ell_{f^s}(f;\xi_h)] + HB_f^2\log(HK/\delta) + B_f^2\log(|\mathcal{H}|), \notag\\
        \sum_{h=1}^HL_{h,\mu^k}^{k-1}(Q^{\mu^k,\dagger}) - L_{h,\mu^k}^{k-1}(g) &\lesssim -\sum_{h=1}^H\sum_{s=1}^{k-1}\mathbb{E}_{\xi_h\sim \boldsymbol{\pi}^k}[\ell_{g^s,\mu^k}(g;\xi_h)] + HB_f^2\log(HK/\delta) + B_f^2\log(|\mathcal{H}|), \notag
    \end{align}
    where $L$, $L_{\mu}$, $\ell$, and $\ell_{\mu}$ are defined in \eqref{eq: implicit optimism mg model free L 1}, \eqref{eq: implicit optimism mg model free L 2}, \eqref{eq: example model free ell mg 1}, and \eqref{eq: example model free ell mg 2}, respectively.
\end{prop}

\begin{proof}[Proof of Proposition \ref{prop: supervised learning model free rl mg}]
    See Appendix \ref{subsec: proof prop supervised learning model free rl mg} for a detailed proof.
\end{proof}

Proposition \ref{prop: supervised learning model free rl mg} specifies Assumption \ref{ass: supervised learning mg} for abstract model-free hypothesis.
Now given a two-player zero-sum MG with TGEC $d_{\mathrm{TGEC}}$, we have the following corollary of Theorem \ref{thm: reg mg}.

\begin{corollary}[Online regret of \texttt{MEX-MG}: model-free hypothesis]\label{cor: regret model free mg}
    Given a two-player zero-sum MG with two-player generalized eluder coefficient $d_{\mathrm{TGEC}}(\cdot)$ and a finite model-free hypothesis class $\mathcal{H}$ satisfying Assumptions \ref{ass: realizability mg} and \ref{ass: completeness mg}, by setting 
    \begin{align}
        \eta = \sqrt{\frac{d_{\mathrm{TGEC}}(1/\sqrt{HK})}{( H\log(HK/\delta)+ \log(|\mathcal{H}|))\cdot B_f^2\cdot K}},
    \end{align}
    then the regret of Algorithm \ref{alg: implicit optimism mg} after $K$ episodes is upper bounded by
    \begin{align}
        \mathrm{Regret}(T)&\lesssim  B_f\cdot \sqrt{d_{\mathrm{TGEC}}(1/\sqrt{HK})\cdot (H\log(HK/\delta)+ \log(|\mathcal{H}|))\cdot K},\label{eq: regret model free mg 1}
    \end{align}
    with probability at least $1-\delta$. Here $B$ is specified in Assumption \ref{ass: completeness mg}.
\end{corollary}

\paragraph*{Linear two-player zero-sum Markov game.}
Next, we introduce the linear two-player zero-sum MG \citep{xie2020learning} as a concrete model-free example, for which we can explicitly specify its TGEC. 
Linear MG is a natural extension of linear MDPs \citep{jin2020provably} to the two-player zero-sum MG setting, whose reward and transition kernels are modeled by linear functions. 

\begin{definition}[Linear two-player zero-sum Markov game]\label{def: linear MG}
    A d-dimensional two-player zero-sum linear Markov game satisfies that $r_h(x,a,b) = \phi_h(x,a,b)^\top\alpha_h$ and $\mathbb{P}_h(x^{\prime} \mid x, a, b)=\phi_h(x, a, b)^{\top} \psi_h^\star(x^{\prime})$ for some known feature mapping $\phi_h(x, a, b) \in \mathbb{R}^d$  and some unknown vector $\alpha_h\in\mathbb{R}^d$ and some unknown function $\psi_h(x') \in\mathbb{R}^d$ satisfying $\|\phi_h(x, a, b)\|_2 \leq 1$ and $\max\{\|\alpha_h\|_2,\|\psi^{\star}_h(x')\|_2\}\leq\sqrt{d}$ for any $(x, a, b, x', h) \in \mathcal{S} \times \mathcal{A} \times \mathcal{B}\times\cS\times[H]$.  
\end{definition}

Linear two-player zero-sum MG covers the tabular two-player zero-sum MG as a special case.
For a linear two-player zero-sum MG, we choose the model-free hypothesis class as, for each $h\in[H]$, 
 \begin{align}\label{eq: hypothesis linear mg}
     \mathcal{H}_h=\Big\{\phi_h(\cdot, \cdot, \cdot)^{\top} \theta_h:\left\|\theta_h\right\|_2 \leq(H+1-h) \sqrt{d}\Big\}.
 \end{align}
The following proposition gives the TGEC of a linear two-player zero-sum MG with hypothesis class \eqref{eq: hypothesis linear mg}.

\begin{prop}[TGEC of linear two-player zero-sum MG]\label{prop: tgec linear mg}
 For a linear two-player zero-sum MG, with model-free hypothesis \eqref{eq: hypothesis linear mg}, it holds that
\begin{align} \label{prop: linear MG1}
    d_{\mathrm{TGEC}}(1/\sqrt{HK})\lesssim d\log(HK),\quad \log \big(\mathcal{N}(\mathcal{H},1/K, \|\cdot\|_{\infty})\big) \lesssim  dH\log(dK),
\end{align}
where $\mathcal{N}(\mathcal{H},1/K, \|\cdot\|_{\infty})$ denotes the $1/K$-covering number of $\cH$ under $\|\cdot\|_{\infty}$-norm.
\end{prop}

\begin{proof}[Proof of Proposition \ref{prop: tgec linear mg}]
  See Appendix \ref{app:linear MG1} for a detailed proof. 
\end{proof}

As proved by \citet{huang2021towards}, a linear two-player zero-sum MG with model-free hypothesis class \eqref{eq: hypothesis linear mg} also satisfies the realizability and completeness assumptions (Assumptions \ref{ass: realizability mg} and \ref{ass: completeness mg}, with $B_f=H$). 
Thus we can specify Theorem \ref{thm: reg mg} for linear two-player zero-sum MGs as follows.

\begin{corollary}[Online regret of \texttt{MEX-MG}: linear two-player zero-sum MG] \label{cor: linear MG}
By setting $\eta = \tilde{\Theta}(\sqrt{1/H^3K})$, the regret of Algorithm \ref{alg: implicit optimism mg} for linear two-player zero-sum MG after $K$ episodes is upper bounded by
\begin{align}
     \mathrm{Regret}_{\mathrm{MG}}(K)\lesssim dH^{3/2}K^{1/2}\log(HKd/\delta),\notag
\end{align}
with probability at least $1-\delta$, where $d$ is the dimension of the  linear two-player zero-sum MG. \end{corollary}

\begin{proof}[Proof of Corollary \ref{cor: linear MG}]
    Using Corollary \ref{cor: regret model free mg}, Proposition \ref{prop: tgec linear mg}, and a covering number argument.
\end{proof}

\subsubsection{Model-based online RL in Two-player Zero-sum Markov Games}

In this subsection, we specify Algorithm \ref{alg: implicit optimism mg} for model-based hypothesis class $\mathcal{H}$ (Example \ref{exp: model based mg}).
In specific, we choose the discrepancy function $\ell$ as the Hellinger distance. 
Given data $\xi_h = (x_h,a_h, b_h ,x_{h+1})$, we let
\begin{align}\label{eq: example model based ell mg}
    \ell_{f'}(f;\xi_h) = \ell_{f',\mu}(f;\xi_h) =  D_{\mathrm{H}}(\mathbb{P}_{h,f}(\cdot|x_h,a_h, b_h)\|\mathbb{P}_{h,f^{\ast}}(\cdot|x_h,a_h, b_h)),
\end{align}
where $D_{\mathrm{H}}(\cdot\|\cdot)$ denotes the Hellinger distance.
We note that due to \eqref{eq: example model based ell mg}, the discrepancy function $\ell$ does not depend on the input $f'\in\mathcal{H}$ and the max-player policy $\mu$.
In the following, we check and specify Assumptions \ref{ass: gec mg} and \ref{ass: supervised learning mg} in Section \ref{subsec: theory mg} for model-based hypothesis classes.

\begin{prop}[Generalization: model-based RL]\label{prop: supervised learning model based rl mg}
    We assume that $\mathcal{H}$ is finite, i.e., $|\mathcal{H}|<+\infty$. Then with probability at least $1-\delta$, for any $k\in[K]$, $f\in\mathcal{H}$, it holds that 
    \begin{align}
        \sum_{h=1}^HL_h^{k-1}(f^{\ast}) - L_h^{k-1}(f) \lesssim -\sum_{h=1}^H\sum_{s=1}^{k-1}\mathbb{E}_{\xi_h\sim \boldsymbol{\pi}^k}[\ell_{f^s}(f;\xi_h)] + H\log(H/\delta) + \log(|\mathcal{H}|),  \notag
    \end{align}
    where $L$ and $\ell$ are defined in \eqref{eq: implicit optimism mg model based L} and \eqref{eq: example model based ell mg} respectively.
\end{prop}

\begin{proof}[Proof of Proposition \ref{prop: supervised learning model based rl mg}]
    This proposition follows from the same proof of Proposition \ref{prop: supervised learning model based rl}. 
\end{proof}

Since $L_h^k = L_{h,\mu}^k$ and $\ell_f = \ell_{f,\mu}$, Proposition \ref{prop: supervised learning model based rl mg} means that Assumption \ref{ass: supervised learning mg} holds.
Now given a two-player zero-sum MG with TGEC $d_{\mathrm{TGEC}}$, we have the following corollary of Theorem~\ref{thm: reg mg}.

\begin{corollary}[Online regret of \texttt{MEX-MG}: model-based hypothesis]\label{cor: regret model based mg}
    Given a two-player zero-sum MG with two-player generalized eluder coefficient $d_{\mathrm{TGEC}}(\cdot)$ and a finite model-based hypothesis class $\mathcal{H}$ with $f^{\ast}\in\cH$, by setting 
    \begin{align}
        \eta = \sqrt{\frac{d_{\mathrm{TGEC}}(1/\sqrt{HK})}{ (H\log(HK/\delta)+\log(|\mathcal{H}|))\cdot K}},
    \end{align}
    then the regret of Algorithm \ref{alg: implicit optimism mg} after $K$ episodes is upper bounded by
    \begin{align}
        \mathrm{Regret}(T)&\lesssim  \sqrt{d_{\mathrm{TGEC}}(1/\sqrt{HK})\cdot (H\log(HK/\delta)+\log(|\mathcal{H}|))\cdot K},\label{eq: regret model based mg 1}
    \end{align}
    with probability at least $1-\delta$. 
\end{corollary}

\paragraph*{Linear mixture two-player zero-sum Markov game.}
Next, we introduce the linear mixture two-player zero-sum MG as a concrete model-based example, for which we can explicitly specify its TGEC. 
Linear mixture MG is a natural extension of linear mixture MDPs \citep{ayoub2020model,modi2020sample,cai2020provably} to the two-player zero-sum MG setting, whose transition kernels are modeled by linear kernels.
But just as the single-agent setting, the linear mixture MG and the linear MG (Definition \ref{def: linear MG}) do not cover each other as special cases \citep{cai2020provably}.

\begin{definition}[Linear mixture two-player zero-sum Markov game]\label{def: linear mixture MG}
    A d-dimensional two-player zero-sum linear mixture Markov game satisfies that $\mathbb{P}_h(x^{\prime} \mid x, a, b)=\phi_h(x, a, b, x')^{\top} \theta_h^{\star}$ for some known feature mapping $\phi_h(x, a, b,x') \in \mathbb{R}^d$  and  some unknown vector $\theta_h^{\star}\in\mathbb{R}^d$ satisfying $\|\phi_h(x, a, b,x')\|_2 \leq 1$ and $\|\theta_h\|_2\leq\sqrt{d}$ for any $(x, a, b, x', h) \in \mathcal{S} \times \mathcal{A} \times \mathcal{B}\times\cS\times[H]$.  
\end{definition}

Linear mixture two-player zero-sum MG also covers the tabular two-player zero-sum MG as a special case.
For a linear mixture two-player zero-sum MG, we choose the model-based hypothesis class as, for each $h$,
 \begin{align}\label{eq: hypothesis linear mixture mg}
     \mathcal{H}_h=\Big\{\phi_h(\cdot, \cdot, \cdot,\cdot)^{\top} \theta_h:\left\|\theta_h\right\|_2 \leq \sqrt{d}\Big\}.
 \end{align}

The following proposition gives the TGEC of a linear mixture two-player zero-sum MG.

\begin{prop}[TGEC of linear mixture two-player zero-sum MG]\label{prop: tgec linear mixture mg}
 For a linear mixture two-player zero-sum MG, with model-free hypothesis \eqref{eq: hypothesis linear mg}, it holds that
\begin{align} \label{prop: linear MG1}
    d_{\mathrm{TGEC}}(1/\sqrt{HK})\lesssim dH^2\log(HK),\quad \log \big(\mathcal{N}(\mathcal{H},1/K, \|\cdot\|_{\infty})\big) \lesssim dH\log(dK).
\end{align}
where $\mathcal{N}(\mathcal{H},1/K, \|\cdot\|_{\infty})$ denotes the $1/K$-covering number of $\cH$ under $\|\cdot\|_{\infty}$-norm.
\end{prop}

\begin{proof}[Proof of Proposition \ref{prop: tgec linear mixture mg}]
  See Appendix \ref{app: linear mixture two-player zero-sum MG} for a detailed proof. 
\end{proof}

Then we can specify Theorem \ref{thm: reg mg} for linear mixture two-player zero-sum MGs as follows.

\begin{corollary}[Online regret of \texttt{MEX-MG}: linear mixture two-player zero-sum MG] \label{cor: linear mixture MG}
By setting $\eta= \tilde{\Theta}(\sqrt{H/K})$, the regret of Algorithm \ref{alg: implicit optimism mg} for linear mixture two-player zero-sum MG after $K$ episodes is upper bounded by
\begin{align}
     \mathrm{Regret}_{\mathrm{MG}}(K)\lesssim dH^{3/2}K^{1/2}\log(HKd/\delta),\notag
\end{align}
with probability at least $1-\delta$, where $d$ is the dimension of the linear mixture two-player zero-sum MG. 
\end{corollary}

\begin{proof}[Proof of Corollary \ref{cor: linear mixture MG}]
    Using Corollary \ref{cor: regret model based mg}, Proposition \ref{prop: tgec linear mixture mg}, and a covering number argument.
\end{proof}

%% file: tex/exp.tex
\section{Experiments}\label{sec: experiments}
In this section, we propose practical versions of \texttt{MEX} in both model-free and model-based fashion. 

We aim to answer the following two questions: 
\begin{enumerate}
    \item What are the practical approaches to implementing \texttt{MEX} in both model-based (\texttt{MEX-MB}) and model-free (\texttt{MEX-MF}) settings via deep RL methods?
    \item Can \texttt{MEX} handle  challenging exploration tasks, especially those that involve sparse reward scenarios?
\end{enumerate}

\subsection{Experiment Setups}
We evaluate the effectiveness of \texttt{MEX} by assessing its performance in both standard gym locomotion tasks and sparse reward locomotion and navigation tasks within the MuJoCo \citep{todorov2012mujoco} environment. 
For sparse reward tasks, we select \texttt{cheetah-vel}, \texttt{walker-vel}, \texttt{hopper-vel},  \texttt{ant-vel}, and \texttt{ant-goal} adapted from \citet{yu2020meta}, where the agent receives a reward \emph{only} when it successfully attains the desired velocity or goal.
To adapt to deep RL settings, we consider infinite-horizon $\gamma$-discounted MDPs and corresponding \texttt{MEX} variants. 
We report the results averaged over five random seeds. 
In the sparse-reward tasks, the agent only receives a reward when it achieves the desired velocity or position. Regarding the model-based sparse-reward experiments, we assign a target value of $1$ to the \texttt{vel} parameter for the \texttt{walker-vel} task and $1.5$ for the \texttt{hopper-vel}, \texttt{cheetah-vel}, \texttt{ant-vel} tasks. For the model-free sparse-reward experiments, we set the target \texttt{vel} to $3$ for the  \texttt{hopper-vel}, \texttt{walker-vel}, \texttt{cheetah-vel} tasks, and the target \texttt{goal} to $(2,0)$ for  \texttt{ant-goal} task. 
\subsection{Implementation Details}

\paragraph*{Model-free algorithm.}
For the model-free variant \texttt{MEX-MF}, we observe from \eqref{model_free} that adding a maximization bias term 
to the standard TD error is sufficient for provably efficient exploration. 
However, this may lead to instabilities as the bias term only involves the state-action value function of the current policy, and thus the policy may be ever-changing. 
To address this issue, we adopt a similar treatment as in \texttt{CQL} \citep{kumar2020conservative} by subtracting a baseline state-action value from random policy $\mu=\textrm{Unif}(\cA)$ and obtain the following objective,
\begin{equation}\label{eq_mf_org_obj}
%\cL(\theta) = \EE_{\mathcal{D}}\left[\big(r+\gamma Q_\theta(x',a') - Q_\theta(x,a)\big)^2\right] - \eta' \cdot\EE_{\mathcal{D}} \big[\EE_{a\sim \pi} Q_\theta(x,a) - \EE_{a\sim \mu} Q_\theta(x,a) \big],
\min_\theta\max_\pi\,\EE_{\mathcal{D}}\left[\big(r+\gamma Q_\theta(x',a') - Q_\theta(x,a)\big)^2\right] - \eta' \cdot\EE_{\mathcal{D}} \big[\EE_{a\sim \pi} Q_\theta(x,a) - \EE_{a\sim \mu} Q_\theta(x,a) \big].
\end{equation}
We update $\theta$ and $\pi$ in objective \eqref{eq_mf_org_obj} iteratively in an actor-critic fashion. To stabilize training, 
 we adopt a similar entropy regularization $\cH(\mu)$ over $\mu$ as in CQL~\cite{kumar2020conservative}. By incorporating such a regularization, we obtain the following soft constrained variant of \texttt{MEX-MF}, i.e.
\begin{equation*}
%\cL(\theta) = \EE_{\beta}\left[\big(r+\gamma Q_\theta(x',a') - Q_\theta(x,a)\big)^2\right] -\eta' \cdot\EE_{\beta} \bigg[\EE_{a\sim \pi} Q_\theta(x,a) - \log {\sum_{a\in\cA} Q_\theta(x,a)} \bigg].
\min_\theta\max_\pi\EE_{\beta}\left[\big(r+\gamma Q_\theta(x',a') - Q_\theta(x,a)\big)^2\right] -\eta' \cdot\EE_{\beta} \bigg[\EE_{a\sim \pi} Q_\theta(x,a) - \log {\sum_{a\in\cA} \exp\left(Q_\theta(x,a)\right)} \bigg].
\end{equation*}

\paragraph*{Model-based algorithm.}
For the model-based variant \texttt{MEX-MB}, we use the following objective:
\begin{align}\label{eq_mb_obj}
\max_\phi\max_\pi \,\EE_{(x, a, r, x')\sim\mathcal{D}}\left[\log\mathbb{P}_{\phi}(x', r\mid x, a)\right] + \eta'\cdot \EE_{x\sim \sigma}\big[V^\pi_{\mathbb{P}_\phi}(x)\big],
\end{align}
where we denote by $\sigma(\cdot)$ the initial state distribution, $\mathcal{D}$ the replay buffer, and $\eta'$ corresponds to $1/\eta$ in the previous theory sections. We leverage the \emph{score function} to obtain the model value gradient $\nabla_\phi V^\pi_{\mathbb{P}_\phi}$ in a similar way to likelihood ratio policy gradient \citep{sutton1999policy}, with the gradient of action log-likelihood replaced by the gradient of state and reward log-likelihood in the model. Specifically,
\begin{equation}
\label{eq_model_grad}
    \nabla_\phi\,\EE_{x\sim \sigma}\big[V^\pi_{\mathbb{P}_\phi}(x)\big] = \EE_{\tau^\pi_\phi}\Big[\big(r + \gamma V^\pi_{\mathbb{P}_\phi}(x') - Q^\pi_{\mathbb{P}_\phi}(x, a)\big)\cdot\nabla_\phi\log\mathbb{P}_\phi(x', r \mid x, a)\Big],
\end{equation}
where $\tau^\pi_\phi$ is the trajectory under policy $\pi$ and transition $\mathbb{P}_\phi$, starting from $\sigma$. 
We refer the readers to previous works \citep{rigter2022rambo,wu2022bayesian} for a derivation of \eqref{eq_model_grad}. The model $\phi$ and policy $\pi$ in \eqref{eq_mb_obj} are updated iteratively in a \texttt{Dyna} \citep{sutton1990integrated} style, where model-free policy updates are performed on model-generated data. 
Particularly, we adopt \texttt{SAC} \citep{haarnoja2018soft} to update the policy $\pi$ and estimate the value $Q^\pi_{\mathbb{P}_\phi}$ using the model data generated by the model $\mathbb{P}_\phi$. 
We also follow \cite{rigter2022rambo} to update the model using mini-batches from $\mathcal{D}$ and normalize the advantage $r + \gamma V^\pi_{\mathbb{P}_\phi} - Q^\pi_{\mathbb{P}_\phi}$ within each mini-batch.
We refer the readers to Appendix~\ref{subsec: details mex mb} for more implementation details of \texttt{MEX-MB}.

\subsection{Experimental Results}
We report the performance of \texttt{MEX-MF} and \texttt{MEX-MB} in Figures \ref{fig_overall_mf} and \ref{fig_overall_mb}, respectively. 

\paragraph*{Results for MEX-MF.} We compare \texttt{MEX-MF} with the model-free baseline \texttt{TD3} \citep{fujimoto2018addressing}. 
We observe that \texttt{TD3} fails in many sparse reward tasks, while \texttt{MEX-MF} can significantly boost the performance. 
In standard MuJoCo gym tasks, \texttt{MEX-MF} also steadily outperforms \texttt{TD3} with faster convergence and higher returns.

\begin{figure}[H]
\centering
\subfigure{
    \begin{minipage}[t]{0.3\linewidth}
        \centering
        \includegraphics[width=1\textwidth]{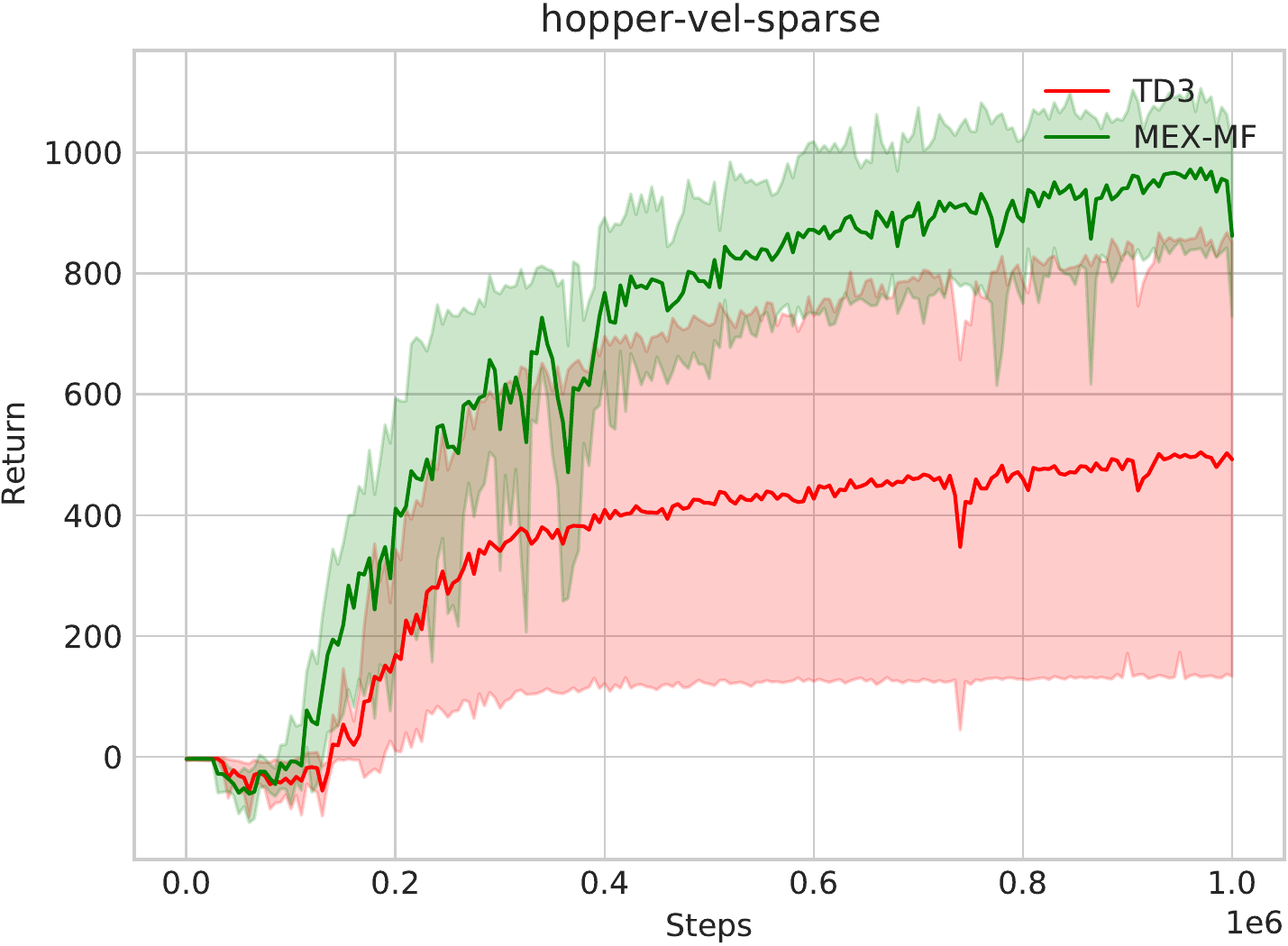}\\
    \end{minipage}
}
\subfigure{
    \begin{minipage}[t]{0.3\linewidth}
        \centering
        \includegraphics[width=1\textwidth]{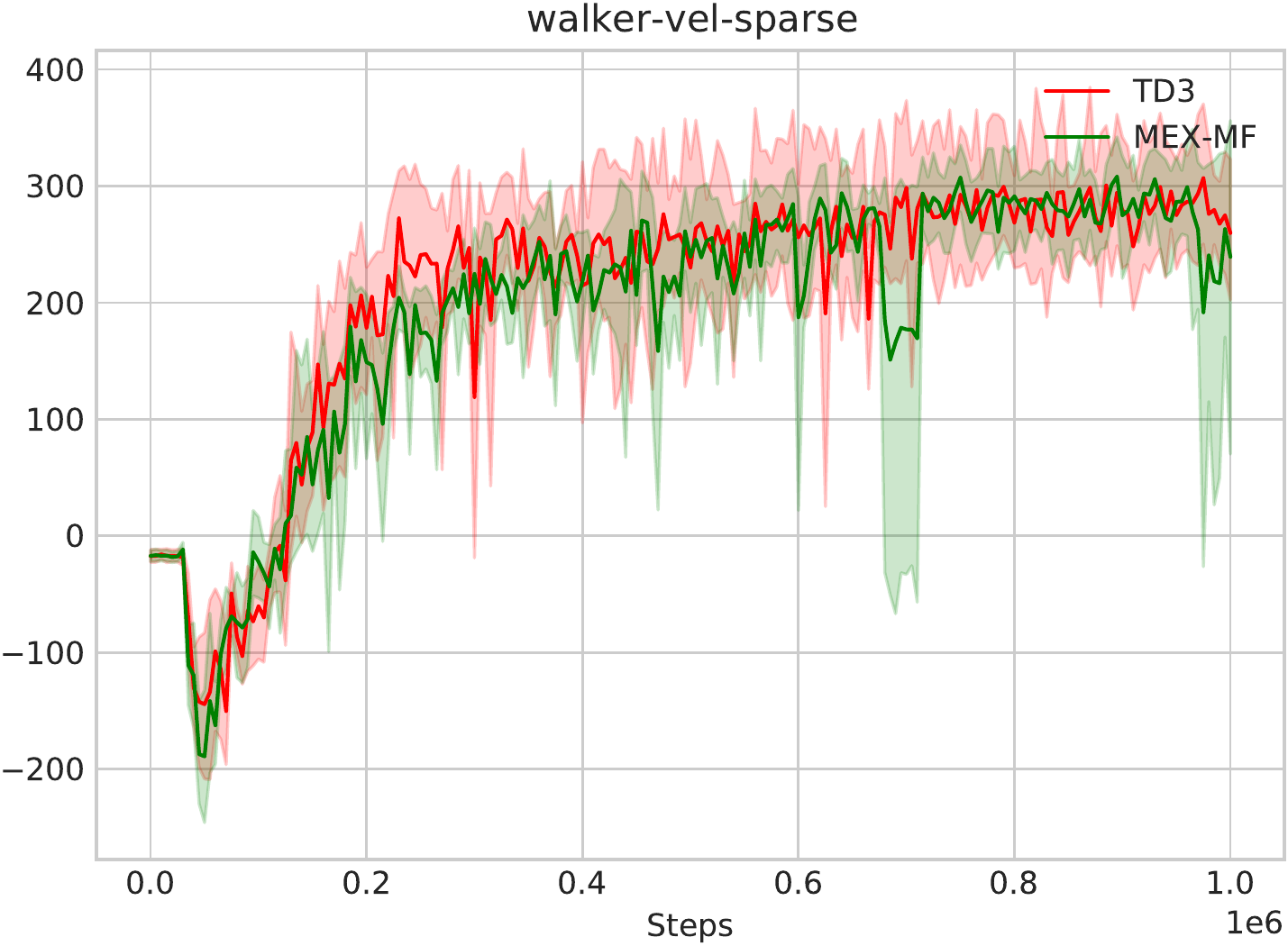}\\
    \end{minipage}
}
\subfigure{
    \begin{minipage}[t]{0.3\linewidth}
        \centering
        \includegraphics[width=1\textwidth]{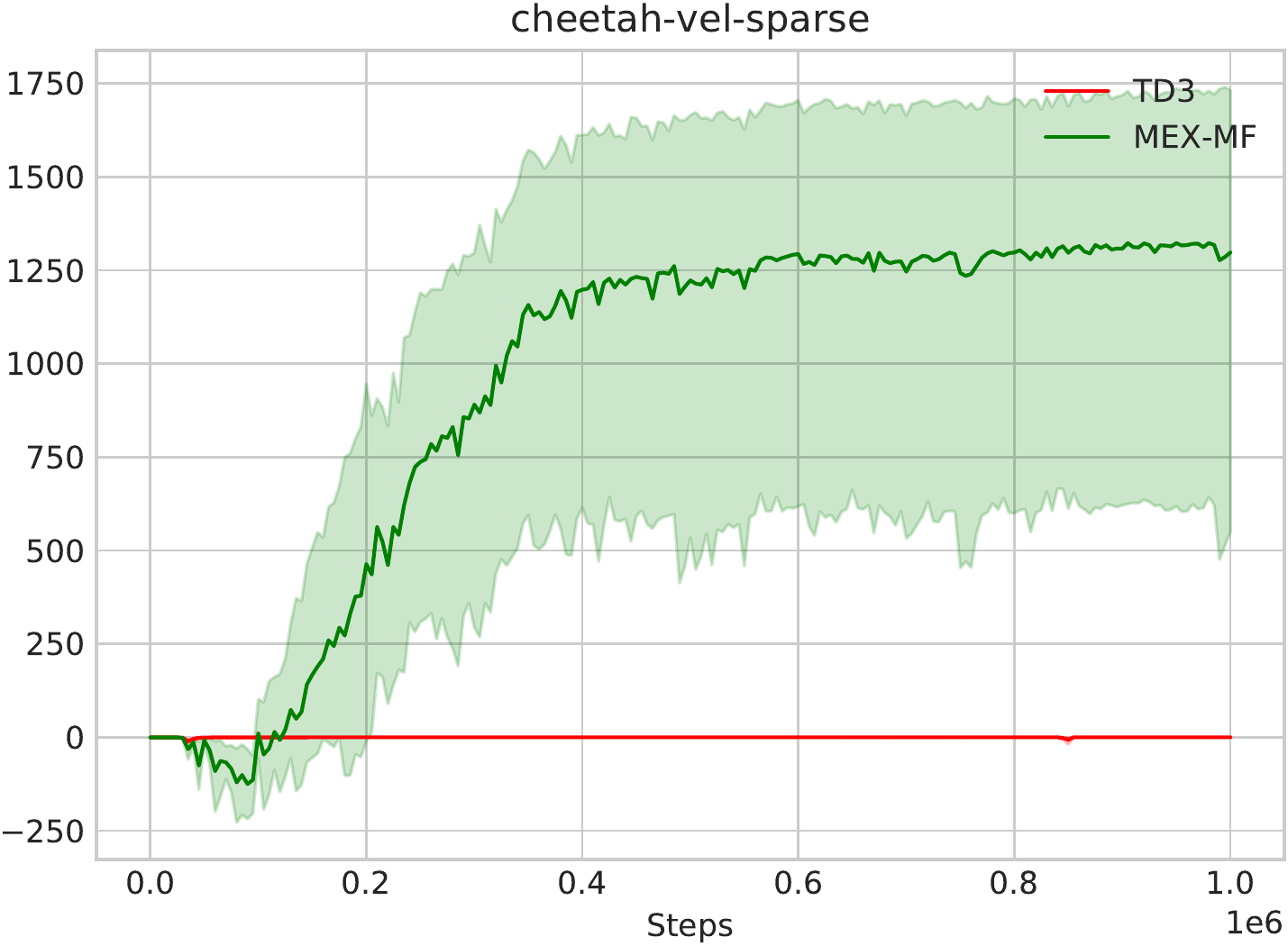}\\
    \end{minipage}
}
\subfigure{
    \begin{minipage}[t]{0.3\linewidth}
        \centering
        \includegraphics[width=1\textwidth]{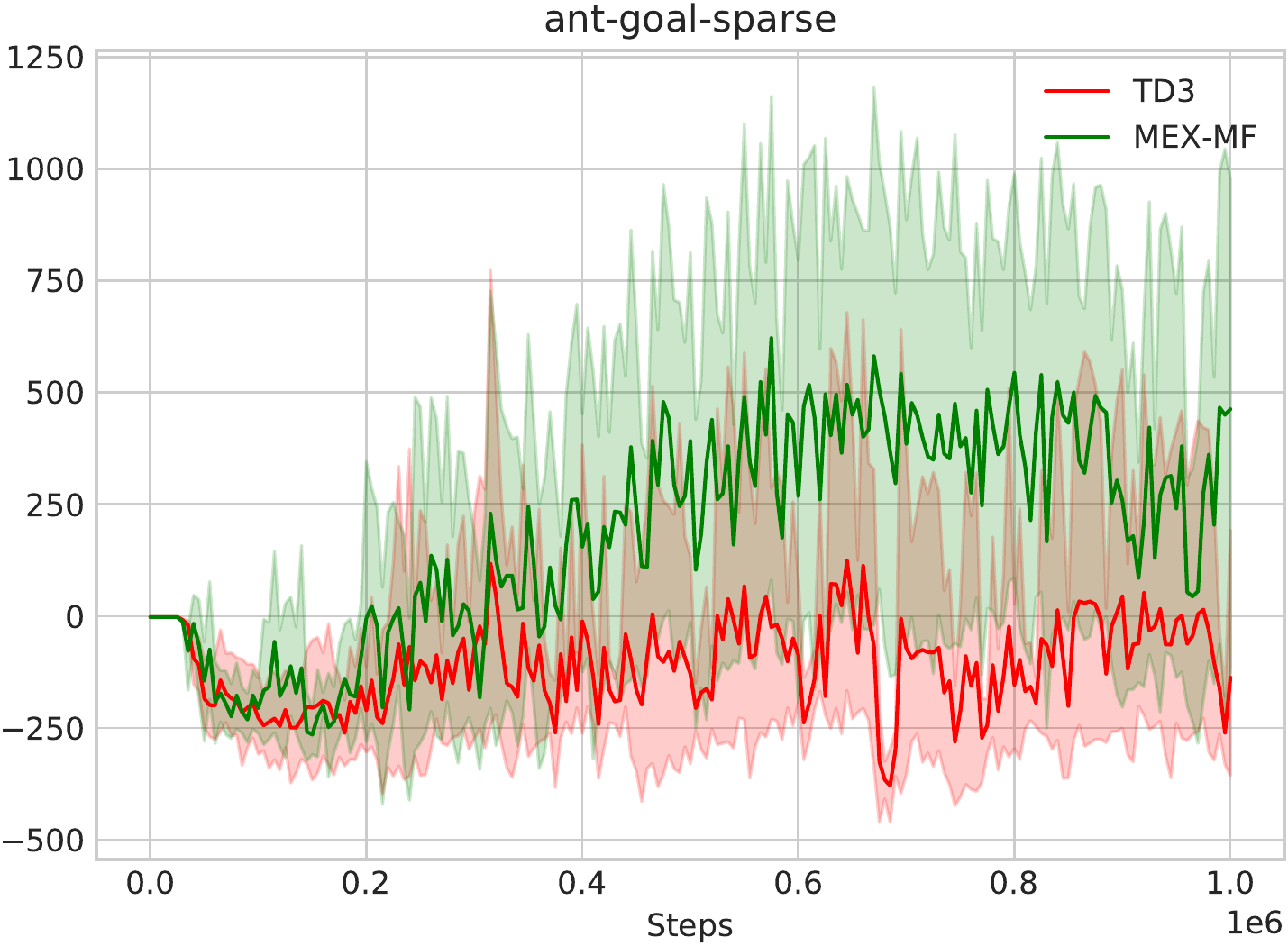}\\
    \end{minipage}%
}
\subfigure{
    \begin{minipage}[t]{0.3\linewidth}
        \centering
        \includegraphics[width=1\textwidth]{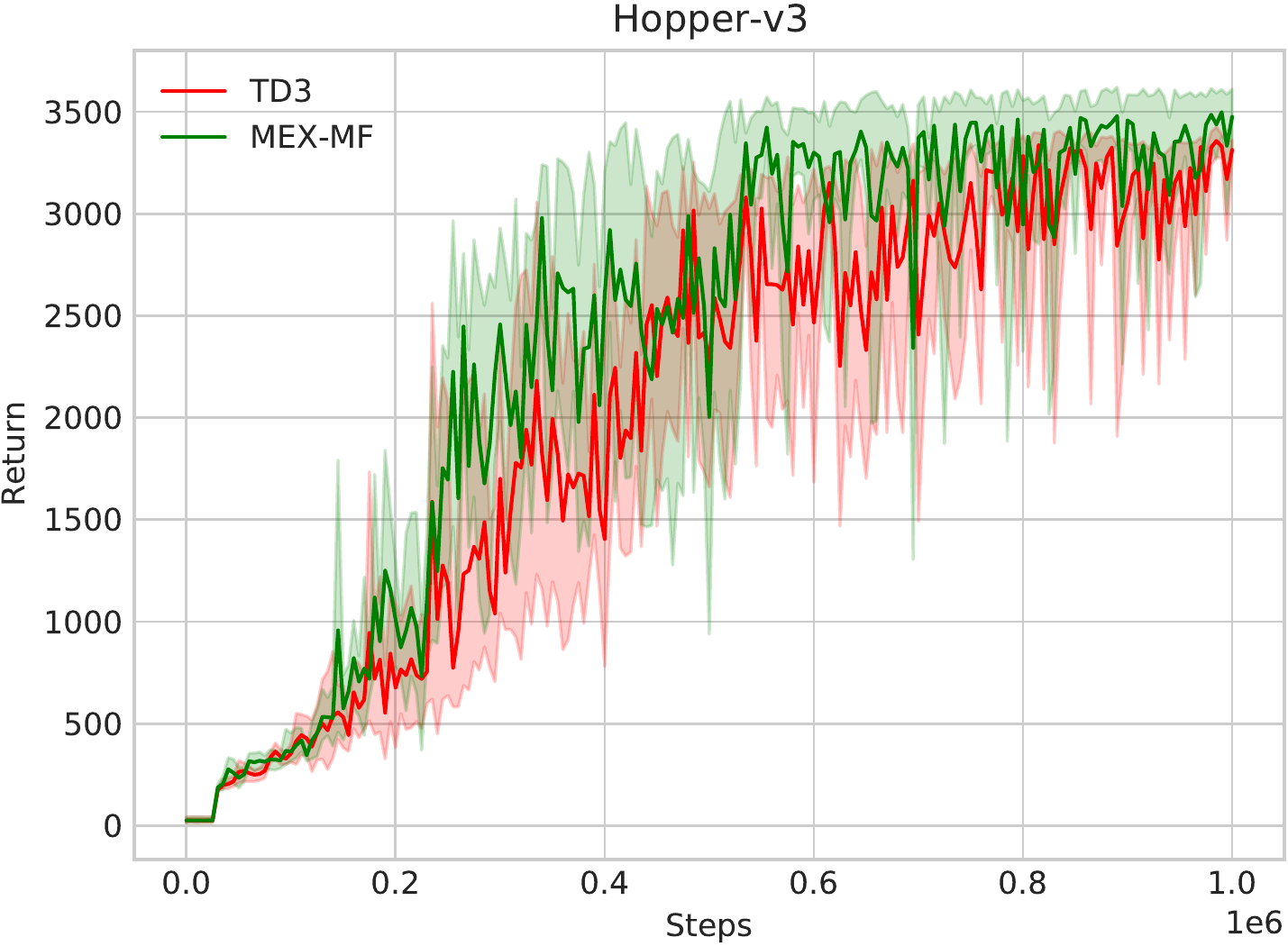}\\
    \end{minipage}
}
\subfigure{
    \begin{minipage}[t]{0.3\linewidth}
        \centering
        \includegraphics[width=1\textwidth]{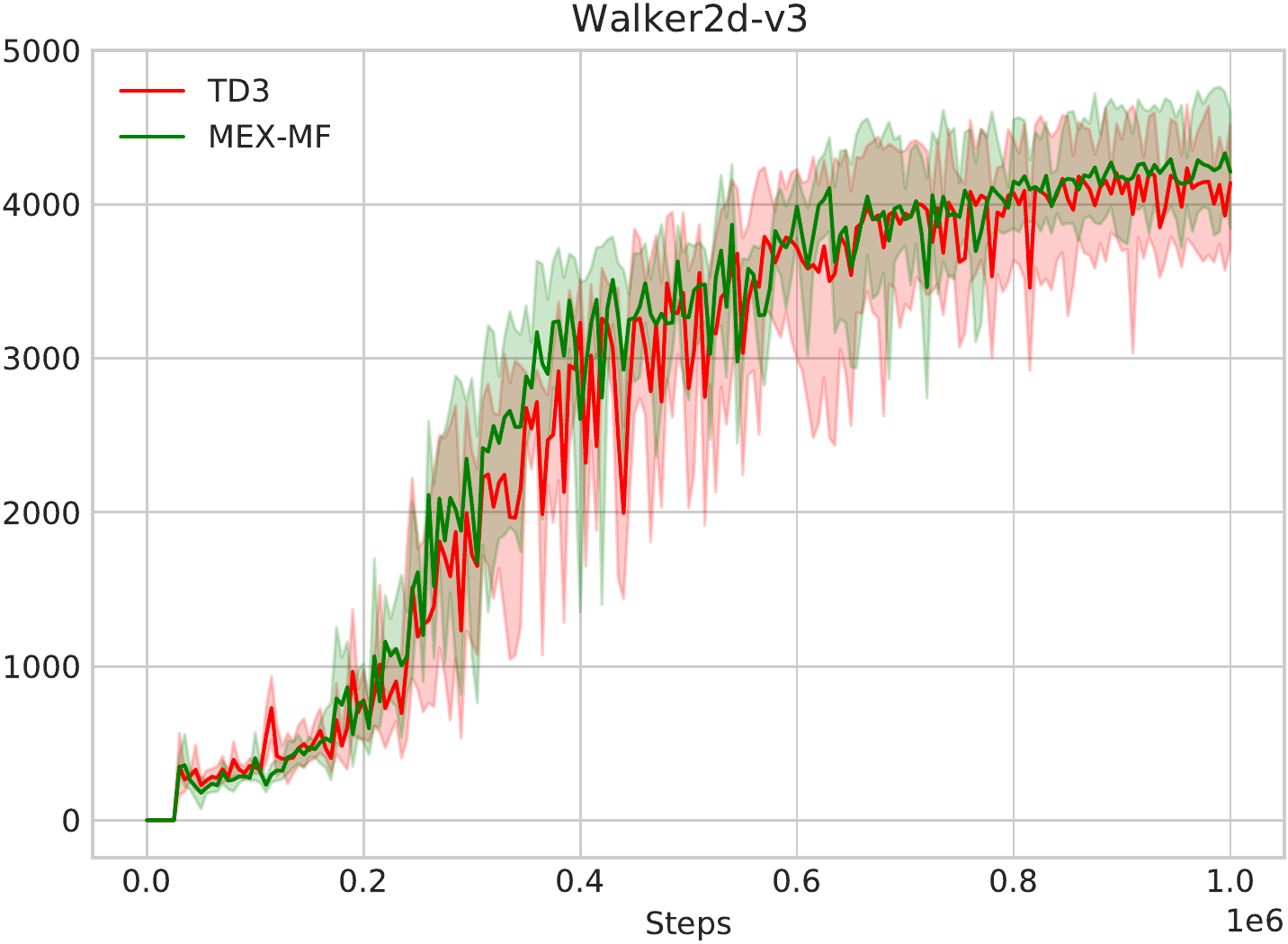}\\
    \end{minipage}
}
\subfigure{
    \begin{minipage}[t]{0.3\linewidth}
        \centering
        \includegraphics[width=1\textwidth]{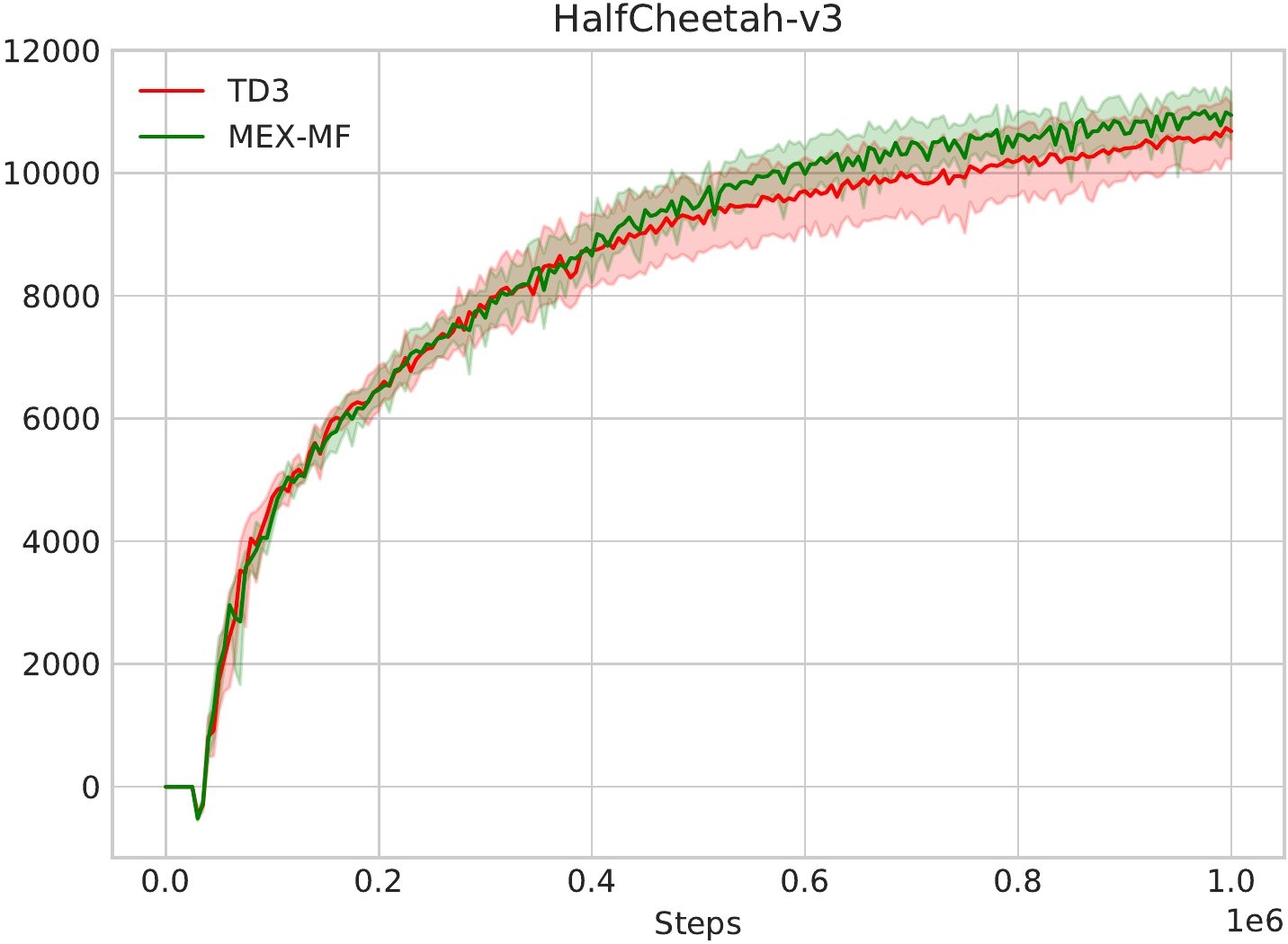}\\
    \end{minipage}
}
\subfigure{
    \begin{minipage}[t]{0.3\linewidth}
        \centering
        \includegraphics[width=1\textwidth]{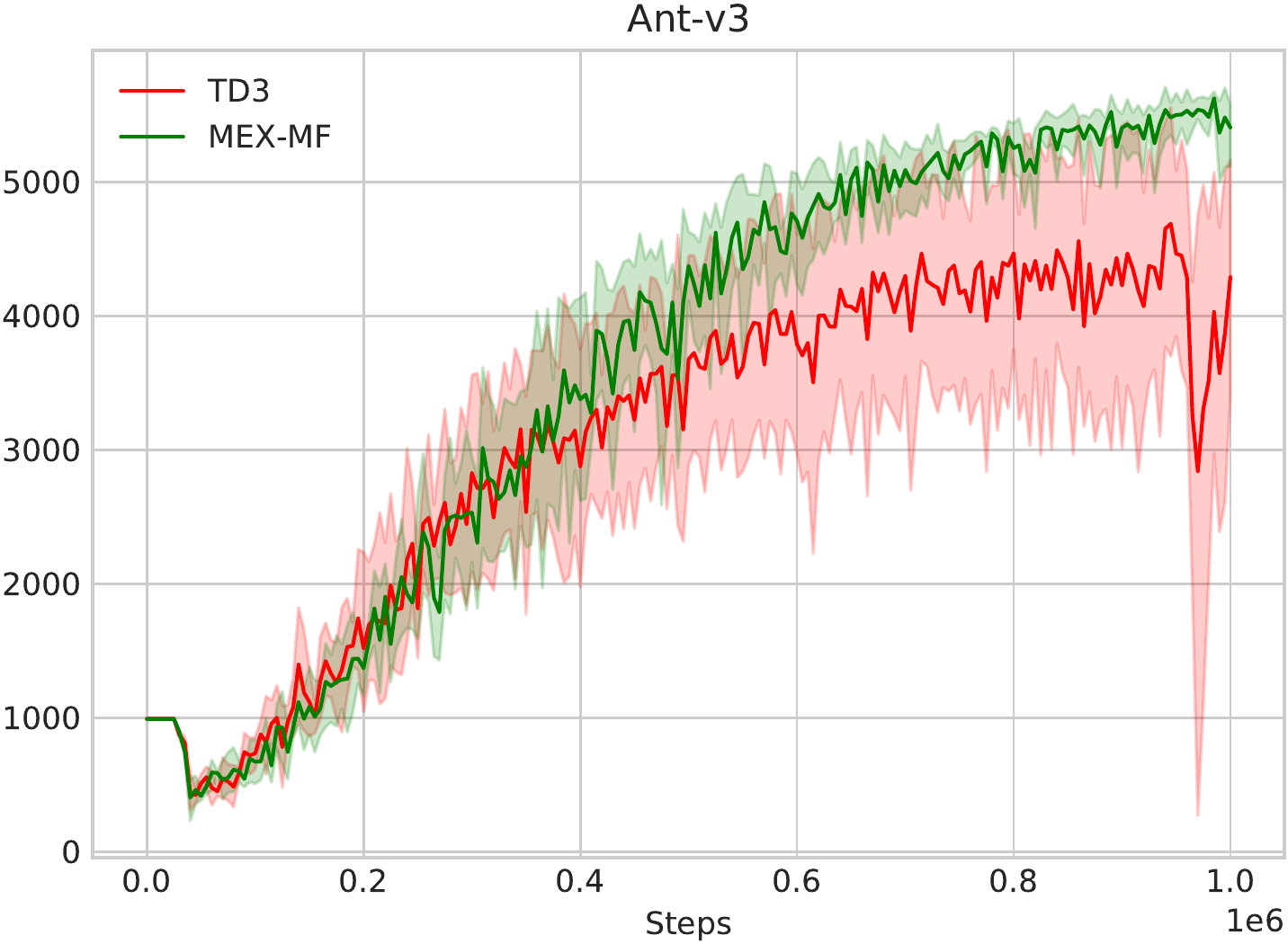}\\
    \end{minipage}%
}
\setlength{\belowcaptionskip}{-10pt}
\caption{Model-free \texttt{MEX-MF} in sparse and standard MuJoCo locomotion tasks.}
\label{fig_overall_mf}
\end{figure}

\paragraph*{Results for MEX-MB.} We compare \texttt{MEX-MB} with \texttt{MBPO} \citep{janner2019trust}, where our method differs from \texttt{MBPO} \textit{only} in the inclusion of the value gradient in \eqref{eq_model_grad} during model updates. 
We find that \texttt{MEX-MB} offers an easy implementation with minimal computational overhead and yet remains highly effective across sparse and standard MuJoCo tasks.
Notably, in the sparse reward settings, \texttt{MEX-MB} excels at achieving the goal velocity and outperforms \texttt{MBPO} by a stable margin. In standard gym tasks, \texttt{MEX-MB} showcases greater sample efficiency in challenging high-dimensional tasks with higher asymptotic returns.

\begin{figure}[t!]
    \centering
    \subfigure{
        \begin{minipage}[t]{0.3\linewidth}
            \centering
            \includegraphics[width=1\textwidth]{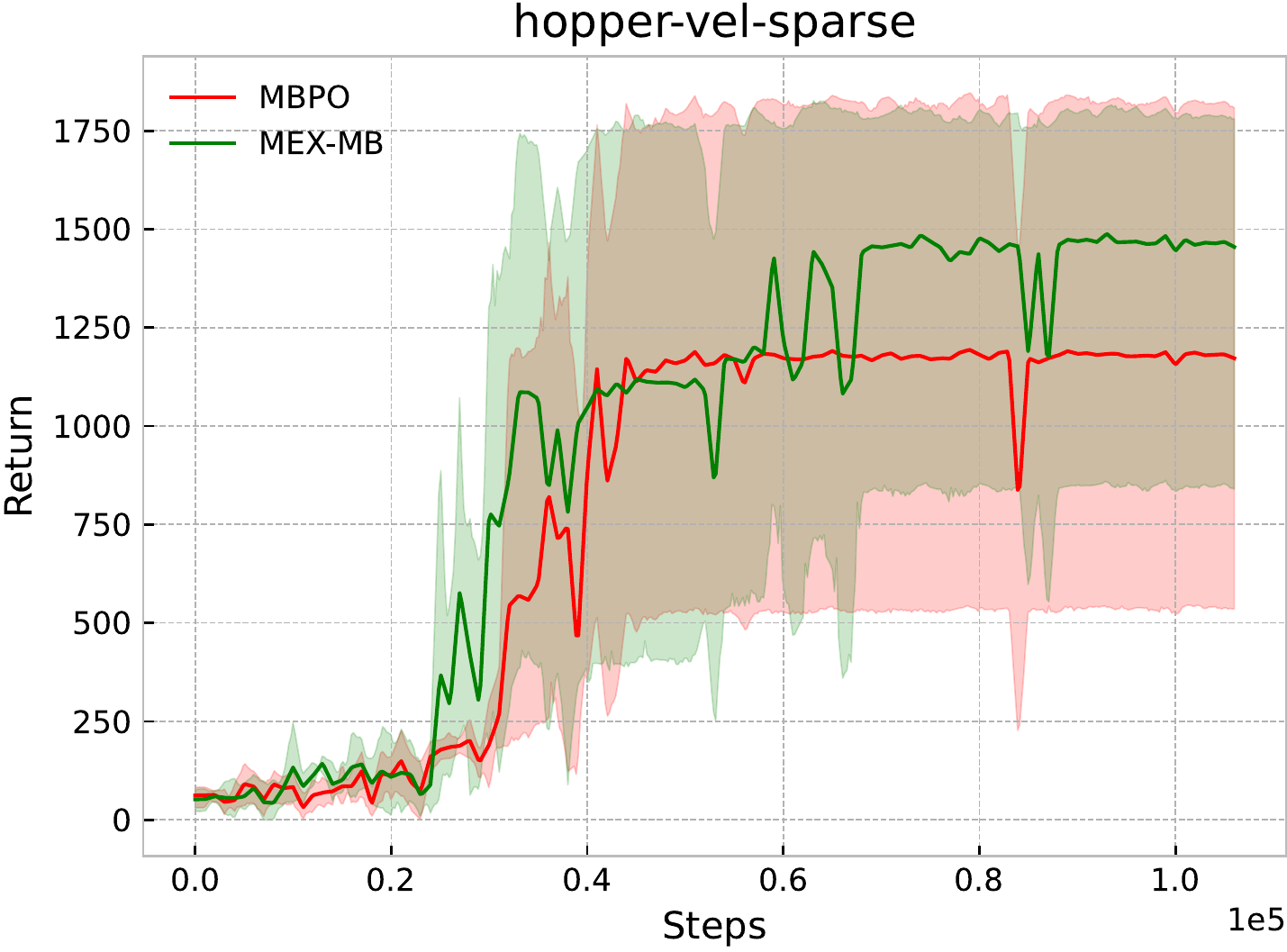}\\
        \end{minipage}%
    }
    \subfigure{
        \begin{minipage}[t]{0.3\linewidth}
            \centering
            \includegraphics[width=1\textwidth]{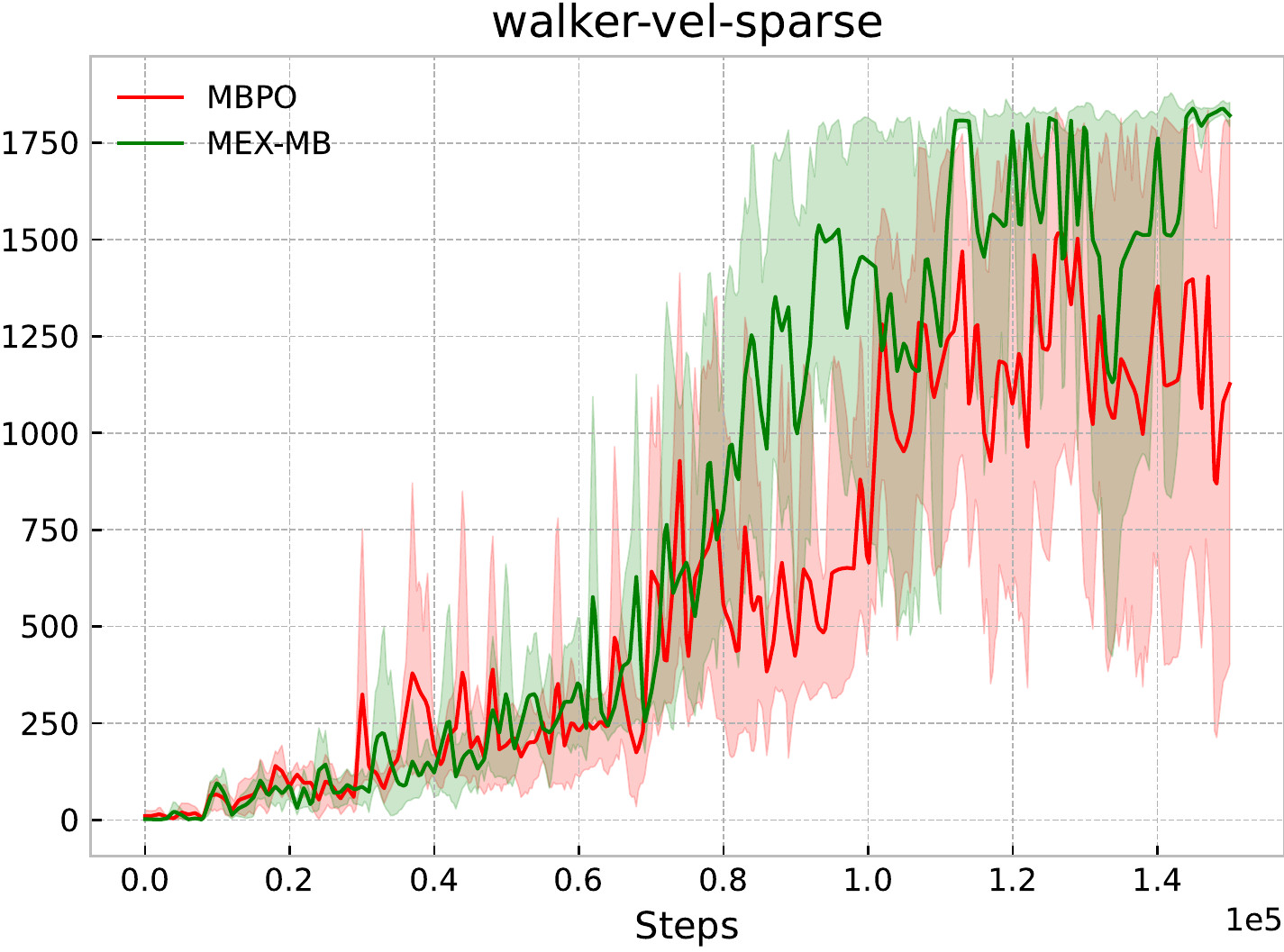}\\
        \end{minipage}
    }
    \subfigure{
        \begin{minipage}[t]{0.3\linewidth}
            \centering
            \includegraphics[width=1\textwidth]{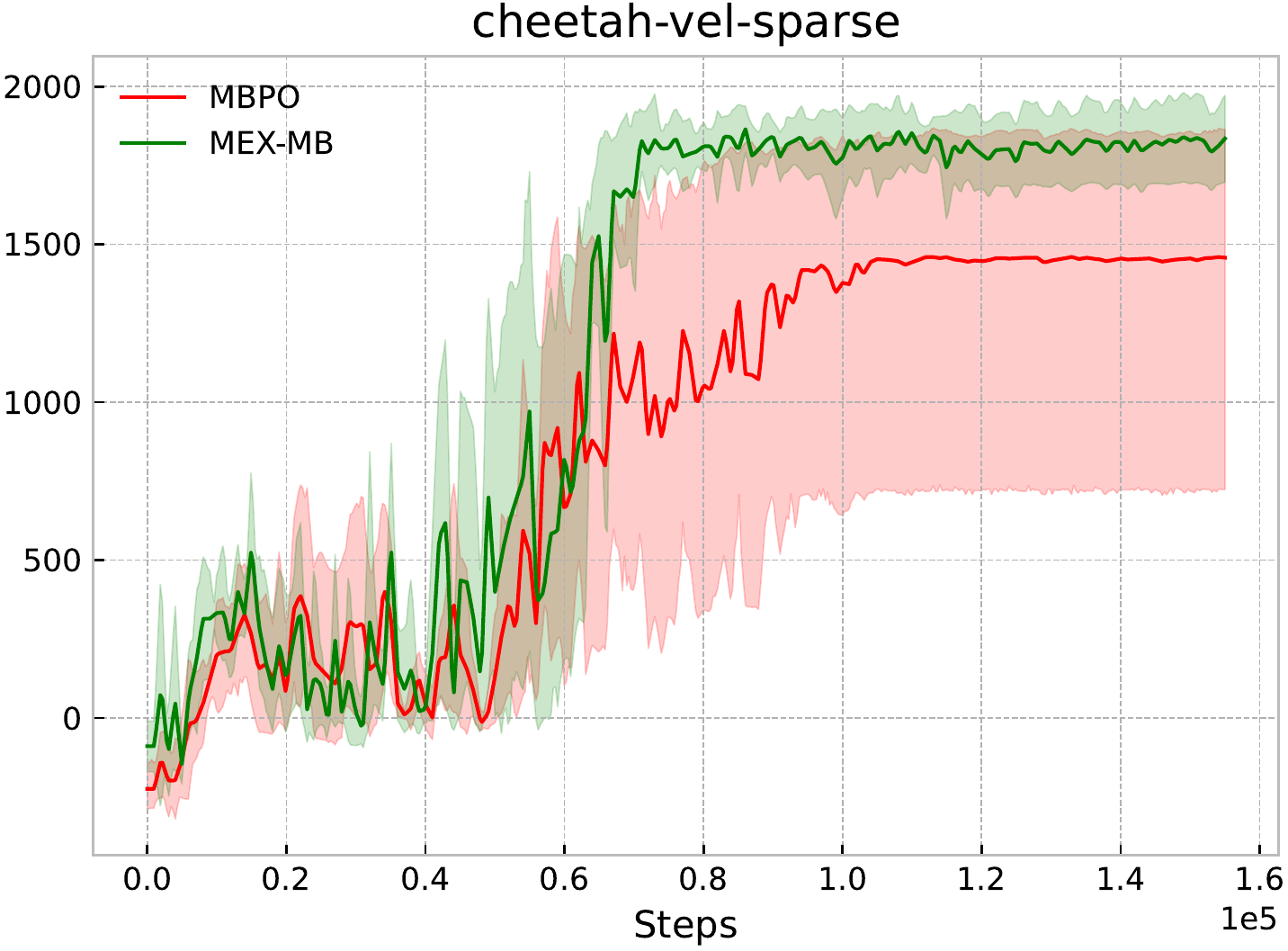}\\
        \end{minipage}
    }
    \subfigure{
        \begin{minipage}[t]{0.3\linewidth}
            \centering
            \includegraphics[width=1\textwidth]{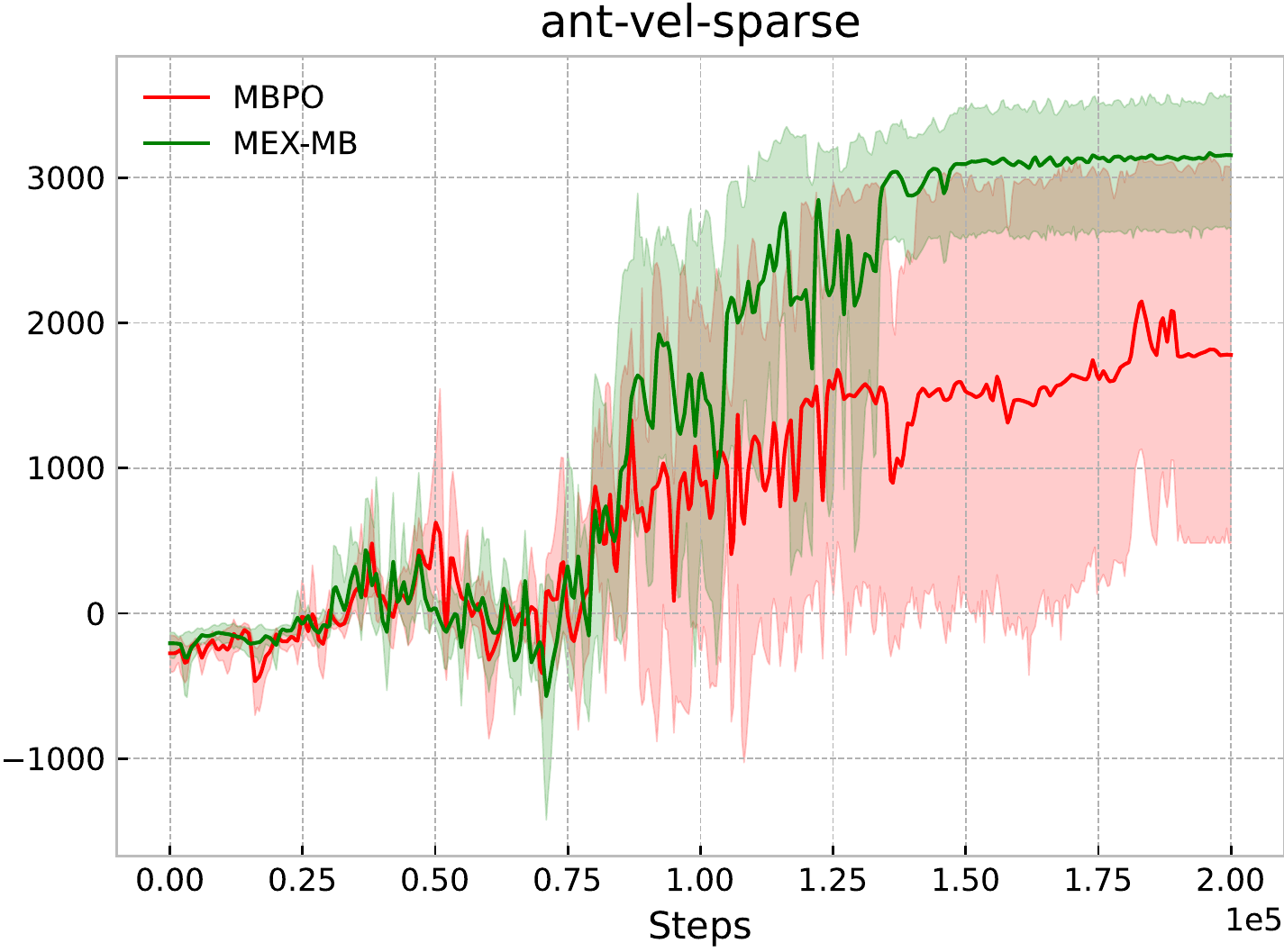}\\
        \end{minipage}
    }
    \subfigure{
        \begin{minipage}[t]{0.3\linewidth}
            \centering
            \includegraphics[width=1\textwidth]{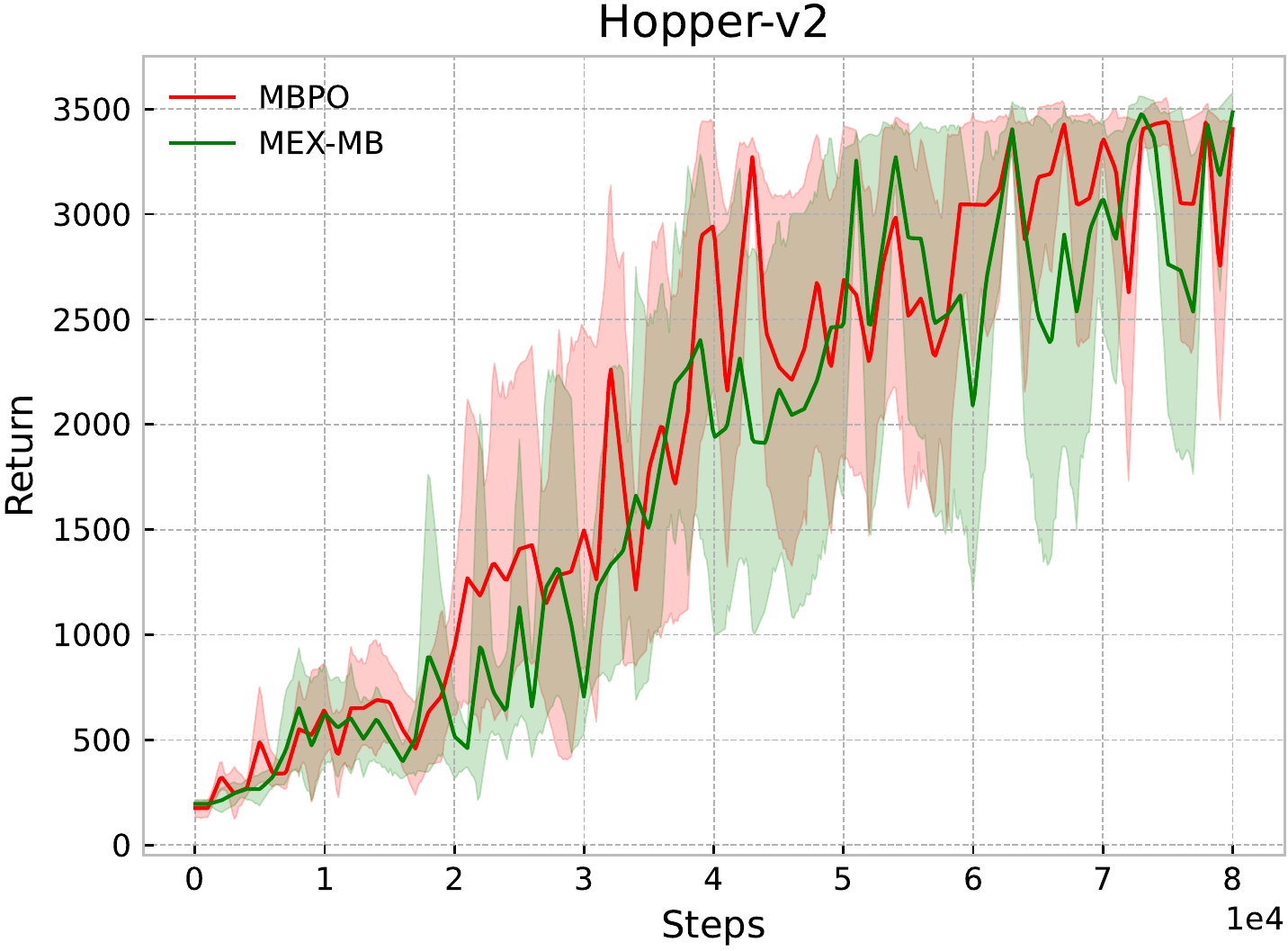}\\
        \end{minipage}%
    }
    \subfigure{
        \begin{minipage}[t]{0.3\linewidth}
            \centering
            \includegraphics[width=1\textwidth]{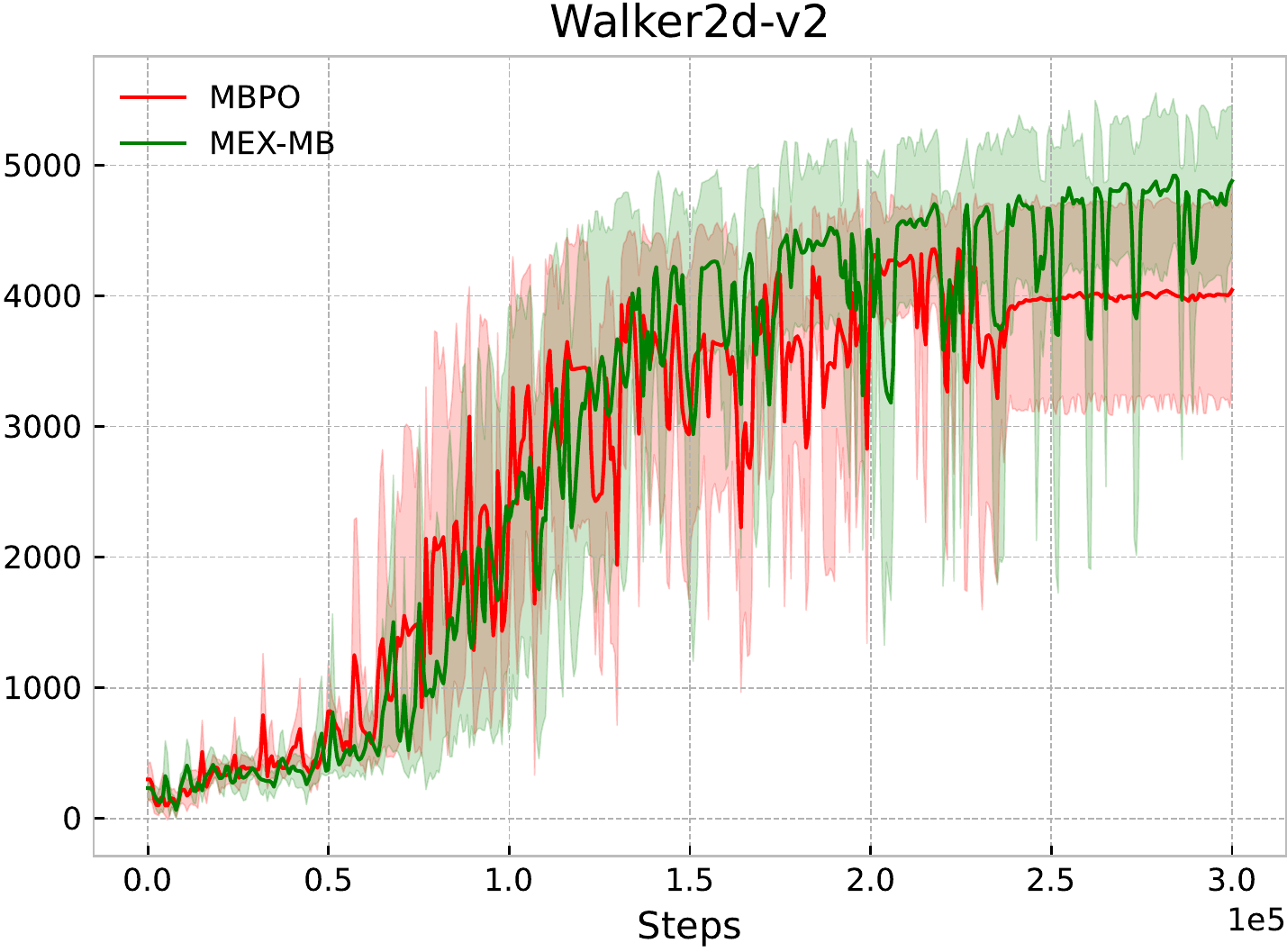}\\
        \end{minipage}
    }
    \subfigure{
        \begin{minipage}[t]{0.3\linewidth}
            \centering
            \includegraphics[width=1\textwidth]{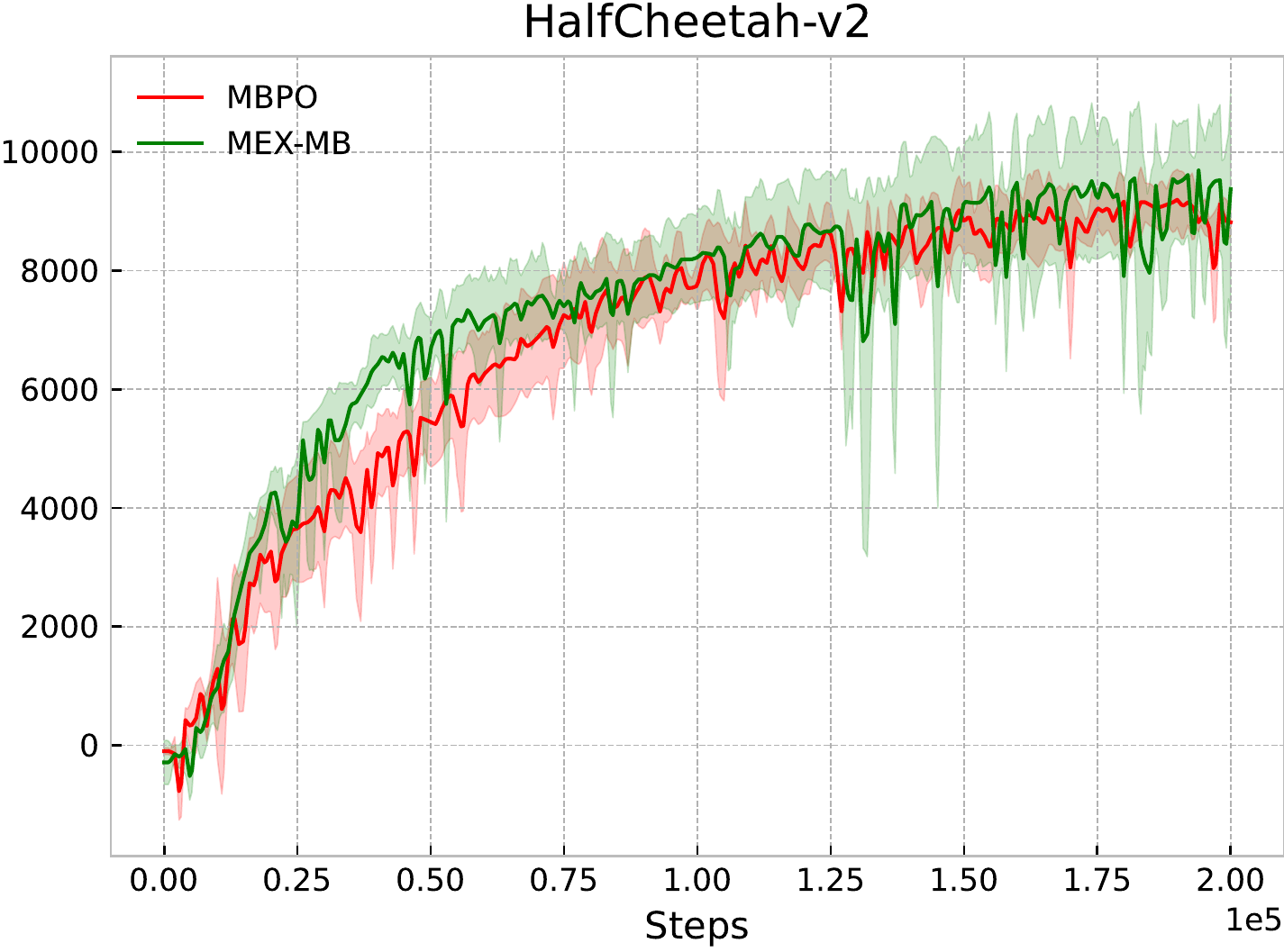}\\
        \end{minipage}
    }
    \subfigure{
        \begin{minipage}[t]{0.3\linewidth}
            \centering
            \includegraphics[width=1\textwidth]{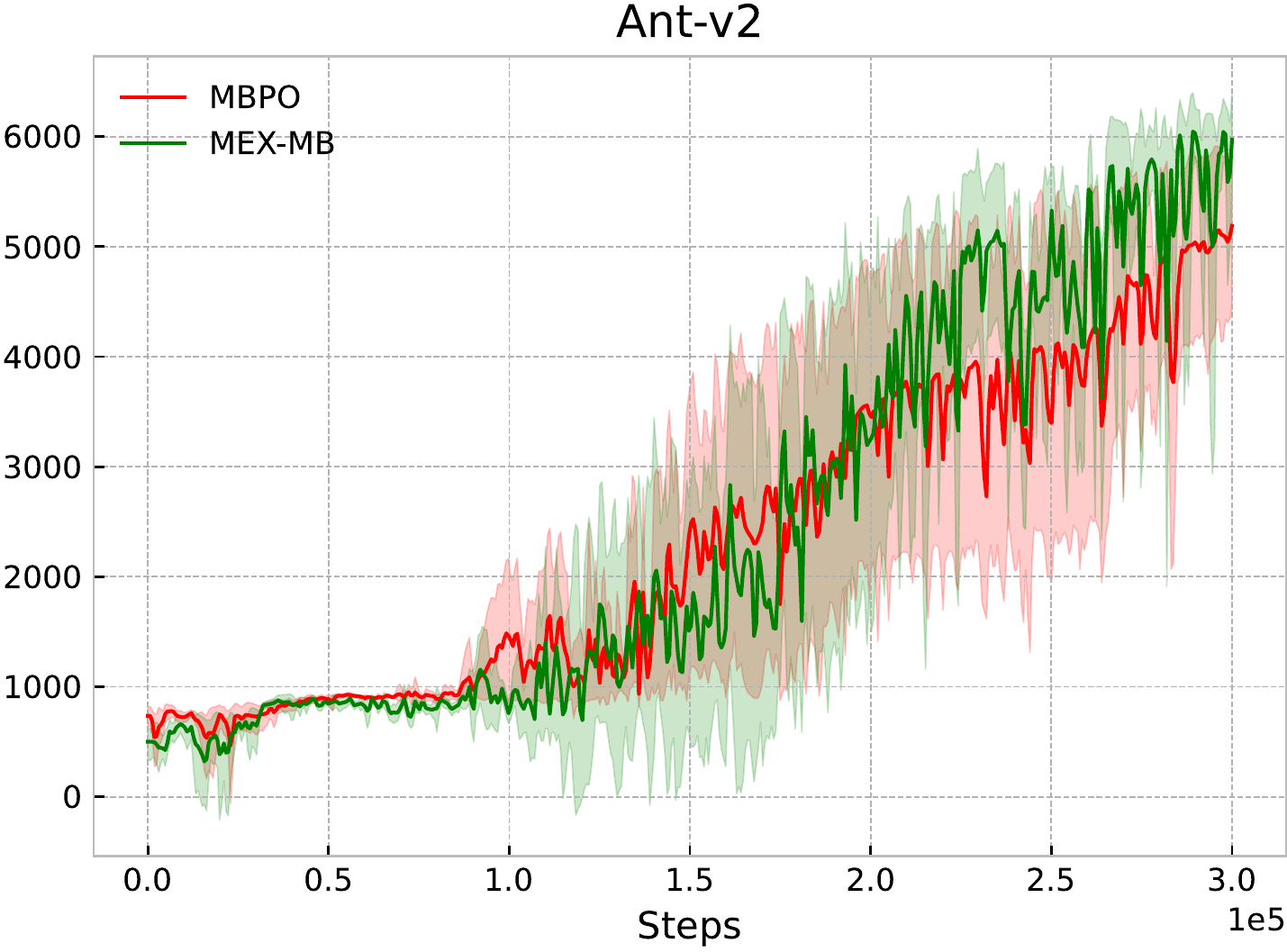}\\
        \end{minipage}
    }
    \setlength{\belowcaptionskip}{-10pt}
    \caption{Model-based \texttt{MEX-MB} in sparse and standard MuJoCo locomotion tasks.}
    \label{fig_overall_mb}
    \end{figure}

%% file: tex/conclusion.tex
\section{Conclusions}

In this paper, we propose a novel online RL algorithm framework \emph{Maximize to Explore} (\texttt{MEX}), aimed at striking a balance between exploration and exploitation in online learning scenarios.
\texttt{MEX} is provably sample-efficient with general function approximations and is easy to implement.
Theoretically, we prove that under mild structural assumptions (low generalized eluder coefficient (GEC)), \texttt{MEX} achieves $\tilde{\mathcal{O}}(\sqrt{K})$-online regret for Markov decision processes. 
We further extend the definition of GEC and the \texttt{MEX} framework to two-player zero-sum Markov games and also prove the $\tilde{\mathcal{O}}(\sqrt{K})$-online regret.
In practice, we adapt \texttt{MEX} to deep RL methods in both model-based and model-free styles and apply them to sparse-reward MuJoCo environments, outperforming baselines significantly. 
We hope that our work can shed light on future research of designing both statistically efficient and practically effective RL algorithms with powerful function approximations.

%% file: tex/appendix/proof_main.tex
\section{Proof of Main Theoretical Results}

\subsection{Proof of Theorem \ref{thm:reg}}\label{subsec: proof regret mdp}

\begin{proof}[Proof of Theorem \ref{thm:reg}]
    Consider the following decomposition of the regret,
    \begin{align}
        \mathrm{Regret}(K) &= \sum_{k=1}^K V_{1}^{\ast}(x_1) - V_{1}^{\pi_{f^k}}(x_1) \notag \\
            & = \underbrace{\sum_{k=1}^K V_{1}^{\ast}(x_1) -  V_{1,f^k}(x_1)}_{\displaystyle{\text{Term (i)}}} + \underbrace{\sum_{k=1}^K V_{1,f^k}(x_1) - V_{1}^{\pi_{f^k}}(x_1)}_{\displaystyle{\text{Term (ii)}}}
    \end{align}
    
    \paragraph{Term (i).} Note that by our definition in both Example \ref{exp: model based} and \ref{exp: model free}, we have that $V_1^{\ast} = V_{1,f^{\ast}}$.
    Thus we can rewrite the term (i) as 
    \begin{align}\label{eq: proof regret mdp 1}
        \text{Term (i)} = \sum_{k=1}^K V_{1,f^{\ast}}(x_1) - V_{1,f^k}(x_1).
    \end{align}
    Then by our choice of $f^k$ in \eqref{eq: implicit optimism mdp} and the fact that $f^{\ast}\in\mathcal{H}$, we have that for each $k\in[K]$,
    \begin{align}\label{eq: proof regret mdp 2}
        V_{1,f^{\ast}}(x_1) - \eta\sum_{h=1}^HL_h^{k-1}(f^{\ast})(x_1) \leq  V_{1,f^{k}}(x_1) - \eta\sum_{h=1}^HL_h^{k-1}(f^{k})(x_1)
    \end{align}
    By combining \eqref{eq: proof regret mdp 1} and \eqref{eq: proof regret mdp 2}, we can derive that 
    \begin{align}\label{eq: proof regret mdp 3}
        \text{Term (i)} \leq \eta \sum_{k=1}^K\sum_{h=1}^HL_h^{k-1}(f^{\ast}) - L_h^{k-1}(f^{k})
    \end{align}
    Now by applying Assumption \ref{ass: supervised learning} to \eqref{eq: proof regret mdp 3}, we can further derive that with probability at least $1-\delta$,
    \begin{align}\label{eq: proof regret mdp term i}
        \text{Term (i)} \leq -c_{\mathrm{(i)}}\cdot\eta\sum_{k=1}^K\sum_{s=1}^{k-1}\sum_{h=1}^H\mathbb{E}_{\xi_h\sim \pi_{\mathrm{exp}}(f^s)}[\ell_{f^s}(f^k;\xi_h)] + c_{\mathrm{(i)}}\cdot\eta B K\big(H\log(HK/\delta) + \log(|\mathcal{H}|)\big).
    \end{align}
    where $c_{\mathrm{(i)}}>0$ is some absolute constant (recall the definition of $\lesssim$).

    \paragraph{Term (ii).}
    For term (ii) of \eqref{eq: proof regret mdp 1}, we apply Assumption \ref{ass: gec} and obtain that, for any $\epsilon>0$,
    \begin{align}
        \text{Term (ii)} \leq\inf_{\mu > 0} \left\{\frac{\mu}{2} \sum_{h=1}^H\sum_{k=1}^K\sum_{s = 1}^{k-1} \mathbb{E}_{\xi_h\sim \pi_{\exp}(f^{s})}[
            \ell_{f^s}(f^k;\xi_h)] + \frac{d_{\mathrm{GEC}}(\epsilon)}{2\mu} + \sqrt{d_{\mathrm{GEC}}(\epsilon)HK} + \epsilon HK \right\}.\notag
    \end{align}
    By taking $\mu/2 = c_{\mathrm{(i)}}\cdot\eta$, we can further derive that,
    \begin{align}\label{eq: proof regret mdp term ii}
        \text{Term (ii)} \leq c_{\mathrm{(i)}}\cdot \eta \sum_{h=1}^H\sum_{k=1}^K\sum_{s = 1}^{k-1} \mathbb{E}_{\xi_h\sim \pi_{\exp}(f^{s})}[
            \ell_{f^s}(f^k;\xi_h)] + \frac{d_{\mathrm{GEC}}(\epsilon)}{4c_{\mathrm{(i)}}\eta} + \sqrt{d_{\mathrm{GEC}}(\epsilon)HK} + \epsilon HK. 
    \end{align}

    \paragraph{Combining Term (i) and Term (ii).}
    Now by combining \eqref{eq: proof regret mdp term i} and \eqref{eq: proof regret mdp term ii}, we can obtain that with probability at least $1-\delta$,
    \begin{align}
        \mathrm{Regret}(T) &=  \text{Term (i)}  +  \text{Term (ii)} \notag \\
        &\leq -c_{\mathrm{(i)}}\cdot\eta\sum_{k=1}^K\sum_{s=1}^{k-1}\sum_{h=1}^H\mathbb{E}_{\xi_h\sim \pi_{\mathrm{exp}}(f^s)}[\ell_{f^s}(f^k;\xi_h)] + c_{\text{(i)}} \cdot\eta B K\big(H\log(HK/\delta) + \log(|\mathcal{H}|)\big), \notag\\
        &\qquad + c_{\mathrm{(i)}}\cdot\eta \sum_{h=1}^H\sum_{k=1}^K\sum_{s = 1}^{k-1} \mathbb{E}_{\xi_h\sim \pi_{\exp}(f^{s})}[
            \ell_{f^s}(f^k;\xi_h)] + \frac{d_{\mathrm{GEC}}(\epsilon)}{4c_{\mathrm{(i)}}\eta} + \sqrt{d_{\mathrm{GEC}}(\epsilon)HK} + \epsilon HK \notag\\
        & = c_{\text{(i)}}\cdot \eta B K\big(H\log(HK/\delta) + \log(|\mathcal{H}|)\big) + \frac{d_{\mathrm{GEC}}(\epsilon)}{4c_{\mathrm{(i)}}\eta} + \sqrt{d_{\mathrm{GEC}}(\epsilon)HK} + \epsilon HK. \label{eq: proof regret 3}
    \end{align}
    By taking $\epsilon = 1/\sqrt{HK}$, $\eta = \sqrt{d_{\mathrm{GEC}}(\epsilon)/(H\log(HK/\delta) + \log(|\mathcal{H}|))\cdot B\cdot K)}$, we can derive from \eqref{eq: proof regret 3} that, with probability at least $1-\delta$, it holds that
    \begin{align}
        \mathrm{Regret}(K) %&\lesssim \sqrt{\mathtt{Poly}(H,B_l)\cdot K}\log(HK\mathcal{N}(\mathcal{H},1/K)/\delta) + (\sqrt{\mathtt{Poly}(H,B_l)}d_{\mathrm{GEC}}(1/\sqrt{HK})+ \sqrt{H})\sqrt{K}\notag \\
        \lesssim \sqrt{d_{\mathrm{GEC}}(1/\sqrt{HK})\cdot (H\log(HK/\delta) + \log(|\mathcal{H}|)) \cdot  B\cdot  K}.
    \end{align}
    This finishes the proof of Theorem \ref{thm:reg}.
\end{proof}

\subsection{Proof of Theorem \ref{thm: reg mg}}\label{subsec: proof regret mg}

\begin{proof}[Proof of Theorem \ref{thm: reg mg}]
    Consider the following decomposition of the regret,
    \begin{align}
        \mathrm{Regret}(K) &= \sum_{k=1}^K V_{1}^{\ast}(x_1) - V_{1}^{\mu^k,\dagger}(x_1) \notag \\
            & = \sum_{k=1}^K V_{1}^{\ast}(x_1) - V_{1}^{\mu^k,\nu^k}(x_1) + \sum_{k=1}^KV_{1}^{\mu^k,\nu^k}(x_1) - V_{1}^{\mu^k,\dagger}(x_1) \notag \\
            & = \underbrace{\sum_{k=1}^K V_{1}^{\star} -  V_{1,f^k}}_{\displaystyle{\text{Term (Max.i)}}} + \underbrace{\sum_{k=1}^K V_{1,f^k} - V_{1}^{\mu^k,\nu^k}}_{\displaystyle{\text{Term (Max.ii)}}} + \underbrace{\sum_{k=1}^K V_{1,g^k}^{\mu^k,\dagger} -  V_{1}^{\mu^k,\dagger}}_{\displaystyle{\text{Term (Min.i)}}} + \underbrace{\sum_{k=1}^K V_{1}^{\mu^k,\nu^k} - V_{1,g^k}^{\mu^k,\dagger}}_{\displaystyle{\text{Term (Min.ii)}}},
    \end{align}
    where in the last equality we omit the dependence on $x_1$ for simplicity.

    \paragraph{Term (Max.i).} Note that by our definition in both Example \ref{exp: model free mg} and Example \ref{exp: model based mg}, we have that $V_1^{\ast} = V_{1,f^{\ast}}$.
    Thus we can rewrite the term (Max.i) as 
    \begin{align}\label{eq: proof regret mg max 1}
        \text{Term (Max.i)} = \sum_{k=1}^K V_{1,f^{\ast}}(x_1) - V_{1,f^k}(x_1).
    \end{align}
    Then by our choice of $f^k$ in \eqref{eq: implicit optimism mg max} and the fact that $f^{\ast}\in\mathcal{H}$, we have that for each $k\in[K]$,
    \begin{align}\label{eq: proof regret mg max 2}
        V_{1,f^{\ast}}(x_1) - \eta\sum_{h=1}^HL_h^{k-1}(f^{\ast}) \leq  V_{1,f^{k}}(x_1) - \eta\sum_{h=1}^HL_h^{k-1}(f^{k}).
    \end{align}
    By combining \eqref{eq: proof regret mg max 1} and \eqref{eq: proof regret mg max 2}, we can derive that 
    \begin{align}\label{eq: proof regret mg max 3}
        \text{Term (Max.i)} \leq \eta \sum_{k=1}^K\sum_{h=1}^HL_h^{k-1}(f^{\ast}) - L_h^{k-1}(f^{k}).
    \end{align}
    Now applying Assumption \ref{ass: supervised learning mg} to \eqref{eq: proof regret mg max 3}, we can further derive that with probability at least $1-\delta$, 
    \begin{align}\label{eq: proof regret mg max 4}
        \text{Term (Max.i)} \leq - c_{\mathrm{(max.i)}}\cdot \eta \sum_{k=1}^K\sum_{s=1}^{k-1}\sum_{h=1}^H\mathbb{E}_{\xi_h\sim \boldsymbol{\pi}^k}[\ell_{f^s}(f;\xi_h)] +c_{\mathrm{(max.i)}}\cdot BK\big(H\log(HK/\delta) + \log(|\mathcal{H}|)\big),
    \end{align}
    for some absolute constant $c_{\mathrm{(max.i)}}>0$.
    
    \paragraph{Term (Max.ii).} For term (Max.ii), we apply Assumption \ref{ass: gec mg} and obtain that, for any $\epsilon>0$, 
    \begin{align*}
        \text{Term (Max.ii)} \le \inf_{\zeta > 0} \left\{\frac{\zeta}{2} \sum_{h=1}^H\sum_{k=1}^K\sum_{s = 1}^{k-1} \mathbb{E}_{\xi_h\sim \boldsymbol{\pi}^k}[
        \ell_{f^s}(f^k;\xi_h)] + \frac{d_{\mathrm{TGEC}}(\epsilon)}{2\zeta} + \sqrt{d_{\mathrm{TGEC}}(\epsilon)HK} + \epsilon HK \right\}.
    \end{align*}
    By taking $\zeta/2 = c_{\mathrm{(max.i)}}\cdot\eta$, we can further derive that 
    \begin{align}\label{eq: proof regret mg max 5}
        \text{Term (Max.ii)} \le c_{\mathrm{(max.i)}}\cdot\eta \sum_{h=1}^H\sum_{k=1}^K\sum_{s = 1}^{k-1} \mathbb{E}_{\xi_h\sim \boldsymbol{\pi}^k}[
        \ell_{f^s}(f^k;\xi_h)] + \frac{d_{\mathrm{TGEC}}(\epsilon)}{4c_{\mathrm{(max.i)}}\eta} + \sqrt{d_{\mathrm{TGEC}}(\epsilon)HK} + \epsilon HK .
    \end{align}

    \paragraph{Term (Min.i).} 
    For either model-free or model-based hypothsis, by our definition in Example \ref{exp: model free mg} and Example \ref{exp: model based mg} respectively, we both have that $V_{1}^{\mu^k,\dagger} = V_{1,\star}^{\mu^k,\dagger}$.
    Here $\star = Q^{\mu^k,\dagger}$ for model-free hypothesis and $\star=f^{\ast}$ for model-based hypothesis.
    %$V_1^{\mu^k,\dagger} = V_{1,Q^{\mu^k,\dagger}}^{\mu^k,\dagger}$\footnote{We remark that this notation is well-defined, since we assume that $Q^{\mu_f,\dagger}\in\mathcal{H}$ for any $f\in\mathcal{H}$ in Assumption \ref{ass: realizability}.}.
    Thus we can rewrite the term (Min.i) as\footnote{We remark that this notation is well-defined, since we assume that $Q^{\mu_f,\dagger}\in\mathcal{H}$ for any $f\in\mathcal{H}$ in Assumption \ref{ass: realizability}.}.
    \begin{align}\label{eq: proof regret mg min 1}
        \text{Term (Min.i)} = \sum_{k=1}^K V_{1,g^k}^{\mu^k,\dagger}(x_1) - V_{1,\star}^{\mu^k,\dagger}(x_1).
    \end{align}
    Then by our choice of $g^k$ in \eqref{eq: implicit optimism mg min} and the fact that $\star\in\mathcal{H}$ (Assumption \ref{ass: realizability mg}), we have that for each $k\in[K]$,
    \begin{align}\label{eq: proof regret mg min 2}
        -V_{1,\star}^{\mu^k,\dagger}(x_1) - \eta \sum_{h=1}^HL_{h,\mu^k}^{k-1}(\star) \leq -V_{1, g^k}^{\mu^k,\dagger}(x_1) - \eta \sum_{h=1}^HL_{h,\mu^k}^{k-1}(g^k)
    \end{align}
    By combining \eqref{eq: proof regret mg min 1} and \eqref{eq: proof regret mg min 2}, we can derive that 
    \begin{align}\label{eq: proof regret mg min 3}
        \text{Term (Min.i)} \leq \eta \sum_{k=1}^K\sum_{h=1}^HL_{h,\mu^k}^{k-1}(\star) - L_{h,\mu^k}^{k-1}(g^{k})
    \end{align}
    Now applying Assumption \ref{ass: supervised learning mg} to \eqref{eq: proof regret mg min 3}, we can further derive that with probability at least $1-\delta$, 
    \begin{align}\label{eq: proof regret mg min 4}
        \text{Term (Min.i)} \leq -c_{\mathrm{(min.i)}}\cdot\eta \sum_{k=1}^K\sum_{s=1}^{k-1}\sum_{h=1}^H\mathbb{E}_{\xi_h\sim \boldsymbol{\pi}^k}[\ell_{g^s, \mu^k}(g;\xi_h)] + c_{\mathrm{(min.i)}}\cdot\eta BK\big(H\log(HK/\delta) + \log(|\mathcal{H}|)\big).
    \end{align}

    \paragraph{Term (Min.ii).} For term (Min.ii), we apply Assumption \ref{ass: gec mg} and obtain that, for any $\epsilon>0$, 
    \begin{align*}
        \text{Term (Min.ii)} \le \inf_{\zeta > 0} \left\{\frac{\zeta}{2} \sum_{h=1}^H\sum_{k=1}^K\sum_{s = 1}^{k-1} \mathbb{E}_{\xi_h\sim \boldsymbol{\pi}^k}[
        \ell_{g^s, \mu^k}(g^k;\xi_h)] + \frac{d_{\mathrm{TGEC}}(\epsilon)}{2\zeta} + \sqrt{d_{\mathrm{TGEC}}(\epsilon)HK} + \epsilon HK \right\}.
    \end{align*}
    By taking $\zeta/2 = c_{\mathrm{(min.i)}}\cdot\eta$, we can further derive that 
    \begin{align}\label{eq: proof regret mg min 5}
        \!\!\!\!\!\!\text{Term (Max.ii)} \le c_{\mathrm{(min.i)}}\cdot\eta \sum_{h=1}^H\sum_{k=1}^K\sum_{s = 1}^{k-1} \mathbb{E}_{\xi_h\sim \boldsymbol{\pi}^k}[
        \ell_{g^s, \mu^k}(g^k;\xi_h)] + \frac{d_{\mathrm{TGEC}}(\epsilon)}{4c_{\mathrm{(min.i)}}\eta} + \sqrt{d_{\mathrm{TGEC}}(\epsilon)HK} + \epsilon HK .
    \end{align}

    \paragraph{Combining Term (Max.i),  Term (Max.ii), Term (Min.i), and Term (Min.ii).}
    Finally, combining \eqref{eq: proof regret mg max 4}, \eqref{eq: proof regret mg max 5}, \eqref{eq: proof regret mg min 4}, and \eqref{eq: proof regret mg max 5}, taking $\epsilon = 1/\sqrt{HK}$ and
    \begin{align}
        \eta = \sqrt{\frac{d_{\mathrm{TGEC}}(1/\sqrt{HK})}{(H\log(HK/\delta) + \log(|\mathcal{H}|))\cdot B\cdot K}},
    \end{align}
    we can finally derive that with probability at least $1-2\delta$,
    \begin{align*}
        \mathrm{Regret}(K) \lesssim  \sqrt{d_{\mathrm{TGEC}}(1/\sqrt{HK})\cdot(H\log(HK/\delta) + \log(|\mathcal{H}|))\cdot  B\cdot K}.
    \end{align*}
    This finishes the proof of Theorem \ref{thm: reg mg}.
\end{proof}

%% file: tex/appendix/proof_mdp_example.tex
\section{Examples of Model-based and Model-free Online RL in MDPs}\label{sec: proof examples mdp}

In this section, we specify Corollaries \ref{cor: regret model free mdp} and \ref{cor: regret model based mdp} to various examples of MDPs with low generalized eluder coefficient (GEC \cite{zhong2022posterior}).
Sections~\ref{subsec: example model free mdp} and \ref{subsec: example model based mdp} consider model-free hypothesis and model-based hypothesis, respectively. 
After, we give proof of the generalization guarantees involved in Section \ref{sec: examples mdp}.
Section~\ref{subsec: proof prop supervised learning model free rl} provides proof of Proposition~\ref{prop: supervised learning model free rl} and Section~\ref{subsec: proof prop supervised learning model based rl} provides proof of Proposition~\ref{prop: supervised learning model based rl}.

\subsection{Examples of Model-free Online RL in MDPs}\label{subsec: example model free mdp}

\paragraph{MDPs with low Bellman eluder dimension.}

In this part, we study MDPs with low Bellman eluder (BE) dimension \citep{jin2021bellman}.
To introduce, we define the notion of $\epsilon$-independence between distributions and the notion of distributional eluder dimension.

\begin{definition}[$\epsilon$-independence between distributions]
    Let $\mathcal{G}$ be a function class on the space $\mathcal{X}$, and let $\nu, \mu_1, \cdots, \mu_n$ be probability measures on $\mathcal{X}$. 
    We say $\nu$ is $\epsilon$-independent of $\{\mu_1,\cdots, \mu_n\}$ with respect to $\mathcal{G}$ if there exists a $g \in \mathcal{G}$ such that $\sqrt{\sum_{i=1}^n(\mathbb{E}_{\mu_i}[g])^2} \leq \epsilon$ but $|\mathbb{E}_\nu[g]|>\epsilon$.
\end{definition}

\begin{definition}[Distributional Eluder (DE) dimension]
    Let $\mathcal{G}$ be a function class on space $\mathcal{X}$, and let $\Pi$ be a family of probability measures on $\mathcal{X}$. 
    The distributional eluder dimension $\operatorname{dim}_{\mathrm{DE}}(\mathcal{G}, \Pi, \epsilon)$ is defined as the length of the longest sequence $\{\rho_1, \cdots, \rho_n\} \subset \Pi$ such that there exists $\epsilon^{\prime} \geq \epsilon$ with $\rho_i$ being $\epsilon^{\prime}$-independent of $\{\rho_1, \cdots, \rho_{i-1}\}$ for each $i \in[n]$.
\end{definition}

The Bellman eluder dimension is based upon the notion of distributional eluder dimsntion. 
For a model-free hypothesis class $\mathcal{H}$, we the Bellman operator $\mathcal{T}_h$ defined in Section \ref{sec: pre} becomes,
\begin{align}\label{eq: bellman operator}
    (\mathcal{T}_hf_{h+1})(x,a) = R_h(x,a)+\mathbb{E}_{x'\sim \mathbb{P}_h(\cdot|x,a)}[V_{h+1,f}(x')],
\end{align}
for any $f\in\mathcal{H}$. 
Then we define the $Q$-type/$V$-type Bellman eluder dimension as the following.

\begin{definition}[$Q$-type Bellman eluder (BE) dimension \citep{jin2021bellman, zhong2022posterior}]
    We define $(I-\mathcal{T}_h) \mathcal{H}=\{(x, a) \mapsto(f_h-\mathcal{T}_h f_{h+1})(x, a): f \in \mathcal{H}\}$ as the set of Bellman residuals induced by $\mathcal{H}$ at step $h$, and let $\Pi=\{\Pi_h\}_{h=1}^H$ be a collection of $H$ families of probability measure over $\mathcal{S} \times \mathcal{A}$. 
    The $Q$-type $\epsilon$-Bellman eluder dimension of $\mathcal{H}$ with respect to $\Pi$ is defined as
    \begin{align*}
        \operatorname{dim}_{\mathrm{BE}}(\mathcal{H}, \Pi, \epsilon)=\max _{h \in[H]}\left\{\operatorname{dim}_{\mathrm{DE}}\left(\left(I-\mathcal{T}_h\right) \mathcal{H}, \Pi_h, \epsilon\right)\right\}.
    \end{align*}
\end{definition}

\begin{definition}[$V$-type Bellman eluder (BE) dimension \citep{jin2021bellman, zhong2022posterior}]
    We define $(I-\mathcal{T}_h) V_{\mathcal{H}}=\{x \mapsto(f_h-\mathcal{T}_h f_{h+1})(x, \pi_{h,f}(x)) : f \in \mathcal{H}\}$ as the set of $V$-type Bellman residuals induced by $\mathcal{H}$ at step $h$, and let $\Pi=\{\Pi_h\}_{h=1}^H$ be a collection of $H$ families of probability measure over $\mathcal{S}$. 
    The $V$-type $\epsilon$-Bellman eluder dimension of $\mathcal{H}$ with respect to $\Pi$ is defined as
    \begin{align*}
        \operatorname{dim}_{\mathrm{VBE}}(\mathcal{H}, \Pi, \epsilon)=\max_{h \in[H]}\left\{\operatorname{dim}_{\mathrm{DE}}\left(\left(I-\mathcal{T}_h\right) V_{\mathcal{H}}, \Pi_h, \epsilon\right)\right\}.
    \end{align*}
\end{definition}

For MDPs with low Bellman eluder dimension, we choose the function $l$ in Assumption \ref{ass: l function} as
\begin{align}\label{eq: low BE l}
    l_{f'}\big((f_h,f_{h+1});\mathcal{D}_h\big)=Q_{h,f}(x_h,a_h) - r_h - V_{h+1,f}(x_{h+1}).
\end{align}
and we choose the operator $\mathcal{P}_h = \mathcal{T}_h$ defined in \eqref{eq: bellman operator}.
One can check that such a choice satisfies Assumption \ref{ass: l function}.
By further choosing the exploration policy as $\pi_{\mathrm{exp}}(f) = \pi_f$ for $Q$-type problems and $\pi_{\exp}(f) = \pi_f\circ_h \mathrm{Unif}(\mathcal{A})$ for $V$-type problems\footnote{The policy $\pi_f\circ_h \mathrm{Unif}(\mathcal{A})$ means that when executing the exploration policy to collect data $\cD_h$ at timestep $h$, the agent first executes policy $\pi_f$ for the first $h-1$ steps and then takes an action uniformly sampled from $\cA$ at timestep $h$.}, we can bound the GEC for MDPs with low BE dimension by the following lemma.

\begin{lemma}[GEC for low Bellman eluder dimension, Lemma 3.16 in \cite{zhong2022posterior}]\label{lem: gec low bellman eluder dimension}
    Let the discrepancy $\ell$ function be chosen as \eqref{eq: example model free ell} with $l$ defined in \eqref{eq: low BE l}.
    Define $\Pi_{\mathcal{H}}$ as the distributions induced by following some $f \in \mathcal{H}$ greedily. 
    For $Q$-type problems, by choosing $\pi_{\mathrm{exp}}(f) = \pi_f$, we have that 
    \begin{align*}
        d_{\mathrm{GEC}}(\epsilon) \leq 2 \operatorname{dim}_{\mathrm{BE}}(\mathcal{H}, \Pi_{\mathcal{H}}, \epsilon) H \cdot \log (K),
    \end{align*}
    For $V$-type problems, by choosing $\pi_{\exp}(f) = \pi_f\circ_h \mathrm{Unif}(\mathcal{A})$, we have that 
    \begin{align*}
        d_{\mathrm{GEC}}(\epsilon) \leq 2 \operatorname{dim}_{\mathrm{VBE}}(\mathcal{H}, \Pi_{\mathcal{H}}, \epsilon) |\mathcal{A}| H \cdot \log (K).
    \end{align*}
\end{lemma}

\begin{proof}[Proof of Lemma \ref{lem: gec low bellman eluder dimension}]
    See Lemma 3.16 in \cite{zhong2022posterior} for a detailed proof.
\end{proof}

By combining Lemma \ref{lem: gec low bellman eluder dimension} and Corollary \ref{cor: regret model free mdp}, we can obtain that
for $Q$-type low Bellman eluder dimension problem, it holds that with probability at least $1-\delta$,
\begin{align}
    \mathrm{Regret}(T) \lesssim B_l^2\cdot \sqrt{\operatorname{dim}_{\mathrm{BE}}(\mathcal{H}, \Pi_{\mathcal{H}}, 1/\sqrt{HK}) \cdot \log(HK|\cH|/\delta)\cdot H^2K},
\end{align}
and for $V$-type Bellman eluder dimension problem, it holds that with probability at least $1-\delta$,
\begin{align}
    \mathrm{Regret}(T) \lesssim B_l^2\cdot \sqrt{\operatorname{dim}_{\mathrm{VBE}}(\mathcal{H}, \Pi_{\mathcal{H}}, 1/\sqrt{HK}) \cdot |\cA|\cdot \log(HK|\cH|/\delta)\cdot H^2K}.%(\log(HK|\mathcal{H}|/\delta) + 2 \operatorname{dim}_{\mathrm{VBE}}(\mathcal{H}, \Pi_{\mathcal{H}}, 1/\sqrt{K}) |\mathcal{A}| \cdot \log (K))H\sqrt{K}.
\end{align}

\paragraph{MDPs of bilinear class.}

In this part, we consider MDPs of bilinear class \citep{du2021bilinear}.

\begin{definition}[Bilinear class \citep{du2021bilinear,zhong2022posterior}]\label{def: bilinear}
    Given an MDP, a model-free hypothesis class $\mathcal{H}$, and a function $l_f: \mathcal{H} \times\mathcal{H}\times(\mathcal{S} \times \mathcal{A} \times \mathbb{R} \times \mathcal{S})  \mapsto \mathbb{R}$, we say the corresponding RL problem is in a bilinear class if there exist functions $W_h: \mathcal{H} \mapsto \mathcal{V}$ and $X_h: \mathcal{H} \mapsto \mathcal{V}$ for some Hilbert space $\mathcal{V}$, such that for all $f,g \in \mathcal{H}$ and $h \in[H]$, we have that
    \begin{align*}
        \left|\mathbb{E}_{\pi_f}[Q_{h,f}(x_h, a_h)-R_h(x_h, a_h)-V_{h+1,f}(x_{h+1})]\right| & \leq\left|\langle W_h(f)-W_h(f^{\ast}), X_h(f)\rangle_{\mathcal{V}}\right|, \\
        \left|\mathbb{E}_{x_h \sim \pi_f, a_h \sim \tilde{\pi}}[l_f(g; \xi_h)]\right| & =\left|\langle W_h(g)-W_h(f^{\ast}), X_h(f)\rangle_{\mathcal{V}}\right|,
    \end{align*}
    where $\tilde{\pi}$ is either $\pi_f$ for $Q$-type problems or $\pi_g$ for $V$-type problems.
    Meanwhile, we make the assumption that $\sup _{f \in \mathcal{H}, h \in[H]}\|W_h(f)\|_2 \leq 1$ and $\sup _{f \in \mathcal{H}, h \in[H]}\|X_h(f)\|_2 \leq 1$.
\end{definition}

For MDPs of bilinear class, we choose the function $l$ as the function introduced in the definition of bilinear class.
By choosing the exploration policy as $\pi_{\mathrm{exp}}(f) = \pi_f$ for $Q$-type problems and $\pi_{\exp}(f) = \pi_f\circ_h \mathrm{Unif}(\mathcal{A})$ for $V$-type problems, we can bound the generalized eluder coefficient for MDPs of bilinear class using the following lemma.
To simplify the notation, we define $\mathcal{X}_h=\{X_h(f): f \in \mathcal{H}\}\subseteq\mathcal{V}$ and $\mathcal{X}=\{\mathcal{X}_h: h \in[H]\}$.

\begin{lemma}[GEC for bilinear class, Lemma 3.22 in \cite{zhong2022posterior}]\label{lem: gec bilinear class}
    Let the discrepancy $\ell$ function be chosen as \eqref{eq: example model free ell} with $l$ defined in Definition \ref{def: bilinear}.
    Define the maximum information gain $\gamma_K(\epsilon, \mathcal{X})$ as 
    \begin{align*}
        \gamma_K(\epsilon, \mathcal{X})=\sum_{h=1}^H\max_{x_1, \cdots, x_K \in \mathcal{X}_h} \log \operatorname{det}\Bigg(\mathcal{I}(\cdot)+\frac{1}{\epsilon} \sum_{s=1}^K x_s \langle x_s,\cdot\rangle_{\mathcal{V}}\Bigg)
    \end{align*}
    with $\mathcal{I}$ being the identity mapping.
    Then for $Q$-type problems, choosing $\pi_{\mathrm{exp}}(f) = \pi_f$, we have that 
    \begin{align*}
        d_{\mathrm{GEC}}(\epsilon) \leq 2 \gamma_K(\epsilon,\mathcal{X}).
    \end{align*}
    For $V$-type problems, by choosing $\pi_{\exp}(f) = \pi_f\circ_h \mathrm{Unif}(\mathcal{A})$, we have that 
    \begin{align*}
        d_{\mathrm{GEC}}(\epsilon) \leq 2|\mathcal{A}|\gamma_K(\epsilon,\mathcal{X}).
    \end{align*}
\end{lemma}

\begin{proof}[Proof of Lemma \ref{lem: gec low bellman eluder dimension}]
    See Lemma 3.22 in \cite{zhong2022posterior} for a detailed proof.
\end{proof}

By combining Lemma \ref{lem: gec bilinear class} and Corollary \ref{cor: regret model free mdp}, we know that
For $Q$-type bilinear class problem, it holds that with probability at least $1-\delta$,
\begin{align}
    \mathrm{Regret}(T) \lesssim \sqrt{\gamma_K(1/\sqrt{HK},\mathcal{X}) \cdot \log(HK|\mathcal{H}|/\delta)\cdot HK},
\end{align}
and for $V$-type bilinear class problem, it holds that with probability at least $1-\delta$,
\begin{align}
    \mathrm{Regret}(T) \lesssim \sqrt{\gamma_K(1/\sqrt{HK},\mathcal{X}) \cdot|\cA|\cdot  \log(HK|\mathcal{H}|/\delta)\cdot HK}.
\end{align}

\subsection{Examples of Model-based Online RL in MDPs}\label{subsec: example model based mdp}

\textbf{MDPs with low witness rank.} 
We consider the example of MDPs with low witness rank \citep{sun2019model, agarwal2022model}. 
To introduce, we define the function class $\mathcal{V} = \{\upsilon:\mathcal{S}\times\mathcal{A}\times\mathcal{S}\mapsto[0,1]\}$.

\begin{definition}[Q-type/V-type witness rank \citep{sun2019model, agarwal2022model}] An MDP is called of witness rank $d$ if for any two models $f,f'\in\mathcal{H}$, there exists mappings $X_h:\mathcal{H}\mapsto\mathbb{R}^d$ and $W_h:\mathcal{H}\mapsto\mathbb{R}^d$ for each timestep $h$ such that,
\begin{align*}
    \max_{\upsilon\in\mathcal{V}}\mathbb{E}_{x_h\sim \pi_f,a_h\sim \tilde{\pi}}\left[\left(\mathbb{E}_{x'\sim \mathbb{P}_{h,f'}(\cdot|x_h,a_h)} - \mathbb{E}_{x'\sim \mathbb{P}_{h,f^{\ast}}(\cdot|x_h,a_h)}\right)[\upsilon(x_h,a_h,x')]\right] &\geq \langle W_h(f'),X_h(f)\rangle,\\
    \kappa_{\mathrm{wit}}\cdot\mathbb{E}_{x_h\sim \pi_f,a_h\sim \tilde{\pi}}\left[\left(\mathbb{E}_{x'\sim \mathbb{P}_{h,f'}(\cdot|x_h,a_h)} - \mathbb{E}_{x'\sim \mathbb{P}_{h,f^{\ast}}(\cdot|x_h,a_h)}\right)[V_{h+1,f'}(x')]\right] &\leq \langle W_h(f'),X_h(f)\rangle,
\end{align*}
where $\tilde{\pi}$ is either $\pi_f$ for $Q$-type problems or $\pi_{f'}$ for $V$-type problems and $\kappa_{\mathrm{wit}}\in(0,1]$ is a constant. 
Also, we let $\sup_{f\in\mathcal{H},h\in[H]}\|W_h(f)\|\leq 1$ and  $\sup_{f\in\mathcal{H},h\in[H]}\|X_h(f)\|\leq 1$.
\end{definition}

By choosing the exploration policy as $\pi_{\mathrm{exp}}(f) = \pi_f$ for $Q$-type problems and $\pi_{\exp}(f) = \pi_f\circ_h \mathrm{Unif}(\mathcal{A})$ for $V$-type problems, we can bound the generalized eluder coefficient by the following lemma.

\begin{lemma}[GEC for low witness rank, Lemma 3.22 in \cite{zhong2022posterior}]\label{lem: gec low witness rank}
    Let the discrepancy function $\ell$ be chosen as \eqref{eq: example model based ell}. For $Q$-type problems, by choosing $\pi_{\mathrm{exp}}(f) = \pi_f$, we have that 
    \begin{align*}
        d_{\mathrm{GEC}}(\epsilon) \leq 4dH\cdot\log(1+K/(\epsilon \kappa^2_{\mathrm{wit}})) / \kappa^2_{\mathrm{wit}}.
    \end{align*}
    For $V$-type problems, by choosing $\pi_{\exp}(f) = \pi_f\circ_h \mathrm{Unif}(\mathcal{A})$, we have that
    \begin{align*}
         d_{\mathrm{GEC}}(\epsilon) \leq 4d|\mathcal{A}|H\cdot\log(1+K/(\epsilon \kappa^2_{\mathrm{wit}})) / \kappa^2_{\mathrm{wit}}.
    \end{align*}
\end{lemma}

\begin{proof}[Proof of Lemma \ref{lem: gec low witness rank}]
    See Lemma 3.22 in \cite{zhong2022posterior} for a detailed proof.
\end{proof}

By combining Lemma \ref{lem: gec low witness rank} and Corollary \ref{cor: regret model based mdp}, we know that
For $Q$-type low witness rank problem, it holds that with probability at least $1-\delta$,
\begin{align}
    \mathrm{Regret}(K) \lesssim \sqrt{4dH^2K\cdot\log(H|\mathcal{H}|/\delta)\cdot\log(1+H^{1/2}K^{3/2}/ \kappa^2_{\mathrm{wit}}) / \kappa^2_{\mathrm{wit}}},
\end{align}
and for $V$-type low witness rank problem, it holds that with probability at least $1-\delta$,
\begin{align}
    \mathrm{Regret}(K) \lesssim \sqrt{4d|\mathcal{A}|H^2K\cdot\log(H|\mathcal{H}|/\delta)\cdot\log(1+H^{1/2}K^{3/2}/ \kappa^2_{\mathrm{wit}}) / \kappa^2_{\mathrm{wit}}}.
\end{align}

\subsection{Proof of Proposition \ref{prop: supervised learning model free rl}}\label{subsec: proof prop supervised learning model free rl}

\begin{proof}[Proof of Proposition \ref{prop: supervised learning model free rl}]
    To prove Proposition \ref{prop: supervised learning model free rl}, we define the random variables $X_{h,f}^k$ as 
    \begin{align}\label{eq: proof supervised learning model free rl 0}
        X_{h,f}^k = l_{f^k}((f_{h},f_{h+1});\mathcal{D}_h^k)^2 - l_{f^k}((\mathcal{P}_hf_{h+1},f_{h+1});\mathcal{D}_h^k)^2,
    \end{align}
    for any $f\in\mathcal{H}$, where the operator $\mathcal{P}_h$ is introduced in Assumption \ref{ass: l function}.
    We first show that $X_{h,f}^k$ is an unbiased estimator of the discrepancy function $\ell_{f^k}(f)$.
    Consider that
    \begin{align}
        l_{f^k}((f_{h},f_{h+1});\mathcal{D}_h^k)^2& = \left(l_{f^k}((f_{h},f_{h+1});\mathcal{D}_h^k) - l_{f^k}((\mathcal{P}_hf_{h+1},f_{h+1});\mathcal{D}_h^k) + l_{f^k}((\mathcal{P}_hf_{h+1},f_{h+1});\mathcal{D}_h^k)\right)^2\notag \\
        &= \left(\mathbb{E}_{x_{h+1}^k\sim\mathbb{P}_h(\cdot|x_h^k,a_h^k)}[l_{f^k}((f_{h},f_{h+1});\mathcal{D}_h^k)] + l_{f^k}((\mathcal{P}_hf_{h+1},f_{h+1});\mathcal{D}_h^k) \right)^2 \notag\\
         &= \left(\mathbb{E}_{x_{h+1}^k\sim\mathbb{P}_h(\cdot|x_h^k,a_h^k)}[l_{f^k}((f_{h},f_{h+1});\mathcal{D}_h^k)]\right)^2 + l_{f^k}((\mathcal{P}_hf_{h+1},f_{h+1});\mathcal{D}_h^k)^2 \notag\\
        &\qquad + 2 \mathbb{E}_{x_{h+1}^k\sim\mathbb{P}_h(\cdot|x_h^k,a_h^k)}[l_{f^k}((f_{h},f_{h+1});\mathcal{D}_h^k)]\cdot l_{f^k}((\mathcal{P}_hf_{h+1},f_{h+1});\mathcal{D}_h^k),\label{eq: proof supervised learning model free rl 1}
    \end{align}
    where in the second equality we apply the generalized Bellman completeness condition in Assumption \ref{ass: l function}.
    By the generalized Bellman completeness condition again, we also have that in \eqref{eq: proof supervised learning model free rl 1},
    \begin{align}
        &\mathbb{E}_{x^k_{h+1}\sim \mathbb{P}_h(\cdot|x_h^k,a_h^k)}\left[\mathbb{E}_{x_{h+1}^k\sim\mathbb{P}_h(\cdot|x_h^k,a_h^k)}[l_{f^k}((f_{h},f_{h+1});\mathcal{D}_h^k)]\cdot l_{f^k}((\mathcal{P}_hf_{h+1},f_{h+1});\mathcal{D}_h^k) \right]\notag\\
        &\qquad =\mathbb{E}_{x_{h+1}^k\sim\mathbb{P}_h(\cdot|x_h^k,a_h^k)}[l_{f^k}((f_{h},f_{h+1});\mathcal{D}_h^k)]\cdot \mathbb{E}_{x_{h+1}^k\sim\mathbb{P}_h(\cdot|x_h^k,a_h^k)}[l_{f^k}((\mathcal{P}_hf_{h+1},f_{h+1});\mathcal{D}_h^k) ]\notag\\
        &\qquad =\mathbb{E}_{x_{h+1}^k\sim\mathbb{P}_h(\cdot|x_h^k,a_h^k)}[l_{f^k}((f_{h},f_{h+1});\mathcal{D}_h^k)]\notag\\
        &\qquad\qquad \cdot \mathbb{E}_{x_{h+1}^k\sim\mathbb{P}_h(\cdot|x_h^k,a_h^k)}\left[l_{f^k}((f_h,f_{h+1});\mathcal{D}_h^k) - \mathbb{E}_{x_{h+1}^k\sim\mathbb{P}_h(\cdot|x_h^k,a_h^k)}[l_{f^k}((f_{h},f_{h+1});\mathcal{D}_h^k)]\right] \notag\\
        &\qquad = 0.\label{eq: proof supervised learning model free rl 2}
    \end{align}
    Thus by combining \eqref{eq: proof supervised learning model free rl 1} and \eqref{eq: proof supervised learning model free rl 2}, we can derive that 
    \begin{align}
        &\mathbb{E}_{x^k_{h+1}\sim \mathbb{P}_h(\cdot|x_h^k,a_h^k)}[X_{h,f}^k] = \left(\mathbb{E}_{x_{h+1}^k\sim\mathbb{P}_h(\cdot|x_h^k,a_h^k)}[l_{f^k}((f_{h},f_{h+1});\mathcal{D}_h^k)]\right)^2 = \ell_{f^k}(f;\mathcal{D}_h^k),
    \end{align}
    Now for each timestep $h$, we define a filtration $\{\mathcal{F}_{h,k}\}_{k=1}^K$, with
    \begin{align}\label{eq: filtration model free mdp}
        \mathcal{F}_{h,k} = \sigma\left(\bigcup_{s=1}^{k}\bigcup_{h=1}^H\mathcal{D}_h^s\right),
    \end{align}
    where $\mathcal{D}_h^s = \{x_h^s,a_h^s,r_h^s,x_{h+1}^s\}$. 
    From previous arguments, we can derive that 
    \begin{align}\label{eq: proof supervised learning model free rl 2+}
        \mathbb{E}[X_{h,f}^k|\mathcal{F}_{h,k-1}] = \mathbb{E}\left[\mathbb{E}_{x^k_{h+1}\sim \mathbb{P}_h(\cdot|x_h^k,a_h^k)}[X_{h,f^k}] \middle| \mathcal{F}_{h,k-1}\right] = \mathbb{E}_{\xi_h\sim \pi_{\exp}(f^k)}[ \ell_{f^k}(f;\xi_h)].
    \end{align}
    and that
    \begin{align}\label{eq: proof supervised learning model free rl 2++}
        \mathbb{V}[X_{h,f}^k|\mathcal{F}_{h,k-1}] \leq \mathbb{E}[(X_{h,f}^k)^2|\mathcal{F}_{h,k-1}] \leq 4B_l^2\mathbb{E}[X_{h,f}^k|\mathcal{F}_{h,k-1}] = 4B_l^2\mathbb{E}_{\xi_h\sim \pi_{\exp}(f^k)}[ \ell_{f^k}(f;\xi_h)],
    \end{align}
    where $B_l$ is the upper bound of $l$ defined in Assumption \ref{ass: l function}.
    By applying Lemma \ref{lem: freedman}, \eqref{eq: proof supervised learning model free rl 2+}, and \eqref{eq: proof supervised learning model free rl 2++}, we can obtain that with probability at least $1-\delta$, for any $(h,k)\in[H]\times[K]$, $(f_h,f_{h+1})\in\mathcal{H}_h\times\mathcal{H}_{h+1}$\footnote{Here $l_{f^s}((f_h,f_{h+1});\cD_h^s)$ and $\ell_{f^s}(f;\xi_h)$ depend on $f$ only through $(f_h, f_{h+1})$.},
    \begin{align}\label{eq: proof supervised learning model free rl 3}
        \left|\sum_{s=1}^{k-1}\mathbb{E}_{\xi_h\sim \pi_{\exp}(f^s)}[ \ell_{f^s}(f;\xi_h)] - \sum_{s=1}^{k-1}X_{h,f}^s\right| \lesssim \frac{1}{2}\sum_{s=1}^{k-1}\mathbb{E}_{\xi_h\sim \pi_{\exp}(f^s)}[ \ell_{f^s}(f;\xi_h)] +  8B_l^2\log(HK|\mathcal{H}_h||\mathcal{H}_{h+1}|/\delta).
    \end{align}
    Rearranging terms in \eqref{eq: proof supervised learning model free rl 3}, we can further obtain that 
    \begin{align}\label{eq: proof supervised learning model free rl 4}
        - \sum_{s=1}^{k-1}X_{h,f}^s \lesssim -\frac{1}{2}\sum_{s=1}^{k-1}\mathbb{E}_{\xi_h\sim \pi_{\exp}(f^s)}[ \ell_{f^s}(f;\xi_h)] +  8B_l^2\log(HK|\mathcal{H}_h||\mathcal{H}_{h+1}|/\delta).
    \end{align}
    Meanwhile, by the definition of $X_{h,f}^k$ in \eqref{eq: proof supervised learning model free rl 0} and the loss function $L$ in \eqref{eq: implicit optimism mdp free based L}, we have that 
    \begin{align}\label{eq: proof supervised learning model free rl 5}
        \sum_{s=1}^{k-1}X_{h,f}^s &= \sum_{s=1}^{k-1}l_{f^s}((f_{h},f_{h+1}),\mathcal{D}_h^s)^2 -\sum_{s=1}^{k-1}l_{f^k}((\mathcal{P}_hf_{h+1},f_{h+1}),\mathcal{D}_h^s)^2\notag\\
        &\leq \sum_{s=1}^{k-1} l_{f^s}((f_{h},f_{h+1}),\mathcal{D}_h^s)^2 - \inf_{f^{\prime}_h\in\mathcal{F}} \sum_{s=1}^{k-1} l_{f^s} ((f^{\prime}_h,f_{h+1}),\mathcal{D}_h^s)^2\notag\\
        &=L_h^{k-1}(f).
    \end{align}
    Thus by \eqref{eq: proof supervised learning model free rl 4} and \eqref{eq: proof supervised learning model free rl 5}, we can derive that with probability at least $1-\delta$, for any $f\in\mathcal{H}$, $k\in[K]$,
    \begin{align}\label{eq: proof supervised learning model free rl 6}
        -\sum_{h=1}^HL_h^{k-1}(f) \lesssim -\frac{1}{2}\sum_{h=1}^H\sum_{s=1}^{k-1}\mathbb{E}_{\xi_h\sim \pi_{\exp}(f^s)}[ \ell_{f^s}(f;\xi_h)] +  8HB_l^2\log\left(HK/\delta\right) + 16B_l^2\log(|\mathcal{H}|).
    \end{align}
    Finally, we deal with the term $L_h^{k-1}(f^{\ast})$. 
    To this end, we invoke the following lemma.

    \begin{lemma} \label{lem: f star model free mdp}
        With probability at least $1-\delta$, it holds that for each $k\in[K]$,
        \begin{align*}
            \sum_{h=1}^HL^{k-1}_h(f^{\ast}) \lesssim 8HB_l^2\log\left(HK|\mathcal{H}|/\delta\right) + 16B_l^2\log(|\mathcal{H}|).
        \end{align*}
    \end{lemma}
    
    \begin{proof}[Proof of Lemma \ref{lem: f star model free mdp}]
        To prove Lemma \ref{lem: f star model free mdp}, we define the random variables $W_{h,f}^k$ as 
        \begin{align*}
            W_{h,f}^k = l_{f^k}((f_h,f_{h+1}^{\ast});\mathcal{D}_h^k)^2 - l_{f^k}((f_h^{\ast},f_{h+1}^{\ast});\mathcal{D}_h^k)^2.
        \end{align*}
        Using the same argument as 
        \eqref{eq: proof supervised learning model free rl 1} and \eqref{eq: proof supervised learning model free rl 2}, together with the condition $\mathcal{P}_{h}f^{\ast}_{h+1} = f_h^{\ast}$ in Assumption~\ref{ass: l function}, we can show that 
        \begin{align}
            \mathbb{E}_{x_{h+1}^k\sim \mathbb{P}_h(\cdot|x_h^k,a_h^k)}[W_{h,f}^k] = \left( \mathbb{E}_{x_{h+1}^k\sim \mathbb{P}_h(\cdot|x_h^k,a_h^k)}[l_{f^k}((f_h,f_{h+1}^{\ast});\mathcal{D}_h)]\right)^2.
        \end{align}
        Under the filtration $\{\mathcal{F}_{h,k}\}_{k=1}^K$ defined in the proof of Proposition \ref{prop: supervised learning model free rl}, i.e, \eqref{eq: filtration model free mdp}, one can derive that 
        \begin{align}
            \mathbb{E}[W_{h,f}^k|\mathcal{F}_{h,k-1}] &= \mathbb{E}\left[\mathbb{E}_{x^k_{h+1}\sim \mathbb{P}_h(\cdot|x_h^k,a_h^k)}[W_{h,f^k}] \middle| \mathcal{F}_{h,k-1}\right] \notag \\
            &= \mathbb{E}_{\mathcal{D}_h\sim \pi_{\exp}(f^k)}\left[\left( \mathbb{E}_{x_{h+1}\sim \mathbb{P}_h(\cdot|x_h,a_h)}[l_{f^k}((f_h,f_{h+1}^{\ast});\mathcal{D}_h)]\right)^2\right],\label{eq: proof supervised learning model free rl star 2+}
        \end{align}
        and that
        \begin{align}
        \mathbb{V}[W_{h,f}^k|\mathcal{F}_{h,k-1}] &  \leq 4B_l^2\mathbb{E}[X_{h,f}^k|\mathcal{F}_{h,k-1}] \notag \\
        &= 4B_l^2\mathbb{E}_{\mathcal{D}_h\sim \pi_{\exp}(f^k)}\left[\left( \mathbb{E}_{x_{h+1}\sim \mathbb{P}_h(\cdot|x_h,a_h)}[l_{f^k}((f_h,f_{h+1}^{\ast});\mathcal{D}_h)]\right)^2\right].\label{eq: proof supervised learning model free rl star 2++}
    \end{align}
    By applying Lemma \ref{lem: freedman}, \eqref{eq: proof supervised learning model free rl star 2+}, and \eqref{eq: proof supervised learning model free rl star 2++}, we obtain that with probability at least $1-\delta$, for any $(h,k)\in[H]\times[K]$ and $(f_h, f_{h+1})\in\mathcal{H}_h\times\mathcal{H}_{h+1}$, 
    \begin{align*}
        &\left|\sum_{s=1}^{k-1}W_{h,f}^s - \sum_{s=1}^{k-1}\mathbb{E}_{\mathcal{D}_h\sim \pi_{\exp}(f^k)}\left[\left( \mathbb{E}_{x_{h+1}\sim \mathbb{P}_h(\cdot|x_h,a_h)}[l_{f^s}((f_h,f_{h+1}^{\ast});\mathcal{D}_h)]\right)^2\right]\right| \lesssim 4B_l^2\log(HK|\mathcal{H}_h||\mathcal{H}_{h+1}|/\delta)\notag \\
        & \qquad +\sqrt{\log(HK|\mathcal{H}_h||\mathcal{H}_{h+1}|/\delta)\cdot \sum_{s=1}^{k-1}\mathbb{E}_{\mathcal{D}_h\sim \pi_{\exp}(f^s)}\left[\left( \mathbb{E}_{x_{h+1}\sim \mathbb{P}_h(\cdot|x_h,a_h)}[l_{f^k}((f_h,f_{h+1}^{\ast});\mathcal{D}_h)]\right)^2\right]}.
    \end{align*}
    Rearranging terms, we have that with probability at least $1-\delta$, for any $f\in\mathcal{H}$, $(h,k)\in[H]\times[K]$,
    \begin{align*}
        -\sum_{s=1}^{k-1}W_{h,f}^s &\lesssim 4B_l^2\log(HK|\mathcal{H}_h||\mathcal{H}_{h+1}|/\delta) - \sum_{s=1}^{k-1}\mathbb{E}_{\mathcal{D}_h\sim \pi_{\exp}(f^s)}\left[\left( \mathbb{E}_{x_{h+1}\sim \mathbb{P}_h(\cdot|x_h,a_h)}[l_{f^s}((f_h,f_{h+1}^{\ast});\mathcal{D}_h)]\right)^2\right] \\
        &\qquad + \sqrt{\log(HK|\mathcal{H}_h||\mathcal{H}_{h+1}|/\delta)\cdot \sum_{s=1}^{k-1}\mathbb{E}_{\mathcal{D}_h\sim \pi_{\exp}(f^s)}\left[\left( \mathbb{E}_{x_{h+1}\sim \mathbb{P}_h(\cdot|x_h,a_h)}[l_{f^s}((f_h,f_{h+1}^{\ast});\mathcal{D}_h)]\right)^2\right]} \\
        & \lesssim 8B_l^2\log(HK|\mathcal{H}_h||\mathcal{H}_{h+1}|/\delta),
    \end{align*}
    where in the second inequality we use the inequality $-x^2 + ax \leq a^2/4$.
    Thus, with probability at least $1-\delta$, for any $k\in[K]$, it holds that
    \begin{align*}
        \sum_{h=1}^HL_h^{k-1}(f^{\ast}) 
         &= \sum_{h=1}^H\left(\sum_{s=1}^{k-1}  l_{f^k}((f_h^{\ast},f_{h+1}^{\ast});\mathcal{D}_h^s)^2 - \inf_{f_h\in\mathcal{H}_h} \sum_{s=1}^{k-1}l_{f^k}((f_h,f_{h+1}^{\ast});\mathcal{D}_h^s)^2 \right) \\
         & =\sum_{h=1}^H\sup_{f_h\in\mathcal{H}_h}\sum_{s=1}^{k-1}  -W_{h,f}^s  \lesssim 8HB_l^2\log(HK/\delta) + 16B_l^2\log(|\mathcal{H}|).
    \end{align*}
    This finishes the proof of Lemma \ref{lem: f star model free mdp}.
    \end{proof}

    Finally, combining \eqref{eq: proof supervised learning model free rl 6} and Lemma \ref{lem: f star model free mdp}, with probability at least $1-\delta$, for any $f\in\mathcal{H}$, $k\in[K]$,
    \begin{align*}
        \sum_{h=1}^HL_h^{k-1}(f^{\ast}) - L_h^{k-1}(f) \lesssim -\frac{1}{2}\sum_{h=1}^H\sum_{s=1}^{k-1}\mathbb{E}_{\xi_h\sim \pi_{\exp}(f^s)}[ \ell_{f^s}(f;\xi_h)] +  16HB_l^2\log(HK/\delta) + 32B_l^2\log(|\mathcal{H}|).
    \end{align*}
    This finishes the proof of Proposition \ref{prop: supervised learning model free rl}.
\end{proof}

\subsection{Proof of Proposition \ref{prop: supervised learning model based rl}}\label{subsec: proof prop supervised learning model based rl}

\begin{proof}[Proof of Proposition \ref{prop: supervised learning model based rl}]
    For notational simplicity, given $f\in\mathcal{H}$, we denote the random variables $X_{h,f}^k$ as 
    \begin{align}\label{eq: X model-based}
        X_{h,f}^{k} = \log\left(\frac{\mathbb{P}_{h,f^{\ast}}(x_{h+1}^k|x_h^k,a_h^k)}{\mathbb{P}_{h,f}(x_{h+1}^k|x_h^k,a_h^k)}\right).
    \end{align}
    Then by the definition of $L_h^k$ in \eqref{eq: implicit optimism mdp model based L}, we have that,
    \begin{align}
        \sum_{h=1}^HL_h^{k-1}(f^{\ast}) - L_h^{k-1}(f) = -\sum_{h=1}^H\sum_{s=1}^{k-1}X_{h,f}^s.
    \end{align}
    Now we define a filtration $\{\mathcal{F}_{h,k}\}_{k=1}^K$ for each step $h\in[H]$ with 
    \begin{align}
        \mathcal{F}_{h,k} = \sigma\left(\bigcup_{s=1}^k\bigcup_{h=1}^H\mathcal{D}_h^s\right).
    \end{align}
    Then by \eqref{eq: X model-based} we know that $X_{h,f}^k\in\mathcal{F}_{h,k}$ for any $(h,k)\in[H]\times[K]$.
    Therefore, by applying Lemma \ref{lem: concentration}, we have that with probability at least $1-\delta$, for any $(h,k)\in[H]\times[K]$ and $f_h\in\mathcal{H}_h$,
    \begin{align}\label{eq: proof model based supervised concentration}
         -\frac{1}{2}\sum_{s=1}^{k-1}X_{h,f}^s \leq \sum_{s=1}^{k-1}\log\mathbb{E}\left[\exp\left\{-\frac{1}{2}X_{h,f}^s \right\}\middle| \mathcal{F}_{s-1}\right] + \log(H|\mathcal{H}_h|/\delta).
    \end{align}
    Meanwhile, we can calculate that in \eqref{eq: proof model based supervised concentration}, the conditional expectation equals to
    \begin{align}
        \mathbb{E}\left[\exp\left\{-\frac{1}{2}X_{h,f}^s\right\} \middle| \mathcal{F}_{s-1}\right]& = \mathbb{E}\left[\sqrt{\frac{\mathbb{P}_{h,f}(x_{h+1}^s|x_h^s,a_h^s)}{\mathbb{P}_{h,f^{\ast}}(x_{h+1}^s|x_h^s,a_h^s)}} \middle| \mathcal{F}_{s-1}\right] \notag\\
        &= \mathbb{E}_{(x_h^s,a_h^s)\sim \pi_{\mathrm{exp}}(f^s),x_{h+1}^s\sim \mathbb{P}_{h,f^{\ast}}(\cdot|x_h^s,a_h^s)}\left[\sqrt{\frac{\mathbb{P}_{h,f}(x_{h+1}^s|x_h^s,a_h^s)}{\mathbb{P}_{h,f^{\ast}}(x_{h+1}^s|x_h^s,a_h^s)}}\right] \notag\\
        & = \mathbb{E}_{(x_h^s,a_h^s)\sim \pi_{\mathrm{exp}}(f^s)}\left[\int_{\mathcal{S}}\sqrt{\mathbb{P}_{h,f}(x_{h+1}^s|x_h^s,a_h^s)\cdot\mathbb{P}_{h,f^{\ast}}(x_{h+1}^s|x_h^s,a_h^s)}\mathrm{d}x_{h+1}^s\right]\notag\\
        & = 1 - \frac{1}{2}\mathbb{E}_{(x_h^s,a_h^s)\sim \pi_{\mathrm{exp}}(f^s)}\left[\int_{\mathcal{S}}\left(\sqrt{\mathbb{P}_{h,f}(x_{h+1}^s|x_h^s,a_h^s)} - \sqrt{\mathbb{P}_{h,f^{\ast}}(x_{h+1}^s|x_h^s,a_h^s)}\right)^2\mathrm{d}x_{h+1}^s\right]\notag\\
        & = 1 - \mathbb{E}_{(x_h^s,a_h^s)\sim \pi_{\mathrm{exp}}(f^s)}\Big[D_{\mathrm{H}}(\mathbb{P}_{h,f^{\ast}}(\cdot|x_h^s,a_h^s)\|\mathbb{P}_{h,f}(\cdot|x_h^s,a_h^s))\Big],\label{eq: proof model based supervised DH}
    \end{align}
    where the first equality uses the definition of $X_{h,f}^s$ in \eqref{eq: X model-based}, the second equality is due to the fact that $\xi^s_h\sim \pi^s$ and $\pi^s\in\mathcal{F}_{s-1}$, and the last equality uses the definition of Hellinger distance $D_{\mathrm{H}}$.
    Thus by combining \eqref{eq: proof model based supervised concentration} and \eqref{eq: proof model based supervised DH}, we can derive that 
    \begin{align*}
        -\frac{1}{2}\sum_{s=1}^{k-1}X_{h,f}^s &\leq \sum_{s=1}^{k-1} \mathbb{E}\left[\exp\left\{-\frac{1}{2}X_{h,f}^s\right\} \middle| \mathcal{F}_{s-1}\right] -1 + \log(H|\mathcal{H}_h|/\delta) \notag\\
        &  = - \sum_{s=1}^{k-1}\mathbb{E}_{(x_h^s,a_h^s)\sim \pi_{\mathrm{exp}}(f^s)}\left[D_{\mathrm{H}}(\mathbb{P}_{h,f^{\ast}}(\cdot|x_h^s,a_h^s)\|\mathbb{P}_{h,f}(\cdot|x_h^s,a_h^s))\right] + \log(H|\mathcal{H}_h|/\delta), \notag
    \end{align*}
    where in the first inequality we use the fact that $\log(x)\leq x-1$.
    Finally, by plugging in the definition of $X_{h,f}^s$, summing over $h\in[H]$, we have that with probability at least $1-\delta$, for any $f\in\mathcal{H}$, any $k\in[K]$, it holds that
    \begin{align}
        &\sum_{h=1}^HL_h^{k-1}(f^{\ast}) - L_h^{k-1}(f) = -\sum_{h=1}^H\sum_{s=1}^{k-1}X_{h,f}^s\notag\\ 
        & \qquad \leq -2\sum_{h=1}^H\sum_{s=1}^{k-1} \mathbb{E}_{(x_h^s,a_h^s)\sim \pi_{\mathrm{exp}}(f^s)}\left[D_{\mathrm{H}}(\mathbb{P}_{h,f^{\ast}}(\cdot|x_h^s,a_h^s)\|\mathbb{P}_{h,f}(\cdot|x_h^s,a_h^s))\right] + 2H\log(H/\delta) + 2\log(|\mathcal{H}|), \notag\\
        &\qquad = - 2\sum_{h=1}^H\sum_{s=1}^{k-1}\mathbb{E}_{\xi_h\sim \pi_{\mathrm{exp}}(f^s)}[\ell_{f^s}(f;\xi_h)] + 2H\log(H/\delta) + 2\log(|\mathcal{H}|).
    \end{align}
    This finishes the proof of Proposition \ref{prop: supervised learning model based rl}.
\end{proof}

%% file: tex/appendix/proof_mg_example.tex
\section{Proofs for Model-free and Model-based Online RL in Two-player Zero-sum MGs}

\subsection{Proof of Proposition \ref{prop: tgec linear mg}}\label{app:linear MG1}
\begin{proof}[Proof of Proposition \ref{prop: tgec linear mg}]
    To begin with, we need to introduce the performance difference lemma in two-player zero-sum MG, which are presented in Lemma 1 and Lemma 2 in \citet{xiong22b}. 
    \begin{lemma}[Value decomposition for the max-player] \label{lem: decomposition max}Let $\mu=\mu_f$ and $\nu$ be an arbitrary policy taken by the min-player. It holds that
\begin{align}
V_{1,f}\left(x_1\right)-V_1^{\mu, \nu}\left(x_1\right) 
& \leq \sum_{h=1}^H \mathbb{E}_{\xi_h\sim(\mu, \nu)}\left[ \mathcal{E}_h(f_h, f_{h+1}; \xi_h)\right]\label{ineq:11}
\end{align}
where max-player Bellman error $\mathcal{E}_h(f_h, f_{h+1}  ; \xi_h)$ is defined as 
\begin{align}
    \mathcal{E}_h(f_h, f_{h+1}; \xi_h) = Q_{h,f}(x_h,a_h,b_h) - r_h - (\mathbb{P}_h V_{h+1,f})(x_h,a_h,b_h), \label{eq:,g_be}
\end{align}
and $\xi_h = (x_h,a_h,b_h,r_h)$. (Actually, this coincides with the NE Bellman error defined in \eqref{eq: ne bellman}.)
\label{lem:mg_perform_1}
\end{lemma}
     \begin{lemma}[Value decomposition for the min-player] Suppose that $\mu=\mu_f$ is taken by the max-player and $g$ is the hypothesis selected by the min-player. Let $\nu$ be the policy taken by the min-player. Then, it holds that
\begin{align}
V_1^{\mu, \nu}\left(x_1\right)-V_{1,g}^{\mu,\dagger}\left(x_1\right)
& =-\sum_{h=1}^H \mathbb{E}_{\xi_h\sim(\mu, \nu)} \left[\mathcal{E}_h^{\mu}(g_h, g_{h+1}; \xi_h)\right] , \label{ineq:12}
\end{align} where the min-player Bellman error $\mathcal{E}_h^\mu(g_h, g_{h+1} ; \xi_h)$ is defined as 
\begin{align}
    \mathcal{E}_h^\mu(g_h, g_{h+1} ; \xi_h) = Q_{h,g}^{\mu,\dagger}(x_h,a_h,b_h) - r_h - (\mathbb{P}_h V_{h+1,g}^{\mu,\dagger})(x_h,a_h,b_h), \label{eq:,g_be 2}
\end{align}
and $\xi_h = (x_h,a_h,b_h,r_h)$. 
\label{lem:mg_perform_2}  \end{lemma}
We note that the value decomposition for the max-player is an inequality because of the property of minimax formulation. 
Note also that the right side of \eqref{ineq:12} is a general version of the right side of \eqref{ineq:11} when choosing $\mu = \mu_f$. 
Now we are ready to prove Proposition \ref{prop: tgec linear mg}.
The lemmas suggest that we only need to upper-bound the term $\sum_{k=1}^K\sum_{h=1}^H |\mathbb{E}_{\boldsymbol{\pi}^k}[ \mathcal{E}_h^{  \mu}(g_h^k, g_{h+1}^k; \xi_h)]| $ for all admissible max-player policy $ \mu$.
To this end, we provide a more general result by the following proposition.
For simplicity, we denote by $\boldsymbol{\pi}^k = (\mu^k,\nu^k)$.

\begin{prop}
    For a $d$-dimensional two-player zero-sum Markov game, we assume that its  expected min-player bellman error can be decomposed as follows  
    \begin{align}\label{eq: bilinear condition 1}
        \mathbb{E}_{\xi_h\sim\boldsymbol{\pi}^s} \left[\mathcal{E}_h^{\mu}(g_h, g_{h+1}; \xi_h)\right] = \langle W_h(g,\mu),X_h(g,\boldsymbol{\pi}^s,\mu)\rangle,
    \end{align}
    for some $W_h(g,\mu), X_h(g,\boldsymbol{\pi},\mu) \in \mathbb{R}^d$, and the discrepancy function $\ell_{g^\prime,\mu}(g;\xi_h)$ can be lower bounded as follows
    \begin{align}
        |\langle W_h(g,\mu), X_h(g^\prime,\boldsymbol{\pi},\mu)\rangle|^2 \le \mathbb{E}_{\xi_h\sim \boldsymbol{\pi}} [\ell_{g^\prime,\mu}(g;\xi_h)], \label{eq:mg:bi:c2}
    \end{align}
    for all the admissible max-player policy $\mu\in\mathbf{M}$. 
    Also, we assume that $\|W_h(\cdot,\cdot)\|_2\le B_W$, $\|X_h(\cdot,\cdot,\cdot)\|_2\le B_X$ for some $B_W,B_X>0$ and for all timestep $h\in[H]$. 
    Then it holds that 
 \begin{align}
    \sum_{k=1}^K\sum_{h=1}^H\big|\mathbb{E}_{\xi_h\sim \boldsymbol{\pi}^k} \left[\mathcal{E}_h^{ \mu}(g_h^k, g_{h+1}^k; \xi_h)\right]\big| &\le \frac{\Tilde{d}(\epsilon)}{4\eta} +\frac{\eta}{2}\sum_{k=1}^K \sum_{h=1}^H \sum_{s=1}^{k-1}\mathbb{E}_{\boldsymbol{\pi}^s} [\ell_{g^s,\mu}(g^k;\xi_h)] +2\min\{HK,2\Tilde{d}(\epsilon)\} +HKB_W\epsilon,   \notag
 \end{align}
 for all admissible max-player policy $\mu\in\mathbf{M}$, $\epsilon\in[0,1]$, $\eta>0$, and $\Tilde{d}(\epsilon):=d\log (1+K B_X^2/(d \epsilon))$.\label{prop:mg_bilinear}
\end{prop}
\begin{proof}
    [Proof of Proposition \ref{prop:mg_bilinear}] We prove this result following a similar procedure as in the proof of Lemma 3.20 in \cite{zhong2022posterior}, where they prove that the low-GEC class contains the bilinear class.
     We denote by 
     $$
     \Sigma_{h,k} = \epsilon I_d + \sum_{s=1}^{k-1} X_h(g^s,\boldsymbol{\pi}^s,\mu)  X_h(g^s,\boldsymbol{\pi}^s,\mu)^\top.
     $$ 
     By Lemma F.3 in \cite{du2021bilinear} and Lemma~\ref{lem:ell}, we first have the following equality,
    \begin{align}
        \sum_{s=1}^k \min\left\{\| X_h(g^s,\boldsymbol{\pi}^s,\mu)\|_{\Sigma_{h,s}^{-1}},1\right\}\le 2\Tilde{d}(\epsilon),\label{eq:mg-ell}
    \end{align} 
    for all $\epsilon\in[0,1]$. 
    Here $\Tilde{d}(\epsilon)$ is defined in Proposition \ref{prop:mg_bilinear}.
    Now, since the reward is bounded by $[0,1]$, we have the following inequalities,
\begin{align}
      &\sum_{k=1}^K\sum_{h=1}^H \big|\mathbb{E}_{\boldsymbol{\pi}^k} \left[\mathcal{E}_h^{ \mu}(g_h^k, g_{h+1}^k; \xi_h)\right]\big|\notag\\
      &\qquad =  \sum_{k=1}^K\sum_{h=1}^H \min\{1,\langle  W_h(g^k,\mu), X_h(g^k,\boldsymbol{\pi}^k,\mu)\rangle\} \textbf{1}\left\{\| X_h(g^k,\boldsymbol{\pi}^k,\mu)\|_{\Sigma_{h,k}^{-1}}\le 1\right\}\notag \\
      &\qquad\qquad +\sum_{k=1}^K\sum_{h=1}^H\min\{1,\langle W_h(g^k,\mu), X_h(g^k,\boldsymbol{\pi}^k,\mu)\rangle\} \textbf{1}\left\{\| X_h(g^k,\boldsymbol{\pi}^k,\mu) \|_{\Sigma^{-1}_{h,k}}>1\right\}\notag\\
      & \qquad\le \sum_{k=1}^K\sum_{h=1}^H
    {  \langle W_h(g^k,\mu), X_h(g^k,\boldsymbol{\pi}^k,\mu)\rangle\textbf{1}\left\{\| X_h(g^k,\boldsymbol{\pi}^k,\mu)\|_{\Sigma_{h,k}^{-1}}\le 1\right\}}+\min\{HK,\Tilde{d}(\epsilon)\}\notag\\
    &\qquad \le \sum_{k=1}^K\sum_{h=1}^H
    \underbrace{ \|W_h(g^k,\mu)\|_{\Sigma_{h,k}}\min\left\{\| X_h(g^k,\boldsymbol{\pi}^k,\mu)\|_{\Sigma_{h,k}^{-1}},1\right\}}_{\text{(A)}_{h,k}}+\min\{HK,\Tilde{d}(\epsilon)\},\label{eq:mg term1}
\end{align}
where the first equality relies on the assumption in Proposition \ref{prop:mg_bilinear}, the second inequality comes from \eqref{eq:mg-ell}, and the last inequality is based on Cauchy Schwarz inequality.
Now we expand term (A$)_{h,k}$ in \eqref{eq:mg term1} as follows.
\begin{align}
   \|W_h(g^k,\mu)\|_{\Sigma_{h,k}}% &\leq \left[\epsilon\|W_h(g^k,\mu)\|_2^2 +\sum_{s=1}^{k-1} |\langle W_h(g^k,\mu), X_h(g^s,\boldsymbol{\pi}^s,\mu)\rangle|^2\right]^{1/2}\notag\\
\le \sqrt{\epsilon}B_W + \left[\sum_{s=1}^{k-1} |\langle W_h(g^k,\mu), X_h(g^s,\boldsymbol{\pi}^s,\mu)\rangle|^2\right]^{1/2}\notag,
\end{align}
where we use the fact that $\|W_h(g^k,\mu)\|_2\le B_W$. Thus we have that 
\begin{align}
    \sum_{k=1}^K\sum_{h=1}^H \text{(A)}_{h,k}&
    \le \sum_{k=1}^K \sum_{h=1}^H\left(\sqrt{\epsilon}B_W + \left[\sum_{s=1}^{k-1} |\langle W_h(g^k,\mu), X_h(g^s,\boldsymbol{\pi}^s,\mu)\rangle|^2\right]^{1/2}\right)\cdot \min\left\{\| X_h(g^k,\boldsymbol{\pi}^k,\mu)\|_{\Sigma_{h,k}^{-1}},1\right\}\notag\\
    &\le \left[\sum_{k=1}^K \sum_{h=1}^H\sqrt{\epsilon}B_W\right]^{1/2}\cdot\left[\sum_{k=1}^K\sum_{h=1}^H \min\left\{\| X_h(g^k,\boldsymbol{\pi}^k,\mu)\|_{\Sigma_{h,k}^{-1}},1\right\}\right]^{1/2} \notag\\
    &\qquad+\left[\sum_{k=1}^K \sum_{h=1}^H\sum_{s=1}^{k-1} |\langle W_h(g^k,\mu), X_h(g^s,\boldsymbol{\pi}^s,\mu)\rangle|^2\right]^{1/2}\cdot\left[\sum_{k=1}^K\sum_{h=1}^H \min\left\{\| X_h(g^k,\boldsymbol{\pi}^k,\mu)\|_{\Sigma_{h,k}^{-1}},1\right\}\right]^{1/2}\notag\\
    &\le \sqrt{HB_WK\epsilon\cdot\min\{2\Tilde{d}(\epsilon),HK\}} +\left[2\Tilde{d}(\epsilon)\sum_{k=1}^K \sum_{h=1}^H\sum_{s=1}^{k-1} |\langle W_h(g^k,\mu), X_h(g^s,\boldsymbol{\pi}^s,\mu)\rangle|^2\right]^{1/2}\notag\\
    &\le \sqrt{HKB_W\epsilon\cdot\min\{2\Tilde{d}(\epsilon),HK\}} +\left[2\Tilde{d}(\epsilon)\sum_{k=1}^K \sum_{h=1}^H \sum_{s=1}^{k-1}\mathbb{E}_{\xi_h\sim \boldsymbol{\pi}^s} [\ell_{g^s,\mu}(g^k;\xi_h)]\right]^{1/2}\notag,
\end{align}
where the second inequality is the result of Cauchy-Schwarz inequality, the third inequality comes from \eqref{eq:mg-ell}, and the last inequality is derived from \eqref{eq:mg:bi:c2}.
Back to the analysis for \eqref{eq:mg term1}, we have that 
\begin{align}
    \sum_{k=1}^K\sum_{h=1}^H \big|\mathbb{E}_{\boldsymbol{\pi}^k} \left[\mathcal{E}_h^{ \mu}(g_h^k,g_{h+1}^k; \xi_h)\right]\big|&\le \sqrt{HKB_W\epsilon\cdot\min\{2\Tilde{d}(\epsilon),HK\}}  \notag \\
    &\qquad +\left[2\Tilde{d}(\epsilon)\sum_{k=1}^K \sum_{h=1}^H \sum_{s=1}^{k-1}\mathbb{E}_{\xi_h\sim \boldsymbol{\pi}^s} [\ell_{g^s,\mu}(g^k;\xi_h)]\right]^{1/2}+\min\{HK,2\Tilde{d}(\epsilon)\}\notag
    \\ & \le \left[2\Tilde{d}(\epsilon)\sum_{k=1}^K \sum_{h=1}^H \sum_{s=1}^{k-1}\mathbb{E}_{\xi_h\sim\boldsymbol{\pi}^s} [\ell_{g^s,\mu}(g^k;\xi_h)]\right]^{1/2}+2\min\{HK,2\Tilde{d}(\epsilon)\}+HKB_W\epsilon\notag\\
    &\le \frac{\Tilde{d}(\epsilon)}{4\eta} +\frac{\eta}{2}\sum_{k=1}^K \sum_{h=1}^H \sum_{s=1}^{k-1}\mathbb{E}_{\xi_h\sim \boldsymbol{\pi}^s} [\ell_{g^s,h}(g^k;\xi_h)] +2\min\{HK,2\Tilde{d}(\epsilon)\} +HKB_W\epsilon,   \notag
\end{align}
where the second inequality comes from the AM-GM inequality and the last inequality uses the basic inequality $2ab \le a^2+b^2$. Here $\eta>0$ can be arbitrarily chosen. Then we finish our proof to Proposition~\ref{prop:mg_bilinear}.
\end{proof}

Back to our proof of Proposition~\ref{prop: tgec linear mg}, we first check the conditions of Proposition \ref{prop:mg_bilinear} for linear two-player zero-sum MGs.  
By Definition \ref{def: linear MG} and the choice of model-free hypothesis class \eqref{eq: hypothesis linear mg}, we know that for any $g\in\cH$ and $\mu\in\mathbf{N}$, it holds that
$$
\begin{aligned}
    Q_{h,g}(x,a,b) - r_h(x,a,b) - (\mathbb{P}_{h}V^{\mu,\dagger}_{h+1,g})(x,a,b) 
     = \phi_h(x,a,b)^\top \left(\theta_{h,g}-\alpha_h - \int_\mathcal{S}\psi_h^\star(x^\prime)V^{\mu,\dagger}_{h+1,g}(x^\prime) \mathrm{d}x^\prime\right),
\end{aligned}
$$
where $\theta_{h,g}$ denotes the parameter of $Q_{h,g}$ and $\alpha_h$ is the reward parameter (see Definition \ref{def: linear MG}). 
Thus we can define $X_h(g,\boldsymbol{\pi},\mu) = \mathbb{E}_{\boldsymbol{\pi}}[\phi_h(x,a,b)]$ and $$W_h(g,\mu) = \theta_{h,g}-\alpha_h - \int_\mathcal{S}\psi_h^\star(x^\prime)V^{\mu,\dagger}_{h+1,g}(x^\prime) \mathrm{d}x^\prime.$$ 
This specifies condition \eqref{eq: bilinear condition 1} of Proposition \ref{prop:mg_bilinear}.
By Jansen inequality and the definition of $\ell_{\mu}$ in \eqref{eq: example model free ell mg 2}, it is obvious that the condition \eqref{eq:mg:bi:c2} of Proposition \ref{prop:mg_bilinear} holds.
By the assumptions of linear two-player zero-sum MGs in Definition \ref{def: linear MG}, we have $B_X \le 1$ and $B_W \le 4H\sqrt{d}$. 
Thus by applying Proposition \ref{prop:mg_bilinear}, we have that 
\begin{align*}
    \sum_{k=1}^KV_1^{\boldsymbol{\pi}^k}(x_1) - V_{1,g^k}^{\mu^k,\dagger}(x_1)&\leq \sum_{k=1}^K\sum_{h=1}^H\big|\mathbb{E}_{\xi_h\sim \boldsymbol{\pi}^k} \left[\mathcal{E}_h^{ \mu}(g_h^k, g_{h+1}^k; \xi_h)\right]\big|\\
    &\leq \frac{\Tilde{d}(\epsilon)}{4\eta} +\frac{\eta}{2}\sum_{k=1}^K \sum_{h=1}^H \sum_{s=1}^{k-1}\mathbb{E}_{\xi_h\sim \boldsymbol{\pi}^s} [\ell_{g^s,\mu}(g^k;\xi_h)] +2\min\{HK,2\Tilde{d}(\epsilon)\} +4\sqrt{d}H^2K\epsilon,
\end{align*}
with $\tilde{d}(\epsilon)=d\log(1+K/d\epsilon)$ and any $\eta>0$.
This proves the second inequality of Assumption \ref{ass: gec mg}.
For the first inequality in Assumption \ref{ass: gec mg}, we take $g^k = f^k$, $\mu = \mu_{f^k}$, and we can then similarly prove that 
\begin{align*}
    \sum_{k=1}^KV_{1,f^k}(x_1) - V_1^{\boldsymbol{\pi}^k}(x_1)\leq
 \frac{\Tilde{d}(\epsilon)}{4\eta} +\frac{\eta}{2}\sum_{k=1}^K \sum_{h=1}^H \sum_{s=1}^{k-1}\mathbb{E}_{\xi_h\sim \boldsymbol{\pi}^s} [\ell_{f^s}(f^k;\xi_h)] +2\min\{HK,2\Tilde{d}(\epsilon)\} +4\sqrt{d}H^2K\epsilon,
\end{align*}
with $\tilde{d}(\epsilon)=d\log(1+K/d\epsilon)$ and any $\eta>0$.
This proves that $d_{\mathrm{TGEC}}(\epsilon) \le \tilde{d}(\epsilon)$.

As for the analysis for covering number, we apply the standard analysis for the covering number of $\mathbb{R}^d$-ball to obtain that 
\begin{align}\log\mathcal{N}(\mathcal{H},\epsilon, \|\cdot\|_{\infty})\le d \log \left(\frac{3}{\epsilon}\right)+d\log\left(\frac{\textbf{Vol}(\mathcal{H})}{\textbf{Vol}(B_d)}\right),\notag
\end{align}
for all $\epsilon\le 1$ and the unit ball $B_d$ in $\mathbb{R}^d$ space. 
Selecting $\epsilon = 1/K$, we finish the proof of Proposition \ref{prop: tgec linear mg}.
\end{proof}

\subsection{Proof of Proposition \ref{prop: tgec linear mixture mg}}
\label{app: linear mixture two-player zero-sum MG}
\begin{proof}[Proof of Proposition \ref{prop: tgec linear mixture mg}]
Similar to the proof of Proposition \ref{prop: tgec linear mg}, we can apply Lemma \ref{lem:mg_perform_1}, Lemma \ref{lem:mg_perform_2}, and Proposition \ref{prop:mg_bilinear} to obtain the upper bound of TGEC for linear mixture two-player zero-sum MGs. 
First we need to check the conditions of Proposition \ref{prop:mg_bilinear}. 
Note that
\begin{align}
    Q_{h,g}^{\mu,\dagger}(x,a,b) - r_h -(\mathbb{P}_{h} V_{h+1,g}^{\mu,\dagger})(x,a,b)& = (\mathbb{P}_{h,g} V_{h+1,g}^{\mu,\dagger})(x,a,b)-(\mathbb{P}_{h} V_{h+1,g}^{\mu,\dagger})(x,a,b)\notag\\
    &= \big(\theta_{h,g} - \theta_{h}^\star\big)^\top \left(\int_\mathcal{S} \phi_h(x,a,b,x^\prime)V_{h+1,g}^{\mu,\dagger}(x^\prime)\mathrm{d}x^\prime\right),
\end{align}
where the first equality comes from the Bellman equation, and the second equality is derived from the definition of linear mixture two-player zero-sum MG (Definition \ref{def: linear mixture MG}). Here $\theta_{h,g}$ denotes the parameter of $\mathbb{P}_{h,g}$. Hence we can define $X_h$ and $W_h$ as
\begin{align}
    X_h(g,\boldsymbol{\pi},\mu) := \mathbb{E}_{\boldsymbol{\pi}}\left[\int_\mathcal{S} \phi_h(x,a,b,x^\prime)V_{h+1,g}^{\mu,\dagger}(x^\prime)\mathrm{d}x^\prime\right],\quad W_h(g,\mu):= \theta_{h,g} - \theta_{h}^\star.
\end{align}
This specifies condition \eqref{eq: bilinear condition 1} of Proposition \ref{prop:mg_bilinear}.
By the assumptions of linear mixture two-player zero-sum MGs in Definition \ref{def: linear mixture MG}, we can obtain that $B_X \le 1$ and $B_W \le 4H\sqrt{d}$.
%By Jansen inequality and the definition of $\ell_{\mu}$ in \eqref{eq: example model based ell mg}, it is obvious that the condition \eqref{eq:mg:bi:c2} of Proposition \ref{prop:mg_bilinear} holds.
As for condition \eqref{eq:mg:bi:c2}, different from the proof of Proposition \ref{prop: tgec linear mg}, since we use Hellinger distance as the discrepancy function $\ell$ for the model-based hypothesis, we propose to connect it to the model-free discrepancy function \eqref{eq: example model free ell mg 2}. 
Notice that
\begin{align}
    \left(Q_{h,g}^{\mu,\dagger}(x,a,b) - r_h -(\mathbb{P}_{h} V_{h+1,g}^{\mu,\dagger})(x,a,b)\right)^2 &= \left((\mathbb{P}_{h,g} V_{h+1,g}^{\mu,\dagger})(x,a,b)-(\mathbb{P}_{h} V_{h+1,g}^{\mu,\dagger})(x,a,b)\right)^2\notag\\
    &\le 4\|V_{h+1,g}^{\mu,\dagger}(\cdot)\|_\infty^2 \cdot D_{\text{TV}}(\mathbb{P}_{h,g}(\cdot\mid x,a,b)\|\mathbb{P}_h(\cdot\mid x,a,b))^2 \notag\\
    &\le 2H^2 D_{\text{H}}(\mathbb{P}_{h,g}(\cdot\mid x,a,b)\|\mathbb{P}_h(\cdot\mid x,a,b))^2 \notag \\
    &\le 2H^2 D_{\text{H}}(\mathbb{P}_{h,g}(\cdot\mid x,a,b)\|\mathbb{P}_h(\cdot\mid x,a,b)),
\end{align}
where the second equality comes from Holder inequality and the fact that the TV distance $D_{\text{TV}}(p\|q) = \|p-q\|_1 /2$ for any two distributions $p$ and $q$, the third inequality follows from the fact that  $D_{\text{TV}}(p\|q) \le \sqrt{2}D_{\text{H}}(p\|q)$, and the last inequality follows from the fact that $D_{\mathrm{H}}(p\|q)\leq 1$.
This shows that the model-based discrepancy function defined in \eqref{eq: example model based ell mg} upper-bounds the model-free discrepancy function up to a factor $2H^2$, that is,
\begin{align}
    \mathbb{E}_{\xi_h\sim \boldsymbol{\pi}}[\ell_{g',\mu}(g;\xi_h)] &=  \mathbb{E}_{\xi_h\sim \boldsymbol{\pi}}[D_{\mathrm{H}}(\mathbb{P}_{h,g}(\cdot|x_h,a_h, b_h)\|\mathbb{P}_{h}(\cdot|x_h,a_h, b_h))]\notag\\
    &\ge \frac{1}{2H^2} \mathbb{E}_{\xi_h\sim \boldsymbol{\pi}}\left[\left(Q_{h,g}^{\mu,\dagger}(x_h,a_h,b_h) - r_h -(\mathbb{P}_{h} V_{h+1,g}^{\mu,\dagger})(x_h,a_h,b_H)\right)^2\right]\notag \\
    &= |\langle W_h(g,\mu),X_h(g,\boldsymbol{\pi}, \mu)\rangle|^2.
\end{align}
Thus by applying Proposition \ref{prop:mg_bilinear}, we have that 
\begin{align*}
    \sum_{k=1}^KV_1^{\boldsymbol{\pi}^k} - V_{1,g^k}^{\mu^k,\dagger}&\leq \sum_{k=1}^K\sum_{h=1}^H\big|\mathbb{E}_{\xi_h\sim \boldsymbol{\pi}^k} \left[\mathcal{E}_h^{ \mu}(g_h^k, g_{h+1}^k; \xi_h)\right]\big|\\
    &\leq \frac{\Tilde{d}(\epsilon)}{4\eta} +\frac{\eta}{4H^2}\sum_{k=1}^K \sum_{h=1}^H \sum_{s=1}^{k-1}\mathbb{E}_{\xi_h\sim \boldsymbol{\pi}^s} [\ell_{g^s,\mu}(g^k;\xi_h)] +2\min\{HK,2\Tilde{d}(\epsilon)\} +4\sqrt{d}H^2K\epsilon\\
    & = \frac{\bar{d}(\epsilon)}{4\eta^\prime} +\frac{\eta^\prime}{2}\sum_{k=1}^K \sum_{h=1}^H \sum_{s=1}^{k-1}\mathbb{E}_{\xi_h\sim \boldsymbol{\pi}^s} [\ell_{g^s,\mu}(g^k;\xi_h)] +2\min\{HK,2\bar{d}(\epsilon)\} +4\sqrt{d}H^2K\epsilon,
\end{align*}
with $\bar{d}(\epsilon) = 2H^2\tilde{d}(\epsilon)=2H^2d\log(1+K/d\epsilon)$ and any $\eta>0$ and $\eta^\prime = \eta/(2H^2)$.
This proves the second inequality of Assumption \ref{ass: gec mg}.
For the first inequality in Assumption \ref{ass: gec mg}, we take $g^k = f^k$ and let $\mu = \mu_{f^k}$, and we can then also similarly prove that 
\begin{align*}
    \sum_{k=1}^KV_{1,f^k} - V_1^{\boldsymbol{\pi}^k}\leq
 \frac{\bar{d}(\epsilon)}{4\eta^\prime} +\frac{\eta^\prime}{2}\sum_{k=1}^K \sum_{h=1}^H \sum_{s=1}^{k-1}\mathbb{E}_{\xi_h\sim\boldsymbol{\pi}^s} [\ell_{f^s}(f^k;\xi_h)] +2\min\{HK,2\bar{d}(\epsilon)\} +4\sqrt{d}H^2K\epsilon.
\end{align*}
This proves that $d_{\mathrm{TGEC}}(\epsilon) \le \bar{d}(\epsilon)$.
As for the analysis of the covering number, it suffices to repeat the same as the proof of Proposition \ref{prop: tgec linear mg}.
This finishes the proof of Proposition \ref{prop: tgec linear mixture mg}.
\end{proof}

\subsection{Proof of Proposition \ref{prop: supervised learning model free rl mg}}\label{subsec: proof prop supervised learning model free rl mg}

\begin{proof}[Proof of Proposition \ref{prop: supervised learning model free rl mg}]
    We first prove the \emph{first} inequality of Proposition \ref{prop: supervised learning model free rl mg}.
    To this end, we define the random vairable $X_{h,f}^k$ as 
    \begin{align}\label{eq: proof supervised learning model free rl mg 0}
        X_{h,f}^k &= \left(Q_{h,f}(x_h^k,a_h^k,b_h^k) - r_h^k - V_{h+1, f}(x_{h+1}^k)\right)^2 \notag \\
        &\qquad  - \left(V_{h+1, f}(x_{h+1}^k) - \mathbb{E}_{x_{h+1}\sim\mathbb{P}_h(\cdot|x_h^k,a_h^k,b_h^k)}[V_{h+1,f}(x_{h+1})]\right)^2.
    \end{align}
    After a calculation similar to \eqref{eq: proof supervised learning model free rl 1} and  \eqref{eq: proof supervised learning model free rl 2}, we can derive that 
    \begin{align*}
        \mathbb{E}_{x_{h+1}^k\sim\mathbb{P}_h(\cdot|x_h^k,a_h^k,b_h^k)}[X_{h,f}^k] = \Big(Q_{h,f}(x_h^k,a_h^k,b_h^k) - r_h^k - \mathbb{E}_{x_{h+1}\sim\mathbb{P}_h(\cdot|x_h^k,a_h^k,b_h^k)}[V_{h+1,f}(x_{h+1})]\Big)^2.
    \end{align*}
    Now for each timestep $h$, we define a filtration $\{\mathcal{F}_{h,k}\}_{k=1}^K$ with 
    \begin{align}\label{eq: filtration model free mg}
        \mathcal{F}_{h,k} = \sigma\left(\bigcup_{s=1}^{k}\bigcup_{h=1}^H\mathcal{D}_h^s\right),
    \end{align}
    where $\mathcal{D}_h^s = \{x_h^s,a_h^s,b_h^s,r_h^s,x_{h+1}^s\}$. 
    With previous arguments, we can derive that 
    \begin{align}\label{eq: proof supervised learning model free rl mg 2+}
        \mathbb{E}[X_{h,f}^k|\mathcal{F}_{h,k-1}] = \mathbb{E}\left[\mathbb{E}_{x^k_{h+1}\sim \mathbb{P}_h(\cdot|x_h^k,a_h^k, b_h^k)}[X_{h,f^k}]|\mathcal{F}_{h,k-1}\right] = \mathbb{E}_{\xi_h\sim \boldsymbol{\pi}^k}[ \ell_{f^k}(f;\xi_h)],
    \end{align}
    and that 
    \begin{align}\label{eq: proof supervised learning model free rl mg 2++}
        \mathbb{V}[X_{h,f}^k|\mathcal{F}_{h,k-1}] \leq \mathbb{E}[(X_{h,f}^k)^2|\mathcal{F}_{h,k-1}] \leq 4B_f^2\mathbb{E}[X_{h,f}^k|\mathcal{F}_{h,k-1}] = 4B_f^2\mathbb{E}_{\xi_h\sim \boldsymbol{\pi}^k}[ \ell_{f^k}(f;\xi_h)],
    \end{align}
    where $B$ is the upper bound of hypothesis in $\mathcal{H}$ by Assumption \ref{ass: completeness mg}.
    By applying Lemma \ref{lem: freedman}, \eqref{eq: proof supervised learning model free rl mg 2+}, and \eqref{eq: proof supervised learning model free rl mg 2++}, we can obtain that with probability at least $1-\delta$, for any $(h,k)\in[H]\times[K]$ and $(f_h, f_{h+1})\in\mathcal{H}_h\times\mathcal{H}_{h+1}$,
    \begin{align}\label{eq: proof supervised learning model free rl mg 3}
        \left|\sum_{s=1}^{k-1}\mathbb{E}_{\xi_h\sim \boldsymbol{\pi}^s}[ \ell_{f^s}(f;\xi_h)] - \sum_{s=1}^{k-1}X_{h,f}^s\right| \lesssim \frac{1}{2}\sum_{s=1}^{k-1}\mathbb{E}_{\xi_h\sim \boldsymbol{\pi}^s}[ \ell_{f^s}(f;\xi_h)] +  8B_f^2\log(HK|\mathcal{H}_h||\mathcal{H}_{h+1}|/\delta).
    \end{align}
    Rearranging terms in \eqref{eq: proof supervised learning model free rl mg 3}, we can further obtain that 
    \begin{align}\label{eq: proof supervised learning model free rl mg 4}
        - \sum_{s=1}^{k-1}X_{h,f}^s \lesssim -\frac{1}{2}\sum_{s=1}^{k-1}\mathbb{E}_{\xi_h\sim \boldsymbol{\pi}^s}[ \ell_{f^s}(f;\xi_h)] +  8B_f^2\log(HK|\mathcal{H}_h||\mathcal{H}_{h+1}|/\delta).
    \end{align}
    Meanwhile, by the definition of $X_{h,f}$ in \eqref{eq: proof supervised learning model free rl mg 0} and the loss function $L$ in \eqref{eq: implicit optimism mg model free L 1}, we have that 
    \begin{align}\label{eq: proof supervised learning model free rl mg 5}
        &\sum_{s=1}^{k-1}X_{h,f}^s \notag \\
        &\quad = \sum_{s=1}^{k-1}\left(Q_{h,f}(x_h^s,a_h^s,b_h^s) - r_h^s - V_{h+1, f}(x_{h+1}^s)\right)^2  - \sum_{s=1}^{k-1}\left(V_{h+1, f}(x_{h+1}^s) - \mathbb{E}_{x_{h+1}\sim\mathbb{P}_h(\cdot|x_h^s,a_h^s,b_h^s)}[V_{h+1,f}(x_{h+1})]\right)^2\notag\\
        &\quad= \sum_{s=1}^{k-1}\left(Q_{h,f}(x_h^s,a_h^s,b_h^s) - r_h^s - V_{h+1, f}(x_{h+1}^s)\right)^2 - \sum_{s=1}^{k-1}\left(\mathcal{T}_hf(x_h^s,a_h^s,b_h^s) - r_h^s - V_{h+1, f}(x_{h+1}^s)\right)^2\notag\\
        &\quad \leq \sum_{s=1}^{k-1} \left(Q_{h,f}(x_h^s,a_h^s,b_h^s) - r_h^s - V_{h+1, f}(x_{h+1}^s)\right)^2 - \inf_{f_h'\in\mathcal{H}_h}\sum_{s=1}^{k-1} \left(Q_{h,f'}(x_h^s,a_h^s,b_h^s) - r_h^s - V_{h+1, f}(x_{h+1}^s)\right)^2\notag\\
        &\quad =L_h^{k-1}(f).
    \end{align}
    where the last inequality follows from the completeness assumption (Assumption \ref{ass: completeness mg}).
    Combining \eqref{eq: proof supervised learning model free rl mg 4} and \eqref{eq: proof supervised learning model free rl mg 5}, we can derive that with probability at least $1-\delta$, for any $f\in\mathcal{H}$, $k\in[K]$,
    \begin{align}\label{eq: proof supervised learning model free rl mg 6}
        -\sum_{h=1}^HL_h^{k-1}(f) \lesssim -\frac{1}{2}\sum_{h=1}^H\sum_{s=1}^{k-1}\mathbb{E}_{\xi_h\sim \boldsymbol{\pi}^s}[ \ell_{f^s}(f;\xi_h)] +  8HB_f^2\log(HK/\delta) + 16B_f^2\log(|\mathcal{H}|).
    \end{align}
    Finally, we deal with the term $L_h^k(f^{\ast})$. 
    To this end, we invoke the following lemma.

    \begin{lemma} \label{lem: f star model free mg}
        With probability at least $1-\delta$, it holds that for each $k\in[K]$,
        \begin{align*}
            \sum_{h=1}^HL^{k-1}_h(f^{\ast}) \lesssim 8HB_f^2\log(HK/\delta) + 16B_f^2\log(|\mathcal{H}|).
        \end{align*}
    \end{lemma}

    \begin{proof}[Proof of Lemma \ref{lem: f star model free mg}]
        To prove Lemma \ref{lem: f star model free mg}, we define the random variable $W_{h,f}$ as 
        \begin{align*}
            W_{h,f}^k = \left(Q_{h,f}(x_h^k,a_h^k,b_h^k) - r_h^k - V_{h+1, f^{\ast}}(x_{h+1}^k)\right)^2  -\left(Q_{h,f^{\ast}}(x_h^k,a_h^k,b_h^k) - r_h^k - V_{h+1, f^{\ast}}(x_{h+1}^k)\right)^2.
        \end{align*}
        Using the Bellman equation for $Q_{f^{\ast}}$, i.e., 
        \begin{align*}
            Q_{h,f^{\ast}}(x_h^k,a_h^k,b_h^k) = r_h^k + \mathbb{E}_{x_{h+1}\sim\mathbb{P}_h(\cdot|x_h^k,a_h^k,b_h^k)}[V_{h+1, f^{\ast}}(x_{h+1})]
        \end{align*} 
        we can calculate that 
            \begin{align}
            \mathbb{E}_{x_{h+1}^k\sim \mathbb{P}_h(\cdot|x_h^k,a_h^k, b_h^k)}[W_{h,f}^k] = \left(Q_{h,f}(x_h^k,a_h^k,b_h^k) - Q_{h, f^{\ast}}(x_{h}^k, a_h^k, b_h^k)\right)^2.
            \end{align}
        Under the filtration $\{\mathcal{F}_{h,k}\}_{k=1}^K$ defined in the proof of Proposition \ref{prop: supervised learning model free rl mg}, i.e, \eqref{eq: filtration model free mg}, one can derive that 
        \begin{align}
            \mathbb{E}[W_{h,f}^k|\mathcal{F}_{h,k-1}] &= \mathbb{E}\left[\mathbb{E}_{x^k_{h+1}\sim \mathbb{P}_h(\cdot|x_h^k,a_h^k)}[W_{h,f^k}]|\mathcal{F}_{h,k-1}\right] \notag \\
            &= \mathbb{E}_{\xi_h\sim \boldsymbol{\pi}^k}\left[\left(Q_{h,f}(x_h,a_h,b_h) - Q_{h, f^{\ast}}(x_{h}, a_h, b_h)\right)^2\right],\label{eq: proof supervised learning model free rl star mg 2+}
        \end{align}
        where  $\xi_h = (x_h,a_h,b_h,r_h,x_{h+1})$, and that
        \begin{align}
        \mathbb{V}[W_{h,f}^k|\mathcal{F}_{h,k-1}]  \leq 4B_f^2\mathbb{E}[X_{h,f}^k|\mathcal{F}_{h,k-1}] = 4B_f^2\mathbb{E}_{\xi_h\sim \boldsymbol{\pi}^k}\left[\left(Q_{h,f}(x_h,a_h,b_h) - Q_{h, f^{\ast}}(x_{h}, a_h, b_h)\right)^2\right].\label{eq: proof supervised learning model free rl star mg 2++}
    \end{align}
    By applying Lemma \ref{lem: freedman}, \eqref{eq: proof supervised learning model free rl star mg 2+}, and \eqref{eq: proof supervised learning model free rl star mg 2++}, we can obtain that with probability at least $1-\delta$, for any $f\in\mathcal{H}$, $(h,k)\in[H]\times[K]$,
    \begin{align*}
        &\left|\sum_{s=1}^{k-1}W_{h,f}^s - \sum_{s=1}^{k-1}\mathbb{E}_{\xi_h\sim \boldsymbol{\pi}^k}\left[\left(Q_{h,f}(x_h,a_h,b_h) - Q_{h, f^{\ast}}(x_{h}, a_h, b_h)\right)^2\right]\right| \lesssim 4B_f^2\log(HK|\mathcal{H}_h||\mathcal{H}_{h+1}|/\delta)\notag \\
        & \qquad +\sqrt{\log(HK|\mathcal{H}_h||\mathcal{H}_{h+1}|/\delta)\cdot \sum_{s=1}^{k-1}\mathbb{E}_{\xi_h\sim \boldsymbol{\pi}^s}\left[\left(Q_{h,f}(x_h,a_h,b_h) - Q_{h, f^{\ast}}(x_{h}, a_h, b_h)\right)^2\right]}.
    \end{align*}
    Rearranging terms, we have that with probability at least $1-\delta$, for any $(f_h,f_{h+1})\in\mathcal{H}\times\mathcal{H}_{h+1}$, $(h,k)\in[H]\times[K]$,
    \begin{align*}
        -\sum_{s=1}^{k-1}W_{h,f}^s &\lesssim 4B_f^2\log(HK|\mathcal{H}_h||\mathcal{H}_{h+1}|/\delta) - \sum_{s=1}^{k-1}\mathbb{E}_{\xi_h\sim \boldsymbol{\pi}^s}\left[\left(Q_{h,f}(x_h,a_h,b_h) - Q_{h, f^{\ast}}(x_{h}, a_h, b_h)\right)^2\right] \\
        &\qquad + \sqrt{\log(HK|\mathcal{H}_h||\mathcal{H}_{h+1}|/\delta)\cdot \sum_{s=1}^{k-1}\mathbb{E}_{\xi_h\sim \boldsymbol{\pi}^s}\left[\left(Q_{h,f}(x_h,a_h,b_h) - Q_{h, f^{\ast}}(x_{h}, a_h, b_h)\right)^2\right]} \\
        & \lesssim 8B_f^2\log(HK|\mathcal{H}_h||\mathcal{H}_{h+1}|/\delta),
    \end{align*}
    where in the second inequality we use the fact that $-x^2 + ax \leq a^2/4$.
    Thus, with probability at least $1-\delta$, for any $k\in[K]$, it holds that
    \begin{align*}
        \sum_{h=1}^HL_h^{k-1}(f^{\ast}) 
         &= \sum_{h=1}^H\left(\sum_{s=1}^{k-1}  \left(Q_{h,f^{\ast}}(x_h^s,a_h^s,b_h^s) - r_h^s - V_{h+1, f^{\ast}}(x_{h+1}^s)\right)^2 \right. \\
         &\qquad \left. - \inf_{f_h\in\mathcal{H}_h} \sum_{s=1}^{k-1} \left(Q_{h,f}(x_h^s,a_h^s,b_h^s) - r_h^s - V_{h+1, f^{\ast}}(x_{h+1}^s)\right)^2 \right) \\
         & =\sum_{h=1}^H\sup_{f_h\in\mathcal{H}_h}\sum_{s=1}^{k-1}  -W_{h,f}^s  \lesssim 8HB_f^2\log(HK/\delta) + 16B_f^2\log(|\cH|).
    \end{align*}
    This finishes the proof of Lemma \ref{lem: f star model free mg}.
    \end{proof}

    Finally, combining \eqref{eq: proof supervised learning model free rl mg 6} and Lemma \ref{lem: f star model free mg}, we have, with probability at least $1-\delta$, for any $f\in\mathcal{H}$, $k\in[K]$,
    \begin{align*}
        \sum_{h=1}^HL_h^{k-1}(f^{\ast}) - L_h^{k-1}(f) \lesssim -\frac{1}{2}\sum_{h=1}^H\sum_{s=1}^{k-1}\mathbb{E}_{\xi_h\sim \boldsymbol{\pi}s}[ \ell_{f^s}(f;\xi_h)] +  16HB_f^2\log(HK/\delta) + 32 B_f^2\log(|\cH|).
    \end{align*}
    This finishes the proof of the \emph{first} inequality in Proposition \ref{prop: supervised learning model free rl mg}.
    In the following, we prove the \emph{second} inequality in Proposition \ref{prop: supervised learning model free rl mg}.
    To this end, we define the following random variable, for any $f, g\in\cH$ and policy $\mu_f$,
    \begin{align}
        X_{h,g, \mu_f}^k &= \left(Q_{h,g}(x_h^k,a_h^k,b_h^k) - r_h^k - V_{h+1, g}^{\mu_f,\dagger}(x_{h+1}^k)\right)^2 \notag \\
        &\qquad  - \left(V_{h+1, g}^{\mu_f,\dagger}(x_{h+1}^k) - \mathbb{E}_{x_{h+1}\sim\mathbb{P}_h(\cdot|x_h^k,a_h^k,b_h^k)}[V_{h+1,g}^{\mu_f,\dagger}(x_{h+1})]\right)^2.\label{eq: proof supervised learning model free rl mg min 0}
    \end{align}
    After a calculation similar to \eqref{eq: proof supervised learning model free rl 1} and \eqref{eq: proof supervised learning model free rl 2}, we can derive that 
    \begin{align*}
        \mathbb{E}_{x_{h+1}^k\sim\mathbb{P}_h(\cdot|x_h^k,a_h^k,b_h^k)}[X_{h,g, \mu_f}^k] = \Big(Q_{h,g}(x_h^k,a_h^k,b_h^k) - r_h^k - \mathbb{E}_{x_{h+1}\sim\mathbb{P}_h(\cdot|x_h^k,a_h^k,b_h^k)}[V_{h+1,g}^{\mu_f,\dagger}(x_{h+1})]\Big)^2.
    \end{align*}
    Following the same argument as in the previous proof of the \emph{first} inequality of Proposition \ref{prop: supervised learning model free rl mg} (see \eqref{eq: proof supervised learning model free rl mg 2+} and \eqref{eq: proof supervised learning model free rl mg 2++}), using the definition of $\ell_{\mu}$ in \eqref{eq: example model free ell mg 2}, we can derive that, under filtration defined in \eqref{eq: filtration model free mg}, 
    \begin{align}\label{eq: proof supervised learning model free rl mg min 2}
        \mathbb{E}[X_{h,f}^k|\mathcal{F}_{h,k-1}] = \mathbb{E}_{\xi_h\sim \boldsymbol{\pi}^k}[ \ell_{g^k, \mu_f}(g;\xi_h)],\qquad \mathbb{V}[X_{h,f}^k|\mathcal{F}_{h,k-1}] \leq 4B_f^2 \mathbb{E}_{\xi_h\sim \boldsymbol{\pi}^k}[ \ell_{g^k, \mu_f}(g;\xi_h)].
    \end{align}
    Using \eqref{eq: proof supervised learning model free rl mg min 2} and Lemma \ref{lem: freedman}, we can obtain that with probability at least $1-\delta$, for any $(h,k)\in[H]\times[K]$ and $(g_h,g_{h+1}, f_{h+1})\in\mathcal{H}_h\times\mathcal{H}_{h+1}\times \mathcal{H}_{h+1}$\footnote{Note that $\ell_{g^s, \mu_f}(g;\xi_h)$ and $V_{h+1,g}^{\mu_f,\dagger}$ depend on $f$ only through $f_{h+1}$.},
    \begin{align}\label{eq: proof supervised learning model free rl mg min 3}
        \left|\sum_{s=1}^{k-1}\mathbb{E}_{\xi_h\sim \boldsymbol{\pi}^s}[ \ell_{g^s, \mu_f}(g;\xi_h)] - \sum_{s=1}^{k-1}X_{h,g, \mu_f}^s\right| \lesssim \frac{1}{2}\sum_{s=1}^{k-1}\mathbb{E}_{\xi_h\sim \boldsymbol{\pi}^s}[ \ell_{g^s,\mu_f}(g;\xi_h)] +  16B_f^2\log(HK|\mathcal{H}_h|^2|\mathcal{H}_{h+1}|/\delta).
    \end{align}
    Rearranging terms in \eqref{eq: proof supervised learning model free rl mg min 3}, we can further obtain that 
    \begin{align}\label{eq: proof supervised learning model free rl mg min 4}
        - \sum_{s=1}^{k-1}X_{h,g,\mu_f}^s \lesssim -\frac{1}{2}\sum_{s=1}^{k-1}\mathbb{E}_{\xi_h\sim \boldsymbol{\pi}^s}[ \ell_{g^s,\mu_f}(g;\xi_h)] +  16B_f^2\log(HK|\mathcal{H}_h|^2|\mathcal{H}_{h+1}|/\delta).
    \end{align}
    Meanwhile, by the definition of $X_{h,f}$ in \eqref{eq: proof supervised learning model free rl mg min 0} and the loss function $L$ in \eqref{eq: implicit optimism mg model free L 1}, we have that 
    \begin{align}\label{eq: proof supervised learning model free rl mg min 5}
        &\sum_{s=1}^{k-1}X_{h,g, \mu_f}^s \notag \\
        &\quad = \sum_{s=1}^{k-1}\left(Q_{h,g}(x_h^s,a_h^s,b_h^s) - r_h^s - V_{h+1, g}^{\mu_f,\dagger}(x_{h+1}^s)\right)^2  - \sum_{s=1}^{k-1}\left(V_{h+1, g}^{\mu_f,\dagger}(x_{h+1}^s) - \mathbb{E}_{x_{h+1}\sim\mathbb{P}_h(\cdot|x_h^s,a_h^s,b_h^s)}[V_{h+1,g}^{\mu_f,\dagger}(x_{h+1})]\right)^2\notag\\
        &\quad= \sum_{s=1}^{k-1}\left(Q_{h,g}(x_h^s,a_h^s,b_h^s) - r_h^s - V_{h+1, g}^{\mu_f,\dagger}(x_{h+1}^s)\right)^2 - \sum_{s=1}^{k-1}\left(\mathcal{T}_h^{\mu_f}g(x_h^s,a_h^s,b_h^s) - r_h^s - V_{h+1, g}^{\mu_f,\dagger}(x_{h+1}^s)\right)^2\notag\\
        &\quad \leq \sum_{s=1}^{k-1} \left(Q_{h,g}(x_h^s,a_h^s,b_h^s) - r_h^s - V_{h+1, g}^{\mu_f,\dagger}(x_{h+1}^s)\right)^2 - \inf_{f_h'\in\mathcal{H}_h}\sum_{s=1}^{k-1} \left(Q_{h,f'}(x_h^s,a_h^s,b_h^s) - r_h^s - V_{h+1, g}^{\mu_f,\dagger}(x_{h+1}^s)\right)^2\notag\\
        &\quad =L_{h,\mu_f}^{k-1}(f).
    \end{align}
    where the last inequality follows from the completeness assumption (Assumption \ref{ass: completeness mg}).
    Combining \eqref{eq: proof supervised learning model free rl mg min 4} and \eqref{eq: proof supervised learning model free rl mg min 5}, we can derive that with probability at least $1-\delta$, for any $f,g\in\mathcal{H}$, $k\in[K]$,
    \begin{align}\label{eq: proof supervised learning model free rl mg min 6}
        -\sum_{h=1}^HL_{h,\mu_f}^{k-1}(f) \lesssim -\frac{1}{2}\sum_{h=1}^H\sum_{s=1}^{k-1}\mathbb{E}_{\xi_h\sim \boldsymbol{\pi}^s}[ \ell_{g^s,\mu_f}(g;\xi_h)] +  16HB_f^2\log(HK/\delta) + 48B_f^2\log(|\cH|).
    \end{align}
    Especially, we take $f = f^k$, we can obtain that with probability at least $1-\delta$, for any $g\in\mathcal{H}$, $k\in[K]$,
    \begin{align}\label{eq: proof supervised learning model free rl mg min 7}
        -\sum_{h=1}^HL_{h,\mu^k}^{k-1}(f) \lesssim -\frac{1}{2}\sum_{h=1}^H\sum_{s=1}^{k-1}\mathbb{E}_{\xi_h\sim \boldsymbol{\pi}^s}[ \ell_{g^s,\mu^k}(g;\xi_h)] +  16HB_f^2\log(HK|\mathcal{H}|/\delta)+48B_f^2\log(|\cH|).
    \end{align}
    Finally, we deal with the term $L_h^k(f^{\ast})$. 
    To this end, we invoke the following lemma.

    \begin{lemma} \label{lem: f star model free mg min}
        With probability at least $1-\delta$, it holds that for each $k\in[K]$,
        \begin{align*}
            \sum_{h=1}^HL^{k-1}_{h,\mu^k}(Q^{\mu^k,\dagger}) \lesssim 16HB_f^2\log(HK/\delta) + 48B_f^2\log(|\cH|).
        \end{align*}
    \end{lemma}

    \begin{proof}[Proof of Lemma \ref{lem: f star model free mg min}]
        To prove Lemma \ref{lem: f star model free mg min}, we define the following random variable, 
        \begin{align*}
            W_{h,g,\mu_f}^k = \left(Q_{h,g}(x_h^k,a_h^k,b_h^k) - r_h^k - V_{h+1}^{\mu_f,\dagger}(x_{h+1}^k)\right)^2  -\left(Q_{h}^{\mu_f,\dagger}(x_h^k,a_h^k,b_h^k) - r_h^k - V_{h+1}^{\mu_f,\dagger}(x_{h+1}^k)\right)^2,
        \end{align*}
        for any $f, g\in\cH$. Using the Bellman equation for $Q^{\mu_f,\dagger}$, i.e., 
        \begin{align*}
            Q^{\mu_f,\dagger}_h(x_h^k,a_h^k,b_h^k) = r_h^k + \mathbb{E}_{x_{h+1}\sim\mathbb{P}_h(\cdot|x_h^k,a_h^k,b_h^k)}[V_{h+1}^{\mu_f,\dagger}(x_{h+1})]
        \end{align*} 
        we can calculate that 
            \begin{align}
            \mathbb{E}_{x_{h+1}^k\sim \mathbb{P}_h(\cdot|x_h^k,a_h^k, b_h^k)}[W_{h,g,\mu_f}^k] = \left(Q_{h,g}(x_h^k,a_h^k,b_h^k) - Q_{h}^{\mu_f,\dagger}(x_{h}^k, a_h^k, b_h^k)\right)^2.
            \end{align}
        Under the filtration $\{\mathcal{F}_{h,k}\}_{k=1}^K$ defined in the proof of Proposition \ref{prop: supervised learning model free rl mg}, i.e, \eqref{eq: filtration model free mg}, we can derive that 
        \begin{align}
            \mathbb{E}[W_{h,g,\mu_f}^k|\mathcal{F}_{h,k-1}] &= \mathbb{E}_{\xi_h\sim \boldsymbol{\pi}^k}\left[\left(Q_{h,g}(x_h,a_h,b_h) - Q_{h}^{\mu_f,\dagger}(x_{h}, a_h, b_h)\right)^2\right],\label{eq: proof supervised learning model free rl star mg min 2+} \\
            \mathbb{V}[W_{h,g,\mu_f}^k|\mathcal{F}_{h,k-1}] &\leq 4B_f^2\mathbb{E}_{\xi_h\sim \boldsymbol{\pi}^k}\left[\left(Q_{h,g}(x_h,a_h,b_h) - Q_{h}^{\mu_f,\dagger}(x_{h}, a_h, b_h)\right)^2\right].\label{eq: proof supervised learning model free rl star mg min 2++}
        \end{align}
        Using Lemma \ref{lem: freedman}, \eqref{eq: proof supervised learning model free rl star mg min 2+}, \eqref{eq: proof supervised learning model free rl star mg min 2++}, we have that, with probability at least $1-\delta$, for any $(h,k)\in[H]\times[K]$, $(g_h,g_{h+1}, f_h, f_{h+1})\in\mathcal{H}_h\times\mathcal{H}_{h+1}\times\mathcal{H}_h\times \mathcal{H}_{h+1}$,
        \begin{align*}
            &\left|\sum_{s=1}^{k-1}W_{h,g,\mu_f}^s - \sum_{s=1}^{k-1}\mathbb{E}_{\xi_h\sim \boldsymbol{\pi}^k}\left[\left(Q_{h,g}(x_h,a_h,b_h) - Q_{h}^{\mu_f,\dagger}(x_{h}, a_h, b_h)\right)^2\right]\right| \lesssim 8B_f^2\log(HK|\mathcal{H}_h|^2|\mathcal{H}_{h+1}|^2/\delta)\notag \\
            & \qquad +\sqrt{\log(HK|\mathcal{H}_h|^2|\mathcal{H}_{h+1}|^2/\delta)\cdot \sum_{s=1}^{k-1}\mathbb{E}_{\xi_h\sim \boldsymbol{\pi}^s}\left[\left(Q_{h,g}(x_h,a_h,b_h) - Q_{h}^{\mu_f,\dagger}(x_{h}, a_h, b_h)\right)^2\right]}.
        \end{align*}
        Rearranging terms, we have that with probability at least $1-\delta$, 
        \begin{align*}
            -\sum_{s=1}^{k-1}W_{h,g,\mu_f}^s &\lesssim 8B_f^2\log(HK|\mathcal{H}_h|^2|\mathcal{H}_{h+1}|^2/\delta) - \sum_{s=1}^{k-1}\mathbb{E}_{\xi_h\sim \boldsymbol{\pi}^s}\left[\left(Q_{h,g}(x_h,a_h,b_h) - Q_{h}^{\mu_f,\dagger}(x_{h}, a_h, b_h)\right)^2\right] \\
            &\qquad + \sqrt{\log(HK|\mathcal{H}_h|^2|\mathcal{H}_{h+1}|^2/\delta)\cdot \sum_{s=1}^{k-1}\mathbb{E}_{\xi_h\sim \boldsymbol{\pi}^s}\left[\left(Q_{h,g}(x_h,a_h,b_h) - Q_{h}^{\mu_f,\dagger}(x_{h}, a_h, b_h)\right)^2\right]} \\
            & \lesssim 16B_f^2\log(HK|\mathcal{H}_h|^2|\mathcal{H}_{h+1}|^2/\delta),
        \end{align*}
        where in the second inequality we use the fact that $-x^2 + ax \leq a^2/4$.
        Now we take $f = f^k$, which gives that with probability at least $1-\delta$, for any $k\in[K]$, it holds that
        \begin{align*}
            \sum_{h=1}^HL_{h,\mu^k}^{k-1}(Q^{\mu^k,\dagger}) 
            &= \sum_{h=1}^H\left(\sum_{s=1}^{k-1}  \left(\underbrace{Q_{h,Q^{\mu^k,\dagger}}(x_h^s,a_h^s,b_h^s)}_{ = Q^{\mu^k,\dagger}_h(x_h^s,a_h^s,b_h^s)} - r_h^s - \underbrace{V_{h+1, Q^{\mu^k,\dagger}}^{\mu^k,\dagger}(x_{h+1}^s)}_{= V^{\mu^k,\dagger}_{h+1}(x_{h+1}^s)}\right)^2 \right. \\
            &\qquad \left. - \inf_{g_h\in\mathcal{H}_h} \sum_{s=1}^{k-1} \left(Q_{h,g}(x_h^s,a_h^s,b_h^s) - r_h^s -  \underbrace{V_{h+1, Q^{\mu^k,\dagger}}^{\mu^k,\dagger}(x_{h+1}^s)}_{= V^{\mu^k,\dagger}_{h+1}(x_{h+1}^s)}\right)^2 \right) \\
            & =\sum_{h=1}^H\sup_{g_h\in\mathcal{H}_h}\sum_{s=1}^{k-1}  -W_{h,g,\mu^k}^s  \lesssim 16HB_f^2\log(HK/\delta) + 64B_f^2\log(|\mathcal{H}|).
        \end{align*}
        This finishes the proof of Lemma \ref{lem: f star model free mg min}.
    \end{proof}
    Finally, combining \eqref{eq: proof supervised learning model free rl mg min 7} and Lemma \ref{lem: f star model free mg min}, we have, with probability at least $1-\delta$, for any $g\in\mathcal{H}$, $k\in[K]$,
    \begin{align*}
        \sum_{h=1}^HL_{h,\mu^k}^{k-1}(Q^{\mu^k,\dagger}) - L_{h,\mu^k}^{k-1}(g) \lesssim -\frac{1}{2}\sum_{h=1}^H\sum_{s=1}^{k-1}\mathbb{E}_{\xi_h\sim \boldsymbol{\pi}s}[ \ell_{g^s,\mu^k}(g;\xi_h)] +  32HB_f^2\log(HK|\mathcal{H}|/\delta) + 112B_f^2\log(|\cH|).
    \end{align*}
    This finishes the proof of the \emph{second} inequality in Proposition \ref{prop: supervised learning model free rl mg} and completes the proof of Proposition~\ref{prop: supervised learning model free rl mg}.
\end{proof}

%% file: tex/appendix/technical.tex
\section{Technical Lemmas}

\begin{lemma}[Martingale exponential inequality]\label{lem: concentration}
    For a sequence of real-valued random variables $\{X_t\}_{t\leq T}$ adapted to a filtration $\{\mathcal{F}_t\}_{t\leq T}$ , the following holds with probability at least $1-\delta$, for any $t\in[T]$,
    \begin{align*}
        -\sum_{s=1}^t X_s \leq \sum_{s=1}^t\log\mathbb{E}[\exp(-X_s)|\mathcal{F}_{s-1}] + \log(1/\delta).
    \end{align*}
\end{lemma}

\begin{proof}[Proof of Lemma \ref{lem: concentration}]
    %See Lemma A.4 in \cite{foster2021statistical} for a detailed proof.
    See e.g., Theorem 13.2 of \citet{zhang2022mathematical} for a detailed proof. 
\end{proof}

\begin{lemma}[Freedman's inequality]\label{lem: freedman}
    Let $\{X_t\}_{t \leq T}$ be a real-valued martingale difference sequence adapted to filtration $\{\mathcal{F}_t\}_{t\leq T}$.
    If $|X_t| \leq R$ almost surely, then for any $\eta \in(0, 1/R)$ it holds that with probability at least $1-\delta$,
    \begin{align*}
        \sum_{t=1}^T X_t \leq \mathcal{O}\left(\eta \sum_{t=1}^T \mathbb{E}[X_t^2|\mathcal{F}_{t-1}]+\frac{\log(1/\delta)}{\eta}\right).
    \end{align*}
\end{lemma}

\begin{proof}[Proof of Lemma \ref{lem: freedman}]
    See \cite{freedman1975tail} for a detailed proof.
\end{proof}

\begin{lemma}[Elliptical potential]
  \label{lem:ell}
  Let $\{x_s\}_{s\in[K]}$ be a sequence of vectors with $x_s\in\mathcal{V}$ for some Hilbert space $\cV$. 
  Let $\Lambda_0$ be a positive definite matrix and define $\Lambda_k = \Lambda_0 + \sum_{s=1}^kx_sx_s^\top$.
  Then it holds that 
  \begin{align*}
      \sum_{s=1}^K\min\left\{1, \|x_s\|_{\Lambda_{s}^{-1}}\right\}\leq 2\log\left(\frac{\det(\Lambda_{K+1})}{\det(\Lambda_1)}\right).
  \end{align*}
\end{lemma}
\begin{proof}[Proof of Lemma \ref{lem:ell}]
  See Lemma 11 of \citet{abbasi2011improved} for a detailed proof.
\end{proof}

\section{Experiment Settings}\label{app_exp_setting}

Our experiments utilize 8 NVIDIA GeForce 1080Ti GPUs and 4 NVIDIA A6000 GPUs. Each result is averaged over five random seeds.

\subsection{Implementation Details of \texttt{MEX-MF}}\label{subsec: details mex mf}
Below, we describe the detailed implementation of the model-free algorithm \texttt{MEX-MF}.
 We select $\eta'$ to be $1e-3$ for sparse-reward tasks and $5e-4$ for standard gym tasks since dense reward tasks require less exploration. Other parameters are kept the same with the baseline~\cite {fujimoto2018addressing} across all domains and are summarized as in Table~\ref{tab:mf_param}.

\subsection{Implementation Details of \texttt{MEX-MB}}\label{subsec: details mex mb}
When employing the model-based algorithm \texttt{MEX-MB}, we configured the parameter $\eta'$ as $1e-4$ for the \texttt{Hopper-v2} and \texttt{hopper-vel} tasks, and $1e-3$ for all other tasks. The hyper-parameters are kept the same with the MBPO baseline~\cite {janner2019trust} across all domains and are summarized as in Table~\ref{tab:mb_param}.

\begin{table}[H]
  \centering
  \begin{minipage}{0.45\textwidth}
    \centering
    \begin{tabular}{ll}
      \toprule
      Hyperparameter   & Value \\
      \midrule
      \hspace{0.3cm} Optimizer & Adam \\
      \hspace{0.3cm} Critic learning rate & 3e-4 \\
      \hspace{0.3cm} Actor learning rate & 3e-4 \\
      \hspace{0.3cm} Mini-batch size & 256 \\
      \hspace{0.3cm} Discount factor & 0.99 \\
      \hspace{0.3cm} Target update rate & 5e-3 \\
      \hspace{0.3cm} Policy noise & 0.2 \\
      \hspace{0.3cm} Policy noise clipping & (-0.5, 0.5) \\
      \hspace{0.3cm} TD3+BC parameter $\alpha$ & 2.5\\
      \midrule
      Architecture & Value \\
      \midrule
      \hspace{0.3cm} Critic hidden dim & 256 \\
      \hspace{0.3cm} Critic hidden layers & 2 \\
      \hspace{0.3cm} Critic activation function & ReLU \\
      \hspace{0.3cm} Actor hidden dim & 256 \\
      \hspace{0.3cm} Actor hidden layers & 2 \\
      \hspace{0.3cm} Actor activation function & ReLU \\
      \bottomrule
    \end{tabular}
    \caption{Hyper-parameters sheet of \texttt{MEX-MF}.}
    \label{tab:mf_param}
  \end{minipage}\hfill
  \begin{minipage}{0.45\textwidth}
    \centering
    \begin{tabular}{ll}
      \toprule
      Hyperparameter   & Value \\
      \midrule
      \hspace{0.3cm} Optimizer & Adam \\
      \hspace{0.3cm} Critic learning rate & 3e-4 \\
      \hspace{0.3cm} Actor learning rate & 3e-4 \\
      \hspace{0.3cm} Model learning rate & 1e-3 \\
      \hspace{0.3cm} Mini-batch size & 256 \\
      \hspace{0.3cm} Discount factor & 0.99 \\
      \hspace{0.3cm} Target update rate & 5e-3 \\
      \hspace{0.3cm} SAC updates per step & 40 \\
      \midrule
      Architecture & Value \\
      \midrule
      \hspace{0.3cm} Critic hidden layers & 3 \\
      \hspace{0.3cm} Critic activation function & ReLU \\
      \hspace{0.3cm} Actor hidden layers & 2 \\
      \hspace{0.3cm} Actor activation function & ReLU \\
      \hspace{0.3cm} Model hidden dim & 200 \\
         \hspace{0.3cm} Model hidden layers & 4 \\
  \hspace{0.3cm} Model activation function & SiLU \\
  \bottomrule
\end{tabular}
\caption{Hyper-parameters sheet of \texttt{MEX-MB}.}
\label{tab:mb_param}
\end{minipage}\hfill
\end{table}

\subsection{Tabular Experiments}
We also conduct experiments in tabular MDPs. Specifically, we evaluate \texttt{MEX-MB} and \texttt{MnM} \citep{eysenbach2022mismatched} in a 10x10 gridworld with stochastic dynamics and sparse reward functions. As illustrated in Figure \ref{fig:gridenv}, the stochastic gridworld environment is associated with a navigation task to reach the red star from the initial upper left cell position. The action space contains four discrete actions, corresponding to moving to the four adjacent cells. The transition noise moves the agent to neighbor states with equal probability. The black region represents the obstacle that the agent cannot enter. The agent receives a $+0.001$ reward at every timestep and a $+10$ when reaching the goal state. Each episode has $200$ timesteps. The performance results are shown in Figure \ref{fig:gridworld}.

\begin{figure}[h]
\begin{minipage}[t]{0.4\textwidth}
  \centering
  \includegraphics[width=0.95\linewidth]{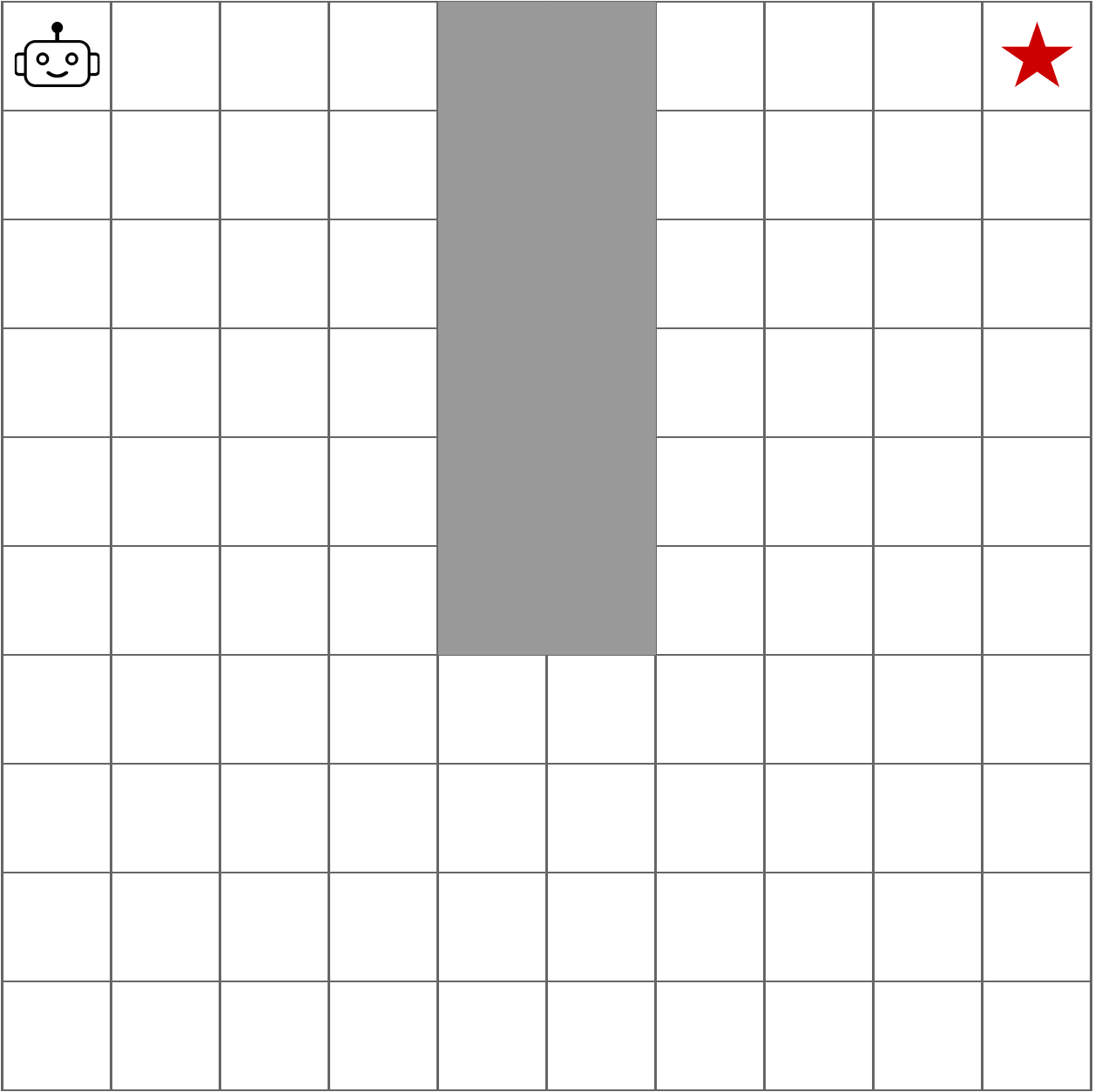}
  \caption{Illustration of the stochastic griworld environment \citep{eysenbach2022mismatched}.}
  \label{fig:gridenv}
\end{minipage}\hfill
\begin{minipage}[t]{0.55\textwidth}
  \centering
  \includegraphics[width=\linewidth]{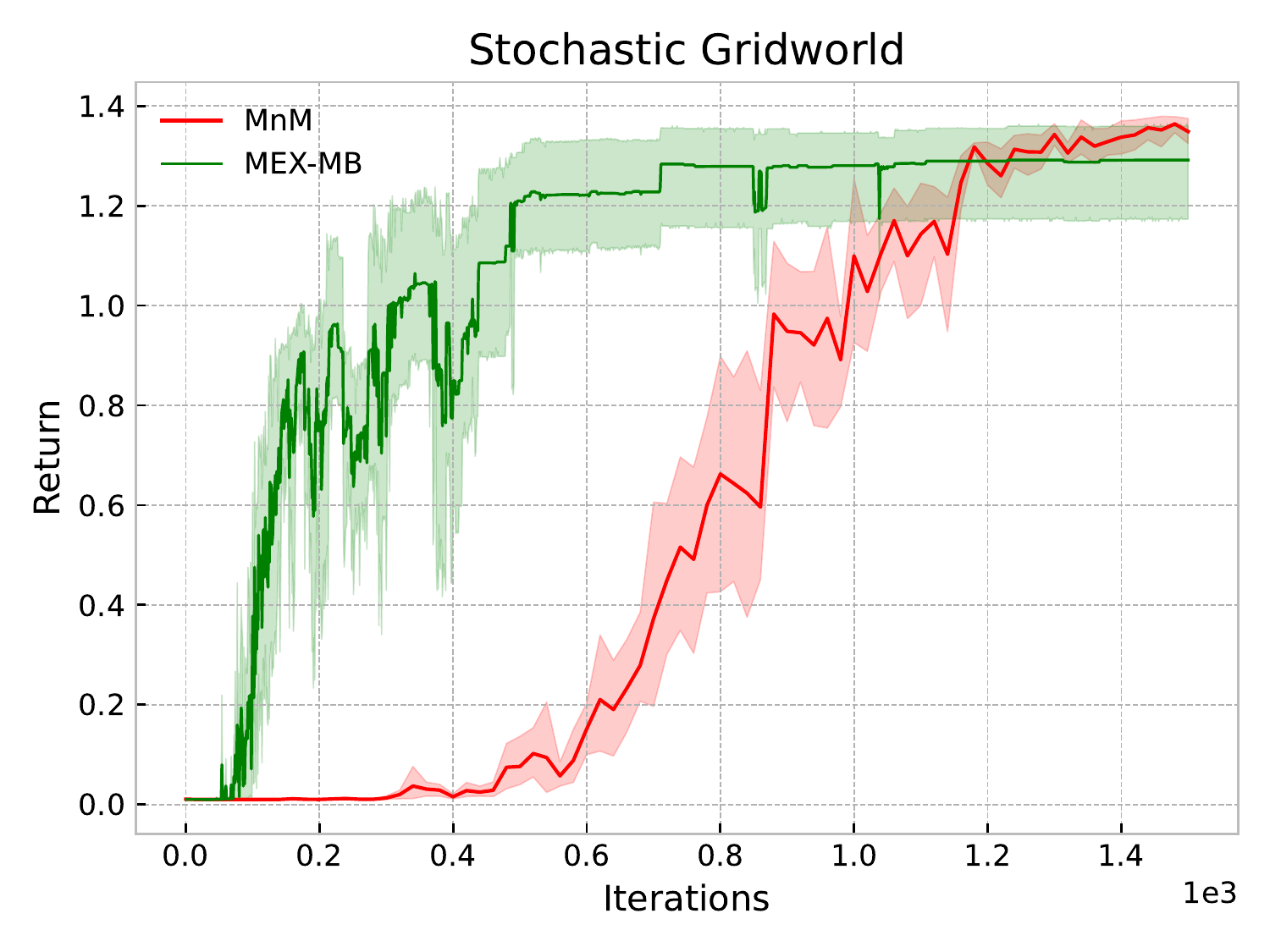}
  \caption{Model-based \texttt{MEX-MB} in the stochastic gridworld environment.}
  \label{fig:gridworld}
\end{minipage}
\end{figure}